\documentclass[10pt]{article}
\usepackage[margin=1in]{geometry}
\usepackage{hyperref}
\usepackage{times,multirow,caption,float,url}
\usepackage{enumitem}
\usepackage{wrapfig}
\usepackage{graphicx}
\usepackage{epsfig,xspace,layout}
\usepackage{subfigure}
\usepackage{color}
\usepackage{amsfonts}
\usepackage{times}
\usepackage{amssymb}
\usepackage{amsmath}
\usepackage{rotating}
\usepackage{amsthm}
\usepackage{graphicx}
\usepackage{epsfig}
\usepackage{mathrsfs}
\usepackage{caption}
\usepackage{subfigure}
\usepackage{mathtools}
\usepackage{makeidx}
\usepackage{multirow} 
\usepackage{dblfloatfix}
\usepackage{threeparttable}
\usepackage{dsfont}
\usepackage[font={small}]{caption}
\usepackage{cleveref}
\usepackage[linesnumbered,algoruled,boxed,lined]{algorithm2e}
\usepackage{tikz}
\usetikzlibrary{shapes,arrows,positioning,calc}
\usepackage{siunitx}
\usepackage{cite}
\usepackage{booktabs}
\usepackage{multirow}
\usepackage{siunitx}
\usepackage[title]{appendix}

\newtheorem{remark}{Remark}
\newtheorem{example}{Example}

\theoremstyle{definition}
\newtheorem{theorem}{Theorem}

\newtheorem{proposition}{Proposition}
\newtheorem{lemma}{Lemma}

\newtheorem{definition}{Definition}

\DeclareMathAlphabet{\mathpzc}{OT1}{pzc}{m}{it}

\DeclareFontFamily{U}{jkpmia}{}
\DeclareFontShape{U}{jkpmia}{m}{it}{<->s*jkpmia}{}
\DeclareFontShape{U}{jkpmia}{bx}{it}{<->s*jkpbmia}{}
\DeclareMathAlphabet{\mathfrak}{U}{jkpmia}{m}{it}

\DeclareMathOperator{\dt}{dt}

\DeclareMathOperator{\Pre}{Pre}
\DeclareMathOperator{\PRE}{PRE}
\newcommand{\Hh}{{\mathbf H}}

\newcommand{\xxi}{{ \boldsymbol \xi}}

\newcommand{\w}{{\mathbf w}}
\newcommand{\Cs}{{\boldsymbol{\cal C}}}

\newcommand{\x}{{\mathbf x}}

\newcommand{\Aa}{{\mathbf A}}
\newcommand{\Bb}{{\mathbf B}}
\renewcommand{\u}{{\mathbf u}}
\newcommand{\Pf}{{\mathbf P}}
\newcommand{\Bx}{\boldsymbol{\mathcal{B}}}

\newcommand{\sigmaa}{\boldsymbol{\sigma}}
\newcommand{\norm}[1]{\|#1\|}

\renewcommand{\c}{\mathbf{c}}

\newcommand{\so}{\emph{SO}(3)}

\newcommand{\gr}{\mathtt{g}}
\newcommand{\hz}{\mathtt{h}}
\newcommand{\Id}{\mathbf{I}}
\newcommand{\Ie}{\mathtt{I}}

\newcommand{\h}{\mathbf{h}}

\newcommand{\0}{\mathbf{0}}
\newcommand{\p}{{\mathbf p}}%
\newcommand{\R}{\mathbb{R}}

\newcommand{\Na}{\mathbb{N}}
\renewcommand{\S}{\mathcal{S}}

\newcommand{\B}{\mathcal{B}}

\newcommand{\X}{\mathcal{X}}

\newcommand{\U}{\mathcal{U}}
\newcommand{\W}{\mathcal{W}}
\newcommand{\Ww}{\mathbf{W}}
\newcommand{\I}{\mathbf{I}}

\newcommand{\F}{\mathcal{F}}

\renewcommand{\P}{\mathcal{P}}

\newcommand{\Reach}{\mathcal{R}}
\newcommand{\Q}{\mathcal{Q}}
\newcommand{\Qc}{\mathbf{Q}}
\newcommand{\Rc}{\mathbf{R}}

\newcommand{\case}[1]{\begin{cases}#1\end{cases}}
\newcommand{\ar}[1]{\left[\begin{array}#1\end{array}\right]}
\newcommand{\al}[1]{\begin{align}#1\end{align}}
\newcommand{\eq}[1]{\begin{equation}#1\end{equation}}
\newcommand{\ald}[1]{\begin{aligned}#1\end{aligned}}

\newcommand{\eqn}[1]{\begin{equation*}#1\end{equation*}}

\newcommand{\subeq}[1]{\begin{subequations}#1\end{subequations}}
\newcommand{\st}{\text{s.t. }}

\newcommand{\In}{{\mathbf 1}}

\newcommand{\omegaa}{\boldsymbol{\omega}}
\newcommand{\f}{\mathbf f}

\renewcommand{\v}{\mathbf v}

\newcommand{\Ca}{{\cal C}}

\newcommand{\Rx}{{\mathscr {R}}}
\newcommand{\Cx}{{\mathscr {C}}}

\newcommand{\footremember}[2]{%
    \footnote{#2}
    \newcounter{#1}
    \setcounter{#1}{\value{footnote}}%
}
\newcommand{\footrecall}[1]{%
    \footnotemark[\value{#1}]%
} 
\setlist[enumerate]{%
wide =0.5\parindent,
listparindent=0pt%
}
\setlist[itemize]{%
wide =0.5\parindent,
listparindent=0pt%
}

\title{Quadruped Capturability and Push Recovery via a Switched-Systems Characterization of Dynamic Balance}
\author{Hua Chen\footremember{sustech}{Department of Mechanical and Energy Engineering, Southern University of Science and Technology, Shenzhen, China. Emails: { \tt chenh6@sustech.edu.cn, hongzj@mail.sustech.edu.cn, 11930364@mail.sustech.edu.cn, zhangw3@sustech.edu.cn}}, Zejun Hong\footrecall{sustech}, Shunpeng Yang\footrecall{sustech}, Patrick M. Wensing\footremember{nd}{ Department of Aerospace and Mechanical Engineering, University of Notre Dame, Notre Dame, IN 46556 USA. Email: {\tt   pwensing@nd.edu}} and Wei Zhang\footrecall{sustech}}
\date{}

\begin{document}

\maketitle
\begin{abstract}
This paper studies capturability and push recovery for quadrupedal locomotion. Despite the rich literature on capturability analysis and push recovery control for legged robots, existing tools are developed mainly for bipeds or humanoids. Distinct quadrupedal features such as point contacts and multiple swinging legs prevent direct application of these methods. To address this gap, we propose a switched systems model for quadruped dynamics, and instantiate the abstract viability concept for quadrupedal locomotion with a time-based gait. Capturability is characterized through a novel specification of {\em dynamically balanced states} that addresses the time-varying nature of quadrupedal locomotion and balance. A linear inverted pendulum (LIP) model is adopted to demonstrate the theory and show how the newly developed quadrupedal capturability can be used in motion planning for quadrupedal push recovery. We formulate and solve an explicit model predictive control (EMPC) problem whose optimal solution fully characterizes quadrupedal capturability with the LIP. Given this analysis, an optimization-based planning scheme is devised for determining footsteps and center of mass references during push recovery. To validate the effectiveness of the overall framework, we conduct numerous simulation and hardware experiments. Simulation results illustrate the necessity of considering dynamic balance for quadrupedal capturability, and the significant improvement in disturbance rejection with the proposed strategy. Experimental validations on a replica of the Mini Cheetah quadruped demonstrate an up to $100\%$ improvement as compared with state-of-the-art.
\end{abstract}

\section{Introduction}\label{sec:intro}

Quadrupeds form an important class of legged robots that are capable of reliably traversing challenging uneven terrains beyond the reach of classical mobile robots. Recent advancements in the design of high performance actuators and hardware structures have recently led to various quadruped platforms and commercial products, including Spot from Boston Dynamics~\cite{BD}, HyQ from IIT~\cite{Semini2011}, Cheetah series from MIT~\cite{Seok2013,Bledt2018,Katz2019}, and ANYmal from ETH~\cite{Hutter2016}, among others. 

Despite many impressive achievements from these systems, the design of effective and robust quadruped control strategies for handling large disturbances remains challenging. In this paper, we focus on quadrupedal push recovery problem that requires rapid determination of discretely changing footsteps and continuous input torques in reaction to large unexpected disturbances. This problem is challenging to solve in its fully generality due to its combinatorial nature and the nonlinear dynamics of the system. Furthermore, online computation requirements pose additional difficulties.

To address these challenges, we develop a quadrupedal push recovery framework that leverages capturability analysis. To account for the flexibility of choosing different gaits for quadrupeds, we adopt a switched system of quadrupedal locomotion. To formalize the notion of quadruped balance during locomotion, we leverage the idea of control invariant sets for the underlying switched systems and propose a dynamical balance concept. With this dynamic balance concept, we extend the tools for capturability analysis to quadrupedal locomotion using reachability theory, which can be pre-computed offline. Building on the offline analysis with a switched linear inverted pendulum (LIP) model, we further construct an efficient online planning scheme for quadrupedal push recovery. By exploiting a translational symmetry property of the set of capturable states, we provide a tractable approach for generating upcoming footsteps and CoM references via solving a series of quadratic programs (QPs). Simulation and hardware experiments with a replica of the Mini Cheetah platform~\cite{Katz2019} demonstrate the effectiveness of the proposed scheme. Furthermore, comparisons with the state-of-the-art model predictive control and whole-body impulse control (MPC+WBIC)~\cite{Kim2019} show that the proposed approach can resist the impact up-to {\bf 200\%}.

\subsection{Related Works}

Throughout the history of push recovery controller synthesis for legged robots, viability and capturability-based approaches play a central role. Viability theory for general nonlinear systems~\cite{Aubin1990,Aubin2011} concerns with determining whether a given state of the nonlinear system can be controlled to stay safe over time, which has a strong implication on the push recovery ability for legged robots. Wieber~\cite{Wieber2002} pioneered the study on viability theory for legged robots. Later on, Wieber~\cite{Wieber2008} further drew the connection between viability and MPC for legged locomotion. Following the main idea of viability and aiming to synthesize practically applicable controllers, the famous capturability concept was proposed. Building upon the seminal capture point notion, Pratt \emph{et al.}~\cite{Pratt2006} synthesized a push recovery controller for humanoids with a simplified LIP model with flywheel. Along this direction, Koolen \emph{et al.}~\cite{Koolen2012} analyzed capturability for the three-dimensional LIP and its extensions with finite-sized feet and a reaction mass body. These analytical results have been later applied for synthesizing locomotion controllers for the M2V2 robot in practice~\cite{Pratt2012}, and later extended to other simplified models such as the variable height inverted pendulum model~\cite{Koolen2016,Caron2018,Liu2021}. Alternatively, Stephens~\cite{Stephens2007} proposed the so-called ankle, hip, and stepping strategies for humanoid push recovery. Later on, these strategies were combined with MPC and were tested using real-world force-controlled robots~\cite{Stephens2010,Stephens2011}. All above mentioned tools focus on bipeds in particular and have not been formally extended to other legged robotic systems. 

As one of the most important parts in push recovery controllers, footstep planning has attracted a considerable amount of research attention. Naturally, all capturability-based schemes mentioned previously provide a guide for footstep planning to enable successful push recovery. Automatic footstep adjustment with a LIP model has been proposed within an MPC framework and been deployed on bipedal robots~\cite{Diedam2008,Herdt2010}. These methods require pre-generated reference step locations that are commonly hard to obtain in practical push recovery scenarios. In addition, the simplified model used for footstep planning in these methods is a LIP with finite-sized support, which cannot fully characterize the point-feet and underactuated feature for quadrupedal locomotion. Scianca {\em et al.}~\cite{Scianca2020} developed an intrinsically stable MPC approach for generating footsteps and step timing for humanoids. Shifting focus from bipeds to quadrupeds, Boussema {\em et al.}~\cite{Boussema2019} proposed a feasible impulse set approach for footstep planning with quadrupeds. Given the estimated feasible impulse sets, upcoming contact locations and contact timings can be optimized. However, analyzing the feasible impulse set is computationally challenging, restricting its applicability to only planar considerations. So far, it remains an open problem to design reliable footstep planning schemes for quadruped push recovery.

Instead of looking for controllers designed exclusively for the task of push recovery, generic optimization-based legged locomotion controllers often demonstrate impressive capability in rejecting external disturbances. Convex MPC controllers use a linearly approximated single rigid body dynamics  model for quadrupedal locomotion and have shown an ability to generalize across a range of gaits~\cite{DiCarlo2018,Kim2019}. Building upon similar ideas, regularized predictive control frameworks that addresses nonlinearities in the original problem via convexifying the cost function have been designed and verified with hardware implementations~\cite{Bledt2017,Bledt2019,Bledt2020}. Leveraging a variational-based linearization technique, optimization-based strategies have demonstrated their ability to manage underactuated balance~\cite{Chignoli2020} and dynamic locomotion~\cite{Ding2021}. Winkler \emph{et al.}~\cite{Winkler2017} worked directly with the nonlinear problem and applied standard solvers to tackle the problem. Whole-body MPC~(e.g.,~\cite{Neunert2018}) is another promising approach that directly optimizes the quadruped's footsteps and center-of-mass (CoM) trajectories with the whole-body model in a single optimization problem. More recently, data-driven quadrupedal controllers utilizing reinforcement learning have also demonstrated remarkable locomotion capabilities ~\cite{Lee2020,Tsounis2020}. Albeit the impressive performance, these optimization-based or learning-based controllers serve a general purpose and behave less reliably in the specific task of push recovery.

\subsection{Contributions}

Despite the rich literature in capturability analysis for bipeds and locomotion controller design for legged robots, capturability analysis and push recovery has not been adequately investigated for quadrupeds. One of the major challenges lies in the lack of considering the time-varying nature of the underlying quadrupedal dynamics during locomotion. In addition, the underactuated point-contact feet for quadrupeds further complicate both system analysis and control synthesis. To address the above issues, this paper rigorously studies quadrupedal capturability and develops an associated push recovery planning scheme. The main contributions of this paper are summarized as follows. 

First, we propose to use switched systems to model quadrupedal locomotion and instantiate the abstract capturability concept to formally characterize quadrupedal capturability. Specifically, we leverage the control invariant set and backward reachable set notions in reachability theory to characterize the time-varying effects of changing contact configurations in quadrupedal capturability, which is the main distinction from classical bipedal capturability analysis approaches.

Second, we further investigate how quadrupedal capturability can be analyzed with practically available tools. We adopt the LIP as the simplified model and provide the associated switched system formulation with time-fixed gaits. With these switched LIP dynamics, we formulate quadrupedal capturability analysis as an EMPC problem and solve it via polytopic operations. By inspecting the quadrupedal capturability analysis result, we illustrate how the quadruped should react to external disturbances applied at different instances during the gait and we show the effects of different gait patterns on capturability. 

Third, we develop a push recovery motion planning scheme for the CoM trajectory and footstep sequence capitalizing on the quadrupedal capturability analysis result. Leveraging this offline obtainable result, real-time motion planning for push recovery can be decoupled into the planning of target footsteps and the planning of trajectories for the CoM and footsteps during push recovery. A quadratic programming scheme addressing the overall planning problem is devised, enabling efficient real-time implementation. This planning scheme is then integrated with an MPC controller to achieve real-time control of the quadruped. Extensive simulations and experiments demonstrate that the proposed push recovery strategy performs at least two-times better than the state-of-the-art methods. Simulations and experiments reveal how different gait patterns and impact timings affect the push recovery performance.   

\subsection{Overview of the Quadrupedal Capturability-based Push Recovery Framework}
 
\begin{figure*}[t!]
	\centering
	{\includegraphics[width=\linewidth]{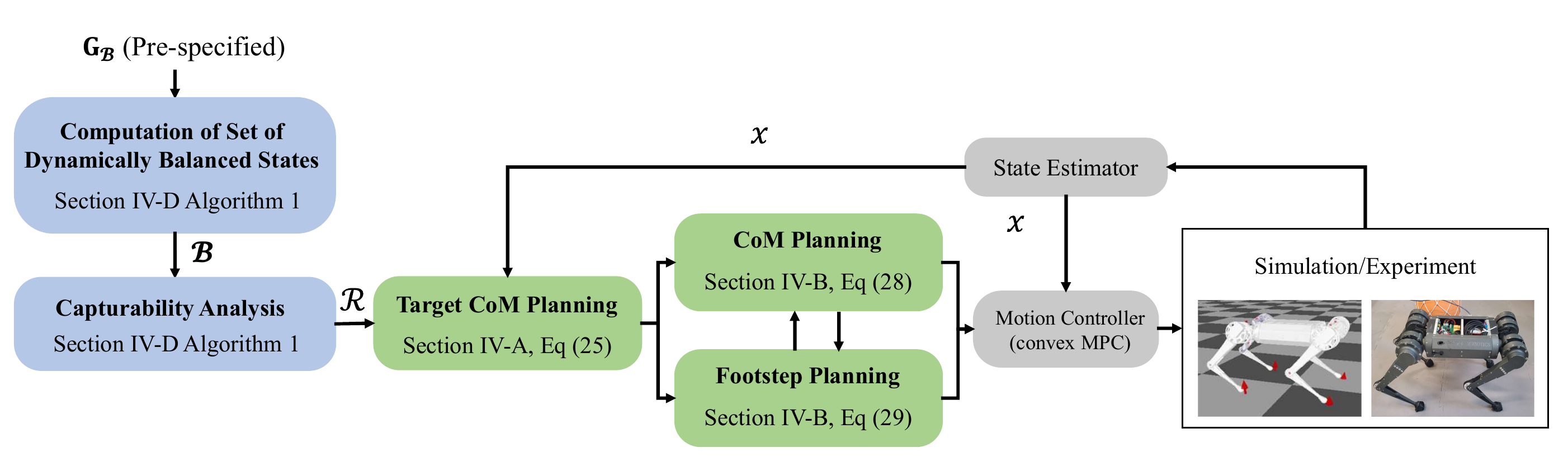}}
	\caption{ \footnotesize   Proposed push recovery framework. Blocks in blue run offline. Blocks in green run online at $\SI{30}{Hz}$. Blocks in gray run online at $500\SI{}{Hz}$.}
	\label{fig:over_arch}
\end{figure*}

 
Fig.~\ref{fig:over_arch} provides an overview of the proposed framework. Overall, the framework involves an offline quadrupedal capturability analysis module and an online optimization-based planning and control module. The offline quadrupedal capturability analysis instantiates the classical capturability concept with a time-switched system that models the quadrupedal locomotion with time induced gaits, and applies an EMPC approach to obtain the capturable states. During online execution, the proposed framework computes the target footsteps based on the offline capturability computations and the real-time estimated CoM state. The target footsteps are then used as a terminal constraint in the motion planning problem, in which the robot updates the sequence of upcoming footsteps and the CoM reference trajectory. Apart from the offline part, the proposed push recovery framework requires solving only a number of quadratic programs, authorizing efficient online implementation. The overall planning framework is then completed with a lower level convex MPC controller and a conventional state estimator.

\section{Switched System Modeling for Quadrupedal Locomotion and Quadrupedal Capturability}\label{sec:sys_qcap}

As mentioned before, push recovery controllers aim to find the discretely changing footsteps and continuous joint input torques for recovery from external pushes. For quadrupeds, the distinct features, such as multiple swinging legs and point-contact feet, prevent direct application of existing push recovery strategies that are mainly developed for bipeds. To address these issues, we first consider switched system modeling for quadrupedal locomotion and then instantiate viability theory for quadrupedal capturability based on the proposed switched system model.  

\subsection{Switched System Modeling for Quadrupedal Locomotion}
Various kinds of models can be used to design controllers for quadrupeds. They can all be viewed as special or simplified cases of whole-body multi-link rigid-body dynamics models. Throughout this paper, we focus on quadrupedal locomotion with time-fixed gaits. To this end, we adopt the following nonlinear switched system to describe the quadrupedal dynamics for notational convenience and theoretical generality.
\subeq{\label{eq:quad_sys}\al{\dot{\x}(t) &= f_{\sigmaa(t)} (\x(t),\u(t),\w(t)), \\ \x(t) &\in \X_{\sigmaa(t)},\ \u(t) \in \U_{\sigmaa(t)},\ \w(t)\in \W_{\sigmaa(t)} }}where $\x$ is the continuous state of the quadruped, $\u$ is the continuous input, $\w$ is the discrete input, $\sigmaa$ is the discrete mode denoting the gait information, $f_{\sigmaa(t)}$ is the underlying dynamics, $\X_{\sigmaa(t)}$, $\U_{\sigmaa(t)}$ and $\W_{\sigmaa(t)}$ are respectively the state, continuous input, and discrete input constraint sets, possibly dependent on the gait mode. The specific meaning of $\x$, $\u$, $\w$, $\sigmaa$ and $f$ depend on the underlying model used. 

\begin{remark}\label{rmk:impulse}
In the above model for quadrupedal locomotion, impulsive effects associated with foot touchdowns are neglected. Such a simplification has been widely accepted in most template models for legged locomotion~\cite{Pratt2006,Koolen2012,DiCarlo2018}. In fact, the impulsive effects can be incorporated into the generic formulation by introducing a state constraint specifying the touchdown condition and an associated state-reset map capturing the impulse. However, considering the impulsive effects requires carefully constructed conditions for identifying touchdown events, and this modeling strategy greatly complicates the overall problem. Hence, we consider the simplified cases to better deliver our main idea. The simulation and hardware experiments in Section~\ref{sec:validation} show the admissibility of this assumption.  
\end{remark}

Below we provide two widely adopted examples to better illustrate how the above switched system can be used for modeling quadrupedal locomotion.

\begin{example}[Single Rigid Body]
The single rigid body (SRB) model is commonly adopted for quadrupedal locomotion analysis and controller design~\cite{DiCarlo2018}. The state of this model consists of the position and orientation of the single rigid body as well as their velocities, i.e., $\x = (\c,\dot{\c},\Rc,\omegaa)$ where $\c \in \R^3$ is the position of the CoM of the body, $\Rc\in \so$ is a rotation matrix representing the orientation of the body and $\omegaa\in \R^3$ is the angular velocity of the body. The continuous input to the single rigid body model consists of the ground reaction forces  $\u = \ar{{cccc} \f_1^\top & \f_2^\top & \f_3^\top & \f_4^\top}^\top \in \R^{12}$ at the footstep locations $\w = \ar{{cccc} \p_1 & \p_2 & \p_3 & \p_4}\in \R^{3\times 4}$. The contact points serve as the discrete inputs to the SRB model that change from step to step. The dynamics of the model are given by \eq{\label{eq:srb}  \dot{\x} = \ar{{c} \dot{\c}\\ \ddot{\c} \\ \dot{\Rc} \\ \dot{\omegaa}}  =  \ar{{c} \dot{\c} \\ \frac{1}{m}\sum\limits_{ i=1}^{4} \In_i \f_i - \textbf{g} \\ \omegaa\times \Rc \\  \Ie^{-1} \Big(\sum\limits_{i=1}^4 \In_i (\p_i\!-\!\c)\!\times\! \f_i \!-\!\omegaa\times \Ie \omegaa\Big) }} where $m$ is the mass, $\Ie$ is the rotational inertia, $\textbf{g}$ is the gravitational acceleration vector, and $\In_i\in \{0,1\}$ is the indicator of whether the $i$-th leg is in contact that is encoded in the discrete mode signal $\sigmaa(t)$.

Consider the trot gait, for instance, where the diagonal legs $\{1,3\}$ and $\{2,4\}$ switch contacts alternatively. The associated SRB dynamics become\subeq{\nonumber \small  \al{\dot{\x} & =\case{ &\! \ar{{c} \dot{\c} \\ \frac{1}{m} (\f_1+\f_3) - \textbf{g} \\ \omegaa \times \Rc \\  \Ie^{-1} \Big( (\p_1\!-\!\c)\!\times \!\f_1 \!+\! (\p_3\! -\!\c)\!\times\! \f_3 \!-\!\omegaa \!\times\! \Ie \omegaa \Big)}\!, \text{if } \!\sigmaa = 1 \\ &\! \ar{{c} \dot{\c} \\ \frac{1}{m} (\f_2+\f_4) - \textbf{g} \\ \omegaa \times \Rc \\  \Ie^{-1} \Big( (\p_2\!-\!\c)\!\times \!\f_2 \!+\! (\p_4\! -\!\c)\!\times\! \f_4 \!-\!\omegaa \!\times\! \Ie \omegaa \Big)}\!, \text{if } \!\sigmaa = 2 }  }}
\end{example}

\begin{example}[3D Linear Inverted Pendulum Model~\cite{Kajita2001}]\label{exp:3dlip}
Similar to the previous example, the three-dimensional linear inverted pendulum (3D-LIP) model can be written in the form~\eqref{eq:quad_sys} as well. Denoting by $\c_{xy} = (c_x,c_y)\in \R^2$ the planar CoM position, and by $\p = (p_x,p_y) \in \S$ the center of pressure (CoP) with $\S$ being the support region, the 3D-LIP dynamics are given by
\subeq{\label{eq:3dlip} \al{ \dot{\x} =  \ar{{c}\dot{\c}_{xy} \\ \ddot{\c}_{xy}} = \ar{{c}\dot{c}_x \\ \dot{c}_y \\  \frac{\gr}{\hz} (c_x - p_x) \\ \frac{\gr}{\hz} (c_y - p_y)  } \label{eq:lip_eom} }} where $\gr=\SI{9.81}{m\per s^2}$ is the gravitational acceleration constant, and $\hz$ is the fixed height  of the CoM.

Note that, the 3D-LIP dynamics remain the same for different modes during locomotion, while the input constraints change over time depending on the contact situation. Given the footstep locations $\w = \ar{{cccc} \p_1 & \p_2 & \p_3 & \p_4}\in \R^{3\times 4}$, the CoP can be obtained by 
\eq{ \label{eq:sigma} \ald{ & \p = \sum\limits_{i=1}^4  \lambda_i  \p_i, \ (\lambda_1,\ldots, \lambda_4) \in \S_{\lambda}  \text{ with } \\ &   \S_{\lambda}   \triangleq \Big\{ (\lambda_1,\ldots, \lambda_4) \in [0,1]^4~|~ \sum\limits_{i=1}^{4}\In_i \lambda_i = 1 \\ & \qquad \qquad \qquad \qquad \qquad \text{ and } \lambda_i= 0,  \text{ if } \In_i= 0. \Big\}}}
Taking $\u = \ar{{cccc}\lambda_1 & \lambda_2 & \lambda_3 &  \lambda_4}^\top \in \R^4$ as the continuous input, the 3D-LIP dynamics for the trot gait are given below
\eq{ \dot{\x} =   A_\text{LIP} \x +B_\text{LIP} \w \u,   \ \u \in \S_{\lambda}, \ \sigmaa = 1,2 \label{eq:ss_lip}} where \eq{A_\text{LIP} = \ar{{cccc} 0 & 1 & 0 & 0\\ 0 & 0 & 0 & 1 \\ \frac{\gr}{\hz} & 0 & 0 & 0\\ 0 & 0 & \frac{\gr}{\hz} & 0},\ \ B_\text{LIP} = \ar{{ccc} 0 & 0 & 0\\ 0 & 0 & 0\\ - \frac{\gr}{\hz} & 0 & 0  \\ 0 & -\frac{\gr}{\hz} & 0}\label{eq:AB_LIP}} are the system matrices associated with the LIP model. 
\end{example}

Given the above switched system modeling, push recovery controller design for quadrupeds can be directly formulated as a switched system optimal control problem. However, directly working with the complex switched dynamics is still practically challenging. Capturability-based hierarchical schemes have been extensively studied and applied to address the push recovery problem. In the following, we will first review the basics in conventional capturability and then instantiate this abstract concept with the above introduced switched system to yield the quadrupedal capturability notion.
 
\subsection{Preliminaries on Classical Capturability for Legged Robots}\label{sec:pre_cap}
 
Capturability originated from the viability concept for general nonlinear systems~\cite{Aubin1990,Aubin2011} and was later extended to study legged locomotion~\cite{Wieber2002}. It concerns with robots' ability to maintain balance and recover from external pushes. The original definition of viability for legged locomotion is given below.

\begin{definition}[Viability of a Legged Robot~\cite{Wieber2002,Pratt2006b}]\label{def:viab_leg}
A robot state $\x$ is {\em viable} if there exists at least one state trajectory starting with $\x$ that never reaches a prescribed set of fallen states $\F$. The {\em viability kernel} is defined as the set consisting all viable states.
\end{definition}

The above definition is quite abstract. Even when the set of fallen states $\F$ can be precisely defined, there are limited analytical or numerical tools for analyzing viability and designing corresponding viability-based fall prevention strategies. To address these issues, the seminal capture point concept was introduced, which is pivotal in the history of push recovery analysis and controller synthesis for legged robots. 

\begin{definition}[Capture Point~\cite{Pratt2006b}]\label{def:capturepoint}
A capture point is a point on the ground, such that if the robot's center-of-pressure (CoP) is maintained to coincide with the capture point, then the robot can be controlled to a stop asymptotically. 
\end{definition}
 
Leveraging the capture point concept, the problem of checking whether a robot state is recoverable without taking step (i.e., it is $0$-step capturable) becomes one of checking if the robot's support region (footprint) covers a capture point. However, capture points for generic legged models are very challenging to compute in general. Thus, simplified models such as the 3D-LIP model described in~\eqref{eq:3dlip} have been widely adopted in capturability analysis. For the 3D-LIP model, the associated (instantaneous) capture point (ICP)~\cite{Koolen2012}, denoted by $\xxi$, is simply given by
\eq{\label{eq:icp} \xxi = \c_{xy} +\sqrt{\frac{\hz}{\gr}} \dot{\c}_{xy}. }
Once the ICP is computed, the basic idea of many push recovery controllers is to take a step so that the robot's support region covers the ICP.

\begin{figure}[tp!]
	\centering
	{\includegraphics[width=0.4\linewidth]{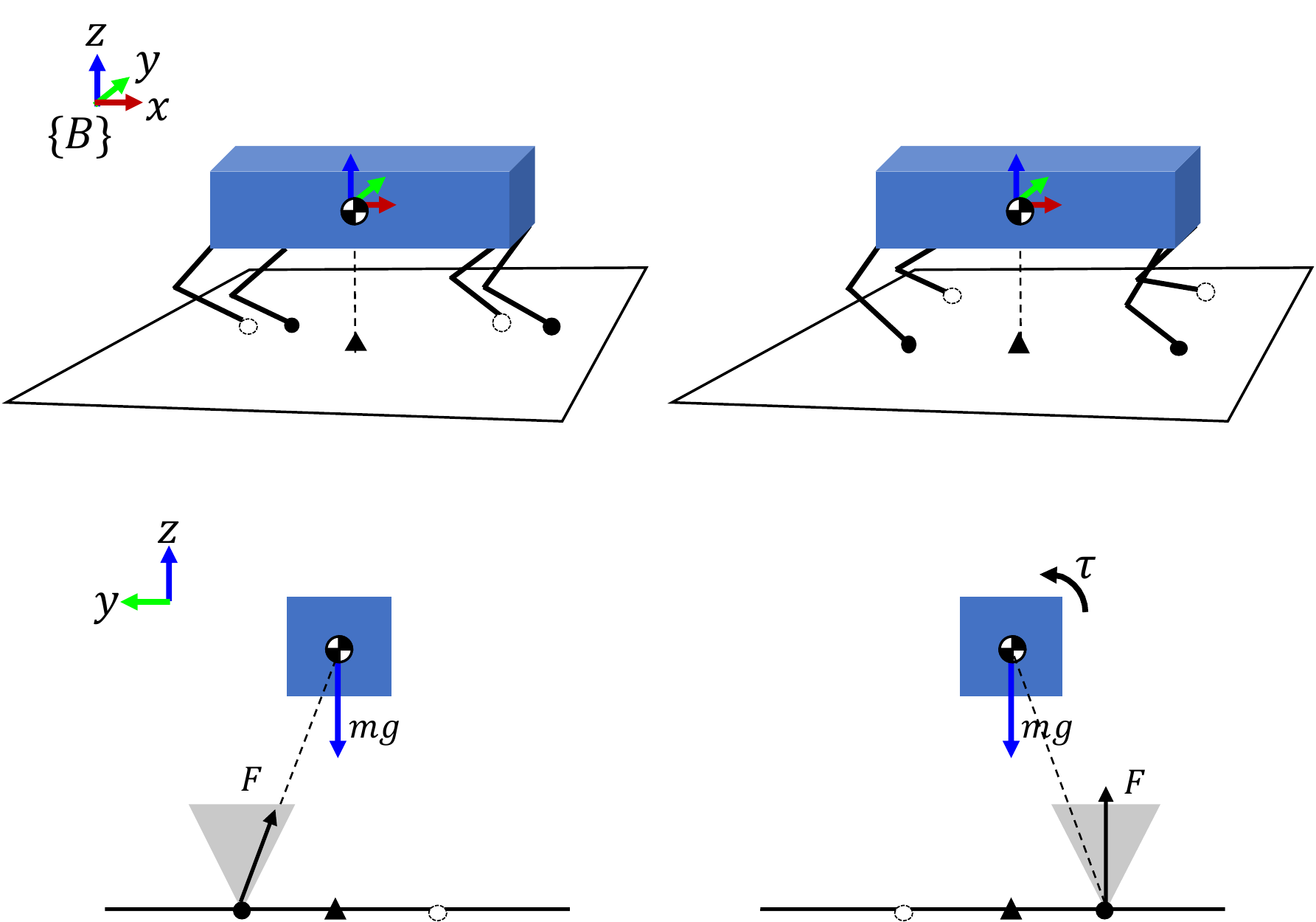}}
	\caption{ \footnotesize   When the quadruped is commanded to pace in place, it is in general impossible to maintain a static balance. Top: 3D illustrations. Bottom: Rear-front view. Due to the contact feature of pace gait, the contact force will either create a linear momentum change (bottom left) or create an angular momentum change (bottom right) or both.}
	\label{fig:pace_issue}
	\vspace{-10px}
\end{figure}

\subsection{Quadrupedal Capturability}

To extend capturability to quadrupeds, it is important to note that most existing capturability analyses explicitly or implicitly rely on the following definition of {\em static balance}.

\begin{definition}[Static Balance]\label{def:stat_bal}
Given a gait signal $\sigmaa(\cdot)$ and the associated discrete footsteps $\w(\cdot)$, a quadrupedal state $\x_b$ is considered in {\em static balance} if it is able to stay at rest at a controlled equilibrium, i.e., if there exists a feasible control signal $\u(\cdot)$ such that
\eq{\label{eq:controlledeq}\dot{\x}_b = f_{\sigmaa(t)}(\x_b,\u(t),\w(t)) = 0, \forall t.}
\end{definition}

However, this definition can be rather restrictive for characterizing stability for quadrupedal locomotion. To see this, consider the example shown in Fig.~\ref{fig:pace_issue} where a quadruped is pacing in place. If the robot is initialized with static balance, then the capture point specified in Definition~\ref{def:capturepoint} is simply given by the projection of the CoM onto the ground. In this case, it is impractical for the quadruped to maintain its support line segments to cover the capture point. Moreover, any feasible contact force corresponding to the pace gait would either result in a change in linear momentum or angular momentum or both, making it difficult to maintain static balance according to Definition~\ref{def:stat_bal}. Therefore, a more general ``balance'' definition incorporating more dynamic behaviors is needed.  

\begin{figure*}[h]
    \centering
    \subfigure[ \footnotesize   Static balance]{
    \includegraphics[width=0.31\linewidth]{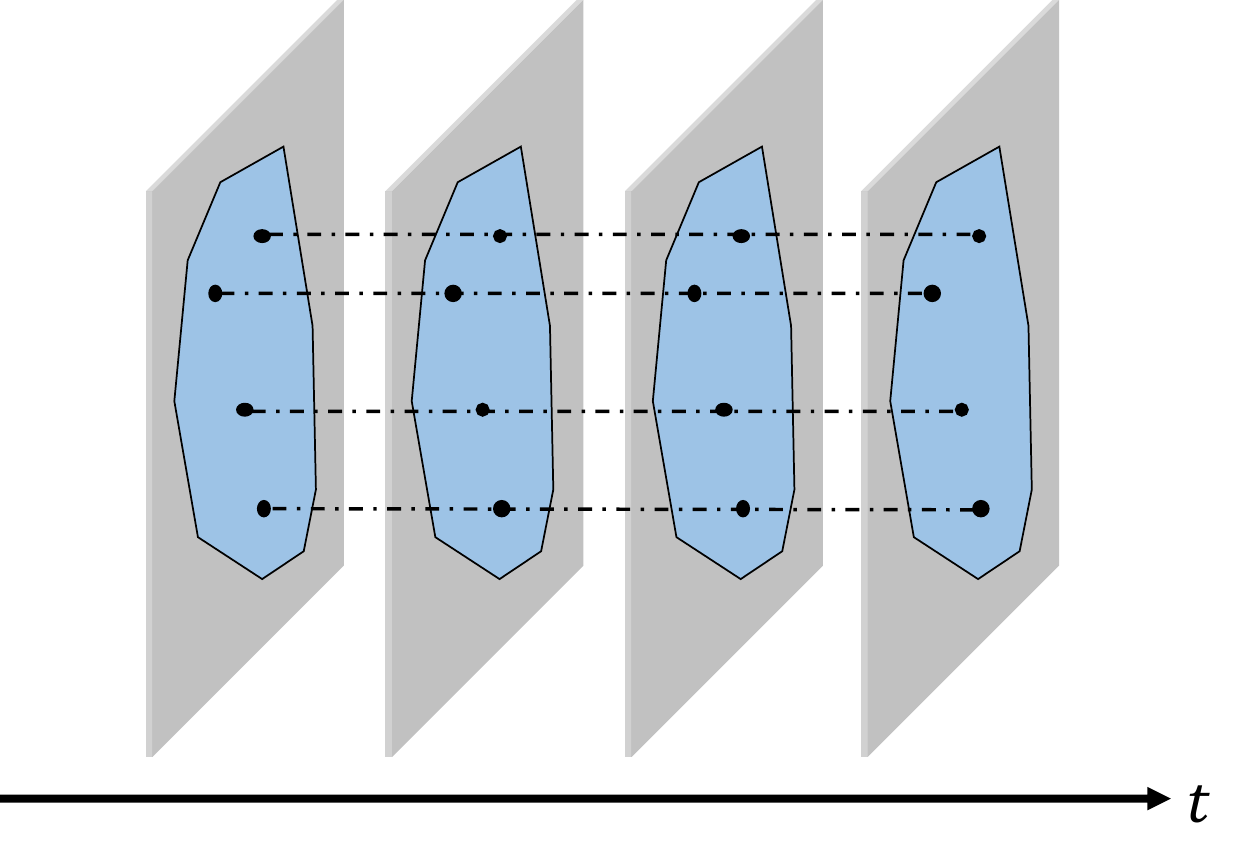}}
    \subfigure[ \footnotesize    Time-invariant dynamic balance]{
    \includegraphics[width=0.31\linewidth]{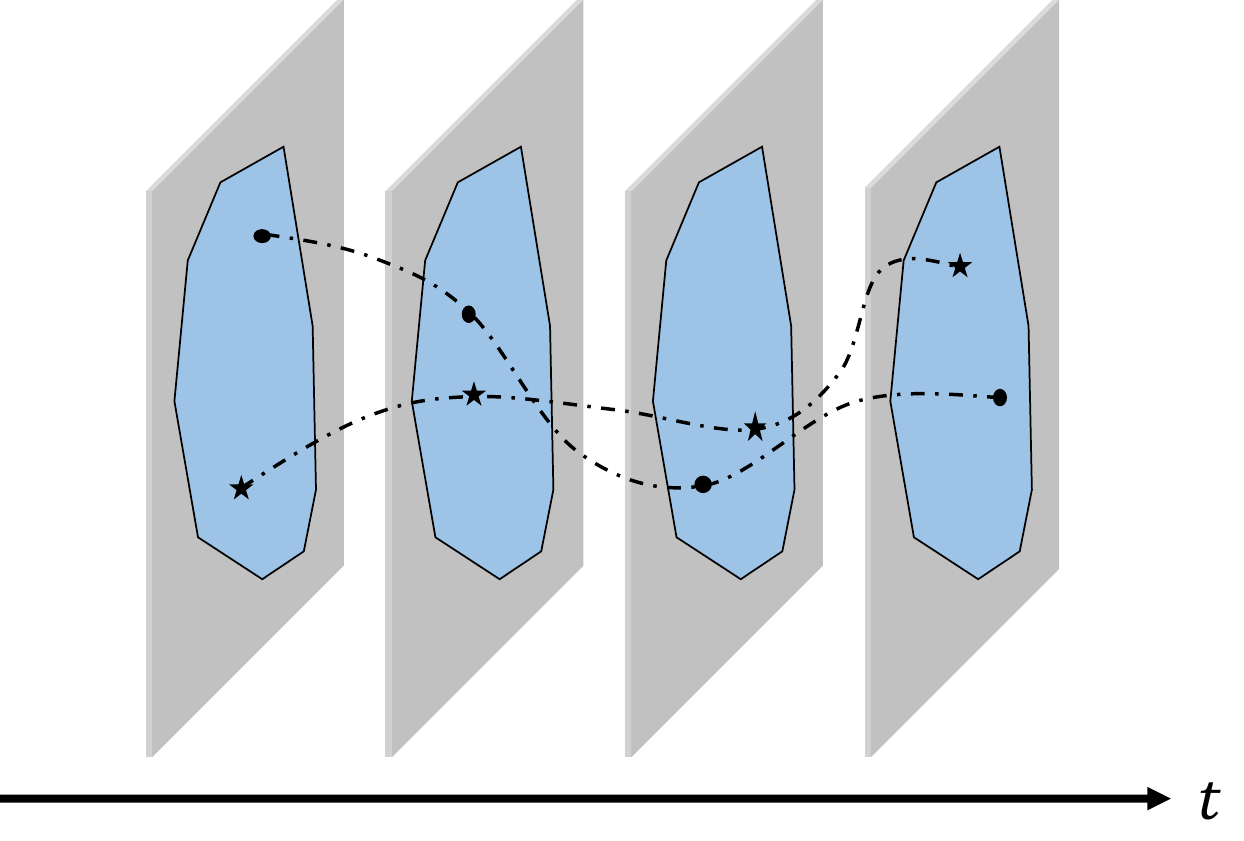}}
    \subfigure[ \footnotesize    Time-varying dynamic balance]{
    \includegraphics[width=0.31\linewidth]{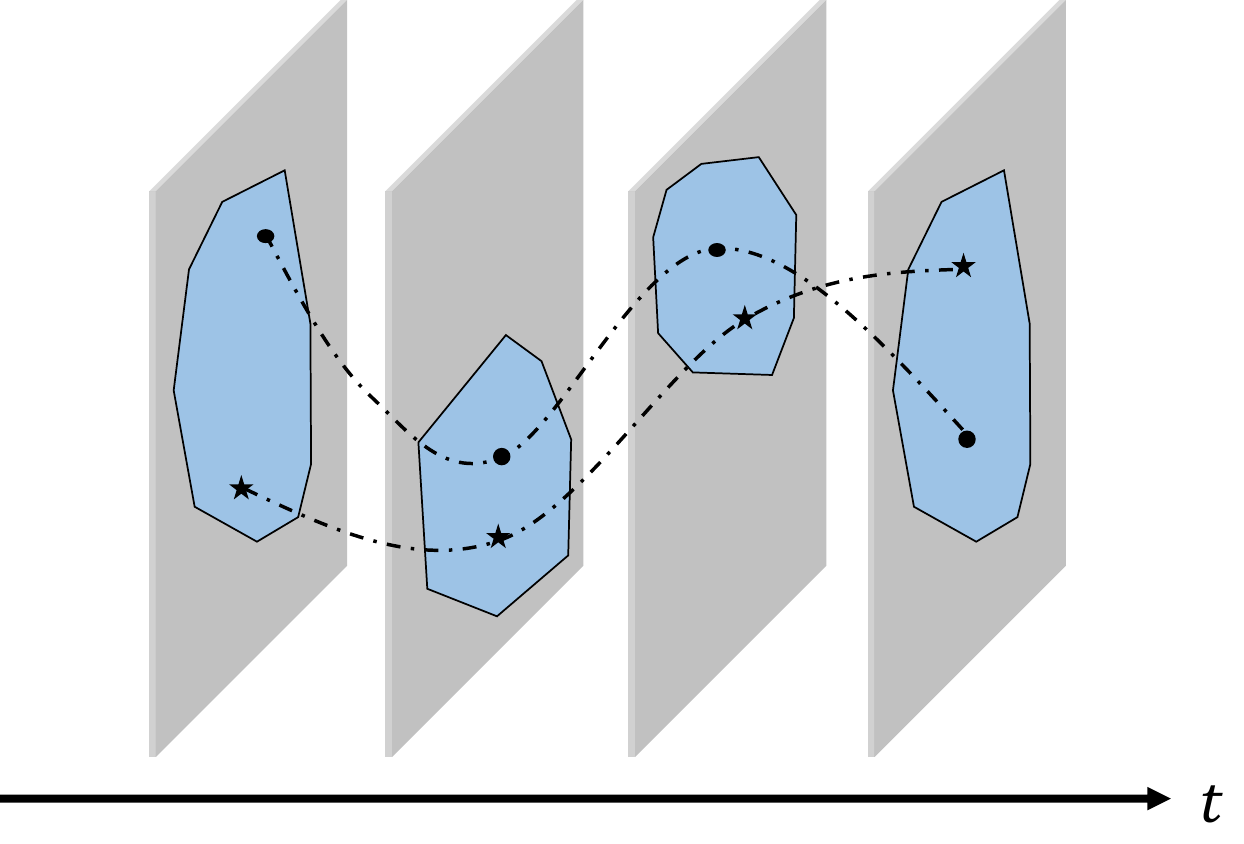}}
    \caption{ \footnotesize    Comparison between different specifications on balance. Grey area represents the state space at each time slice, blue region with black boundary denotes the set (slice of tube) of balanced states, black circles and stars denote states along particular trajectories shown as dash-dotted black lines. }
    \label{fig:balanceset}
    \vspace{-10px}
\end{figure*}

The fundamental idea of our generalization lies in the consideration of \emph{dynamic balance}. Roughly speaking, a quadrupedal state is considered in {\em dynamic balance} if its evolution can be controlled to have relatively small linear and angular momenta with appropriate control indefinitely. This generalization allows the quadruped state to move during balance, but in a controlled and restrictive manner. To mathematically formalize this intuitive description, we introduce the following definition of \emph{dynamic balance}.

\begin{definition}[Dynamic Balance]\label{def:balanceset}
Consider a gait signal $\sigmaa(\cdot)$, a fixed footstep sequence $\w(\cdot)$ over time, and a prescribed target region $\X_T$ where the balanced state must reside. A quadrupedal state $\x_{db}$ is considered in {\em dynamic balance} if it falls inside the set of dynamically balanced states $\B \subset \X_T$ defined as follows. For any $\x_\text{ini} \in \B$, there exists a feasible control signal $\u(\cdot)$ such that 
\eq{\forall t\ge t_\text{ini},\  \phi(t;t_\text{ini},\x_\text{ini},\u(\cdot),\w(\cdot))\in \B, } where $\phi(t;t_\text{ini},\x_\text{ini},\u(\cdot),\w(\cdot))$ is the state trajectory of the dynamics~\eqref{eq:quad_sys} starting from $\x_\text{ini}$ at $t_\text{ini}$ with gait signal $\sigmaa(\cdot)$, footsteps $\w(\cdot)$ and input signal $\u(\cdot)$.
\end{definition}

\begin{remark}\label{rmk:dyn_bs}
It is worth mentioning that the set of all dynamically balanced states itself is time-invariant (i.e., does not change over time). However, there is a fundamental difference between this time-invariant set and the set of all states that are in static balance for classical capturability (Definition~\ref{def:stat_bal}). Specifically, static balance requires not only the set to be time-invariant, but also the quadrupedal state to be not changing over time. On the other hand, any particular quadrupedal state inside the set of dynamic balance may change over time even though the set itself remains time-invariant. Fig.~\ref{fig:balanceset} provides a graphical illustration on the above mentioned difference.
\end{remark}

A single time-invariant set describing dynamic balance is overly restrictive for characterizing the concept of balancing for dynamic quadrupedal locomotion. For example, the robot's state may leave the set at some point during the gait and then re-enter it later on. To account for this issue, we introduce a more complete concept that considers a collection of sets over time (i.e., tube) for characterizing dynamic balance. 

Denoting by $T_G$ the gait period, we consider the tube of dynamically balanced states to be a collection of sets over time $\Bx = \{\B_t\}_{t=0}^{T_G}$. Each set $\B_t$ in the tube contains the quadrupedal states that can be maintained inside the tube when starting at time $t$. The formal definition of this tube is given below.

\begin{definition}[Tube of Dynamically Balanced States]\label{def:tdbs}
Given a gait period $T_G$, gait signal $\sigmaa(\cdot)$, the footsteps $\w(\cdot)$ over time, and the target region $\X_T$, the collection of sets $\Bx=\{\B_t\}_{t=0}^{T_G}$ is called a tube of dynamically balanced states if
\begin{enumerate}
    \item Every slice of the tube belongs to the target region, i.e., $\B_t\in \X_T$ for all $t$
    \item There exists admissible control to keep the trajectory inside the tube if the initial state is inside the tube, i.e.,$\exists \u(\cdot)$ over $[t_\text{ini}, T_G]$, such that $\phi(t;t_\text{ini},\x_\text{ini},\u(\cdot),\w(\cdot))\in \B_t, \ \forall t\in [t_\text{ini},T_G]$ if $(t_\text{ini},\x_\text{ini})  \in [0,T_G]\times \B_{t_\text{ini}} $, 
    \item The tube is periodic, i.e., $\B_0 = \B_{T_G}$ 
\end{enumerate}

 \end{definition}
 
\begin{remark}\label{rmk:controlinvariance}
A graphical illustration on the subtle difference between the tube of dynamically balanced states and the set of dynamically balanced states is provided in Fig.~\ref{fig:balanceset}. From the set-theoretic perspective, the concept of dynamic balance is equivalent to the control invariance property of the underlying quadrupedal dynamics~\cite{Blanchini1999,Blanchini2008}.  
\end{remark}

In the sequel, we will use the term {\em tube} to characterize dynamic balance in a unified way and will simply call a quadrupedal state $\x$ in dynamic balance (or belongs to the tube), if there exists a $t\in [0,T_G]$ such that $\x\in \B_{t}$. With all the above setups, quadrupedal capturability with dynamic balance can be formally defined as follows. 

\begin{definition}[Quadrupedal Capturability]
Given the gait signal $\sigmaa(\cdot)$, a quadrupedal state $\x$ is capturable if there exists a sequence of feasible footsteps over time $\w(\cdot)$ and feasible continuous input $\u(\cdot)$ such that the quadrupedal state eventually achieves dynamic balance, i.e.,
\eq{\exists t>0, \text{ such that } \phi(t;0,\x,\u(\cdot),\w(\cdot)) \in \B_t. }
In the sequel, we denote by $\Ca(t;\Bx)$ the set of capturable states with fixed time horizon $t$ and denote by $\Cs(\Bx)_T = \left\{ \Ca(t;\Bx) \right\}_{t=0}^T$ the collection of sets (tube) of capturable states over the entire horizon. 
\end{definition}

So far, we have generalized conventional capturability to quadrupedal scenarios via considering dynamic balance. Such a generalization accounts for the influences on capturable states attributed to different gait patterns and explicitly characterizes how the capturable states evolve over time. These two new features allow capturability analysis for quadrupedal locomotion with various time-fixed gaits.

\section{Analysis of Quadrupedal Capturability}\label{sec:analysis_QC}

In this section, we further show how quadrupedal capturability discussed in the above section can be analyzed with the concrete 3D-LIP model in Example~\ref{exp:3dlip}.

Given any periodic gait with period $T_G$ and the associated footsteps denoted by $\{\w_{\sigmaa_k} \}_{k=1}^{T_G}$, we take the 3D-LIP model as an example and consider the following discrete-time version of its dynamics    \vspace{-10px}
\eq{\x_{k+1} = \Aa\x_k + \Bb \w_{\sigmaa_k} \u_k, \ \u_k\in \U \label{eq:dlip}\vspace{-5px}}
where $\Aa = e^{A_\text{LIP} \dt}$ and $\Bb = A_\text{LIP}^{-1}(e^{A_\text{LIP} \dt} - I)B_\text{LIP}$ are the discretized system matrices obtained from the continuous-time LIP model~\eqref{eq:AB_LIP}.

The first step in quadrupedal capturability analysis is to construct the tube of states $\Bx$ that are in dynamic balance. Construction of such a tube relies heavily on the fact that it is a control-invariant tube of the underlying dynamics in essence.

Explicit model predictive control (EMPC) is an important topic in the MPC community, which focuses on solving an optimal control (MPC) problem offline to obtain an explicitly function representing the optimal control law. Moreover, solving the EMPC problem leads to an optimal cost function (a.k.a. value function) that encodes information about the set of all feasible states as well as the quality of a particularly given state. We leverage these two features associated with EMPC to develop computationally tractable tools for analyzing quadrupedal capturability with 3D-LIP model.

To analyze capturability for quadrupedal locomotion, we consider the target region $\X_T$ as a orthotope (i.e., hyper-rectangle) for simplicity, given by\vspace{-5px}
\eq{\ald{ \X_T &= \Big\{ \x~|~ c_x \in [\underline{c}_x, \overline{c}_x], c_y \in [\underline{c}_y, \overline{c}_y], \dot{c}_x \in [\underline{\dot{c}}_x, \overline{\dot{c}}_x],\\ & \qquad \qquad   \dot{c}_y \in [\underline{\dot{c}}_y, \overline{\dot{c}}_y], \Big\} = \left\{\x~|~ H^\top\x\le h^\top\right\}.\label{eq:stateset}} } 

Given this target region, optimal solution to the following problem characterizes the tube of dynamically balanced states for the underlying model~\eqref{eq:dlip}

\subeq{\label{eq:empc_cit}\al{ V^N_\B(\x_0;t) =  \\ \min\limits_{\{\u_k\}_{k = t}^{N+t-1} }  & \left( \sum\limits_{k = t}^{N+t-1} \norm{\x_k}^2_\Qc    + \norm{\u_k}^2_\Rc \right)+ \norm{\x_N}^2_{\Qc_F} \\  \st \quad  & \x_{k+1} \!= \!\Aa\x_k\!+\!\Bb\w_{\sigmaa_k} \u_k,  k\!=\!t\!:\!N\!+\!t\!-\!1, \label{eq:empc_dyn} \\ & \x_k\! \in\!  \X_T ,\ \u_k\! \in \! \U,\qquad \quad    k\!=\!t\!:\!N\!+\!t\!-\!1, \label{eq:empc_state_input} \\&  \x_N \!\in\! \X_T. \label{eq:empc_ter}   }    \vspace{-10px}
}

Conventionally, solution to the above optimal control problem refers to the optimal control sequence (or law) $\{\u^*_k\}_{k=0}^{N-1}$. For the purpose of capturability analysis, we take the EMPC perspective on the optimal control problem and think of solution to the above optimal control problem as the associated value function $V^N_\B(\x;t)$ over $\x\in \X$. Such a value function encodes information about all feasible initial states to the optimal control problem as the set $\X^N_{\B,t} = \{\x ~|~ V^N_\B(\x;t) < \infty\}$. More importantly, this set approximates a slice of the desired tube of dynamically balanced states $\B_t$. If $N$ is large enough such that $V^N_\B(\x;t) = V^{N+1}_\B(\x;t)$ holds, then the set $\X^N_{\B,t}$ gives the exact result. 

\begin{remark}\label{rmk:inf_mpc}
In practice, to obtain the overall tube of dynamically balanced states $\Bx$, it is common to first solve the above EMPC problem with a relatively large horizon $N$ and then analyze solution with selected tolerance level. Examples showing the balanced states solved from the EMPC problems are given later in this section (Section~\ref{sec:analysis_QC_res}).

\end{remark}

Once the sets of balanced states are obtained, they are used as the terminal constraints in the problem of computing capturable states. Then, the tube of capturable states can be computed analogously, via working with the following optimal control (EMPC) problem.  
\subeq{\label{eq:empc_brs}\al{  V^N_\Ca(\x_0;t)  =\\ \min\limits_{\{\u_k\}_{k = t}^{N+t-1} }  & \left( \sum\limits_{k = t}^{N+t-1} \norm{\x_k}^2_\Qc    + \norm{\u_k}^2_\Rc \right)+ \norm{\x_N}^2_{\Qc_F} \\  \st \quad  & \x_{k+1} \!= \!\Aa\x_k\!+\!\Bb\w_{\sigmaa_k} \u_k,  k\!=\!t\!:\!N\!+\!t\!-\!1, \label{eq:empc_dyn} \\ & \x_k\! \in\!  \X,\ \u_k\! \in \! \U,\qquad \quad    k\!=\!t\!:\!N\!+\!t\!-\!1,   \label{eq:empc_brs_kunning}\\& \x_N \in \B_t.   \label{eq:empc_brs_terminal}}}

The key differences between~\eqref{eq:empc_cit} and~\eqref{eq:empc_brs} lie in the state constraint in~\eqref{eq:empc_brs_kunning} and the terminal state constraint in~\eqref{eq:empc_brs_terminal}. When determining the tube of dynamically balanced states, all states need to be constrained within the terminal region in which the CoM's linear momentum is limited. Hence, the running and terminal state constraint sets both are $\X_T$. On the other hand, the running state constraint imposed on the evolution of state trajectory is simply the general state constraint ($\x_k \in \X$), while the trajectory eventually needs to enter the tube of balanced states, leading to the constraints in~\eqref{eq:empc_brs}.

Thanks to the linear-quadratic nature of the optimal control problems, the EMPC problems~\eqref{eq:empc_cit} and~\eqref{eq:empc_brs} can be efficiently solved offline with, e.g., multi-parametric quadratic programming (mpQP) based approaches~\cite{Bemporad2002,Baotic2005} or dynamic programming based approaches~\cite{Borrelli2005}. Moreover, the resulting value functions to the above two EMPC problems encode performance properties of balance and capturability that are summarized in the following theorem. 

\begin{theorem}\label{thm:empc_pro}
Given a fixed horizon $T$, feasible sets solved from the EMPC problems~\eqref{eq:empc_cit} and~\eqref{eq:empc_brs} are both polytopic with the following representation
\subeq{\label{eq:critical_region}\al{\label{eq:tube_balance}\X_{\B,i}^N & = \left\{\x ~|~ V_{\B,i}^N(\x)< \infty \right\} = \left\{\x ~|~ \Hh^{\Bx}_i \x \le \h^{\Bx}_i \right\}, \\ \X_{\Ca,i}^N  & = \left\{\x ~|~ V_{\Ca,i}^N(\x)< \infty \right\} = \left\{\x ~|~ \Hh^\Ca_i \x \le \h^\Ca_i \right\}\label{eq:reachset}}}
for some $\Hh^{\Bx}_i$, $\h^{\Bx}_i$, $\Hh^\Ca_i$ and $\h^\Ca_i$ obtained from the EMPC computation. 

Furthermore, both value functions $V_{\B,i}^N(\x)$ and $V_{\Ca,i}^N(\x)$ are continuous, convex and piecewise quadratic over polyhedral partitions of $\X_{\B,i}^N$ and $\X_{\Ca,i}^N$, i.e.,
\eq{\label{eq:pwq}V_{B,i}^N(\x) =  \ar{{c} \x\\ 1}^\top  \Pf_j \ar{{c} \x\\ 1}, \text{ if } \Hh_{i,j}^{\Bx} \ar{{c} \x\\ 1} \le \h_{i,j}^{\Bx},} 
\eq{\label{eq:pwq2}V_{\Ca,i}^N (\x) = \ar{{c} \x\\ 1}^\top \Pf_j \ar{{c} \x\\ 1}, \text{ if } \Hh_{i,j}^\Ca \ar{{c} \x\\ 1} \le \h_{i,j}^\Ca. } for some $ \Hh_{i,j}^{\Bx}$, $\h_{i,j}^{\Bx}$, $\Hh_{i,j}^\Ca $ and $\h_{i,j}^\Ca$ obtained from the EMPC computation. 
\end{theorem}

This theorem follows directly from standard results in EMPC literature~\cite{Bemporad2002,Borrelli2005}. More importantly, this property authorizes us to incorporate capturability performance analysis into the subsequent push recovery controller design. Detailed numerical results regarding the computation will be provided later in Section~\ref{sec:validation}.

\subsection{A Set-theoretic Perspective on EMPC based Capturability Analysis}

In this subsection, we provide additional discussions on the underlying tools used within EMPC computation from a set-theoretic perspective. Such a viewpoint explicitly reveals the underlying operations on the sets of interest. 

Overall, the main operation used in analyzing backward reachability is the following $\Pre$ operator on sets:

\eqn{\Pre_k(\X) = \left\{\x~|~ \exists \u\in \U, \text{ such that } \Aa\x+\Bb\sigmaa_k \u\in \X. \right\}} Intuitively speaking, the $\Pre$ operator constructs the set of system states that can be driven to set $\X$ in a single step using a feasible control input under the given linear dynamics. Suppose the set $\X$ is polyhedral, then the $\Pre$ operator admits a closed-form representation with the Minkowski addition and the linear mapping with polyhedra. The Minkowski addition of two generic sets $\P$ and $\Q$, commonly denoted by $\oplus$, is given by 
\eqn{\P \oplus \Q = \left\{x = p+q ~|~ p\in \P, \ q \in \Q \right\}.} Moreover, for a polyhedral set $\P$ with representation $\P = \{x~|~ H x\le h\}$, we define two linear mappings $A\circ \P$ and $\P\circ A$ in the following form
\eqn{\ald{ \P \circ A & =  \{x~|~ HAx\le h \}, \\ A\circ \P & = \{Ax~| ~x\in \P \}.}}
With the above two operations, it has been shown in~\cite{Borrelli2017} that 
\eq{\label{eq:pre} \Pre_k(\X) = \Big(\X \oplus (-\Bb\sigmaa_k \circ  \U )\Big) \circ \Aa  }
  
Now, computing the tube of dynamically balanced states and capturable states can be simply achieved by sequentially applying the $\Pre_k$ operator backward in time. Due to the periodic nature of the underlying switched LIP dynamics, we further consider the following $\PRE$ operator that encodes how the set of states evolves backward over periods 
\eq{\label{eq:bigpre} \PRE =  \Pre_1   \circ \ldots \circ \Pre_{T_G}.} 

Fig.~\ref{fig:pre} gives a graphical illustration on how the $\Pre_k$ and $\PRE$ operators behave. With these operators, the main procedures for constructing the desired sets are summarized below in Algorithm~\ref{alg:cit}.

\begin{figure}[tp!]
	\centering
	{\includegraphics[width=0.5\linewidth]{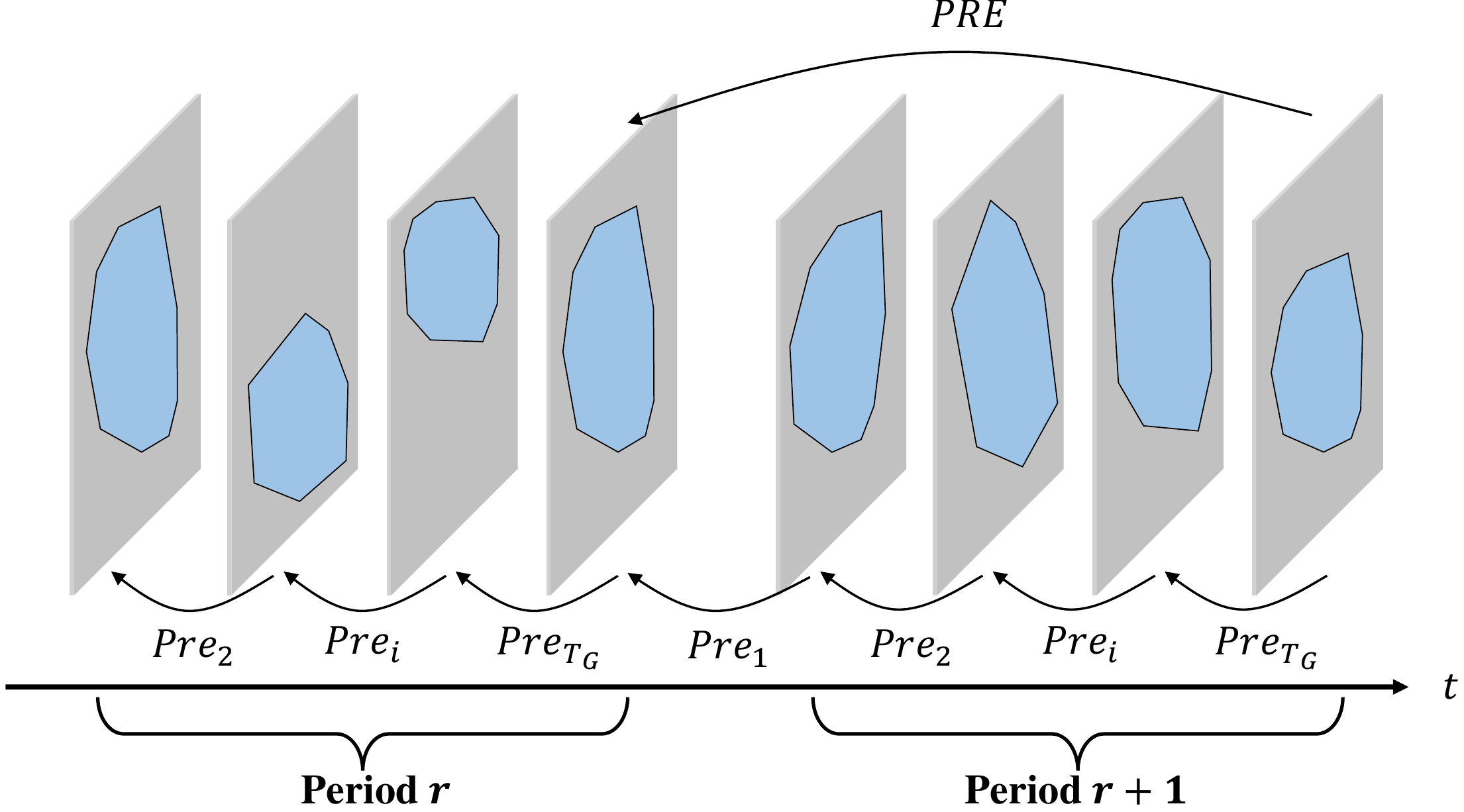}}
	\caption{ \footnotesize   Illustration on the $\Pre_k$ and $\PRE$ operators.  }
	\label{fig:pre}
	\vspace{-10px}
\end{figure}

By adopting the $\PRE$ operator, the above algorithm can be viewed as being applied to a linear time-invariant system characterizing the periodic state evolution of the original system. To this end, let $\v = \ar{{cccc} \u_{T_G}^T &\u_1^T  & \ldots & \u_{T_G-1}^T}^T$ be the stacked inputs taking values in $\U^{T_G}$, and let 
\eq{\label{eq:peri_evo} \ald{ \x^+ &= \Aa^{T_G}\x + [\Aa^{T_G-1}\Bb\sigmaa_{T_G}\   \ldots\  \Bb\sigmaa_{T_G-1}] \v \\ & = \bar{\Aa}\x +\bar{\Bb} \v}} be the state evolution between two consecutive periods. With these notations, the $\PRE$ operator can be understood as 
\eqn{\PRE(\X) = \{\x~|~\exists \v\in \U^{T_G}, \text{ s.t. }\x^+ = \bar{\Aa}\x +\bar{\Bb}\v\in \X \}.} 
Then the following theorem summarizes the convergence property of the above Algorithm. 

\begin{theorem}\label{thm:pre_con}
If there exists a positive integer $K$ such that for all $k\ge K$ the sets 
\eq{\label{eq:gamma}\ald{\Gamma_k(\X) & = \Gamma(\PRE^{k-1}(\X)) \text{ with }\\
\Gamma(\X) &= \Big\{(\x,\v)~|~ \x\in\X, \v\in \U^{T_G},  \bar{\Aa}\x +\bar{\Bb}\v\in \X  \Big\}}}
are nonempty, then 
\eq{\emptyset\neq \PRE^*(\X) = \bigcap_{k=1}^\infty \PRE^k(\X). }
\end{theorem}
Proof of this theorem follows simply from the seminal works such as~\cite{Bertsekas1972}, and is provided in Appendix~\ref{app:pre_con}.

\begin{remark}\label{rmk:convergence}
A sufficient condition ensuring the conditions in Theorem~\ref{thm:pre_con} is the state constraint set $\X$ and admissible input set $\U$ are polyhedral, bounded, and the system dynamics are continuous~\cite{Borrelli2017}. For our problem, all these requirements are satisfied. 
\end{remark}
\begin{algorithm}[bp!]
  \caption{ Computation of tube of balanced states}\label{alg:cit}
  \KwIn{$\Aa$, $\Bb$, $ \U $, $\X_T$, $N$;}
  \KwOut{$\Bx$.} 
   $\Omega_0 \leftarrow \X_T$, and $k \leftarrow 0$\;
  \While{ $\Omega_{k}\neq \Omega_{k-1}$ and $k\le N-1$}{
   $\Omega_{k+1} = \PRE(\Omega_k)\bigcap \Omega_k$\;
   $k \leftarrow k+1$\;
  }
   $\B_{T_G} = \Omega_{k}$\;
   $\B_t = \Pre_{t+1}\circ\ldots \circ \Pre_{T_G} (\B_{T_G})$ for all $t = 1,\ldots, T_G-1$.
\end{algorithm}
Similar to the EMPC based approach for computing the tube of capturable states, if the footsteps $\w$ are given, computation of the tube of capturable states can be accomplished with previous set-operation based approach as well. Specifically, this approach gives an inner-approximation of the exact set of capturable states due to the restrictions on the footsteps. Algorithm~\ref{alg:brs} describes the main procedures of generating the set of all feasible initial conditions associated with the EMPC problem~\eqref{eq:empc_dyn}.

\subsection{Results of Quadrupedal Capturability Analysis}\label{sec:analysis_QC_res}
We apply the EMPC-based capturability analysis to a switched-LIP model corresponding to the MIT Mini-Cheetah robot in this subsection. To be compatible with the hardware platform, the switched-LIP model parameters are chosen as follows. The gait cycle is set to be $\SI{0.3}{s}$ with stance/swing duty cycle being $0.5$ and the constant CoM height is set to be $\SI{0.29}{m}$, which is the nominal height of the robot's floating base. Within a complete gait, we consider six discretizations with time-step $\SI{0.05}{s}$. The target region adopted in the capturability analysis $\X_T$ in~\eqref{eq:stateset} is considered to be 
\eqn{\ald{ \X_T & = \Big\{\x ~|~ c_x\in[-\SI{0.19}{m},\SI{0.19}{m}], \ c_y\in [-\SI{0.11}{m},\SI{0.11}{m}],\\ & \qquad\ \  \dot{c}_x \in [-\SI{0.2}{m \per s},\SI{0.2}{m \per s}], \ \dot{c}_y [-\SI{0.2}{m \per s},\SI{0.2}{m \per s}] \Big\}.}} 

Three different gaits (trot, bound, and pace) are adopted to validate the proposed analysis approach, whose gait timing plots are shown in Fig.~\ref{fig:gait}. For all three gaits under test, the underlying switched dynamics involve two different modes according to which pair of legs is in stance mode, i.e., they are all two-beat gaits according to~\cite{Boussema2019}.
\begin{figure}[b!]
	\centering
	{\includegraphics[width=0.5\linewidth]{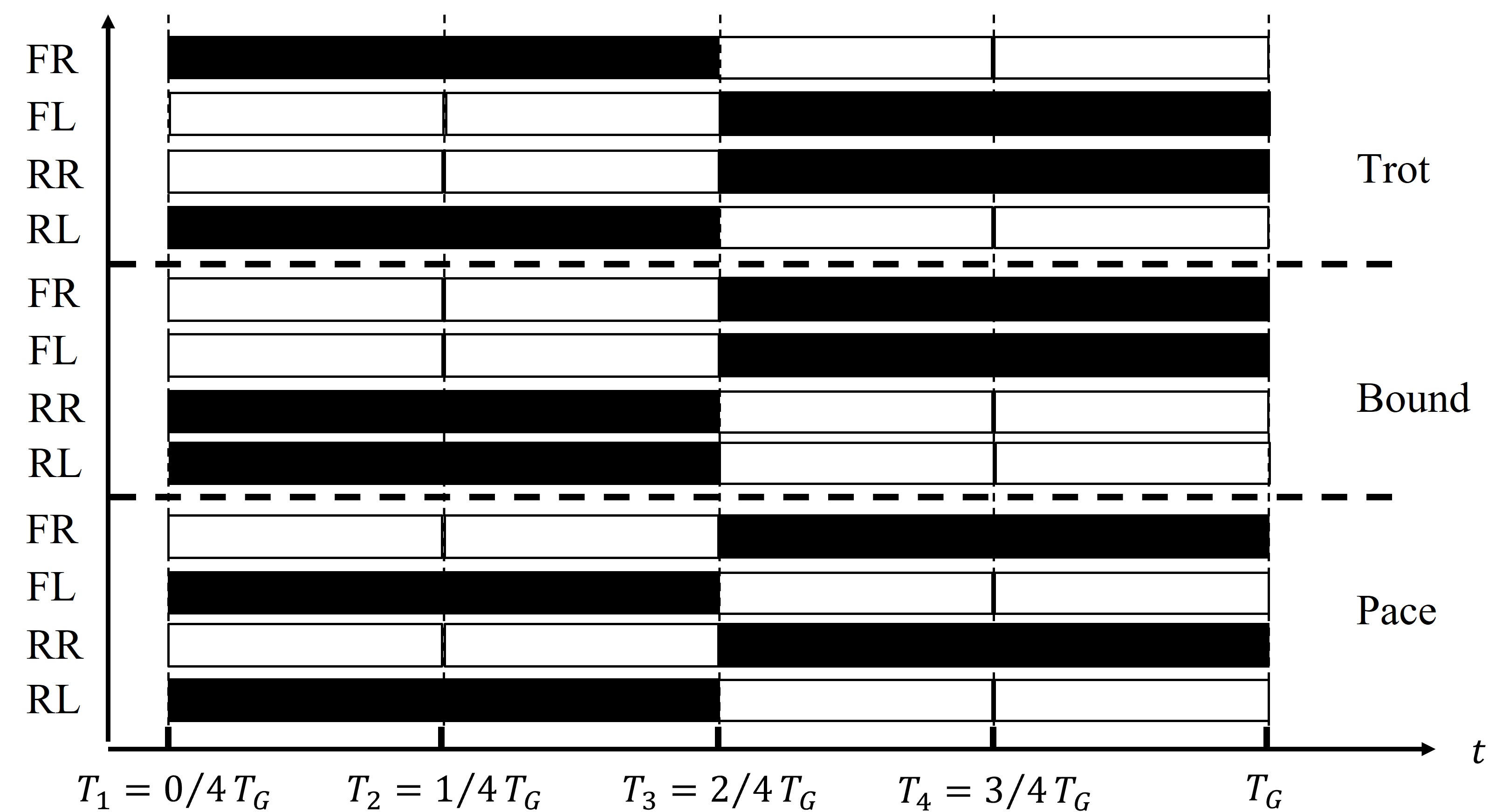}}
	\caption{ \footnotesize   Illustration on gait timings. }
	\label{fig:gait}
\end{figure}
\begin{figure*}[t]
    \centering
    \subfigure[ \footnotesize  $x$-$y$ slices, trot ]{
    \includegraphics[width=0.3\linewidth]{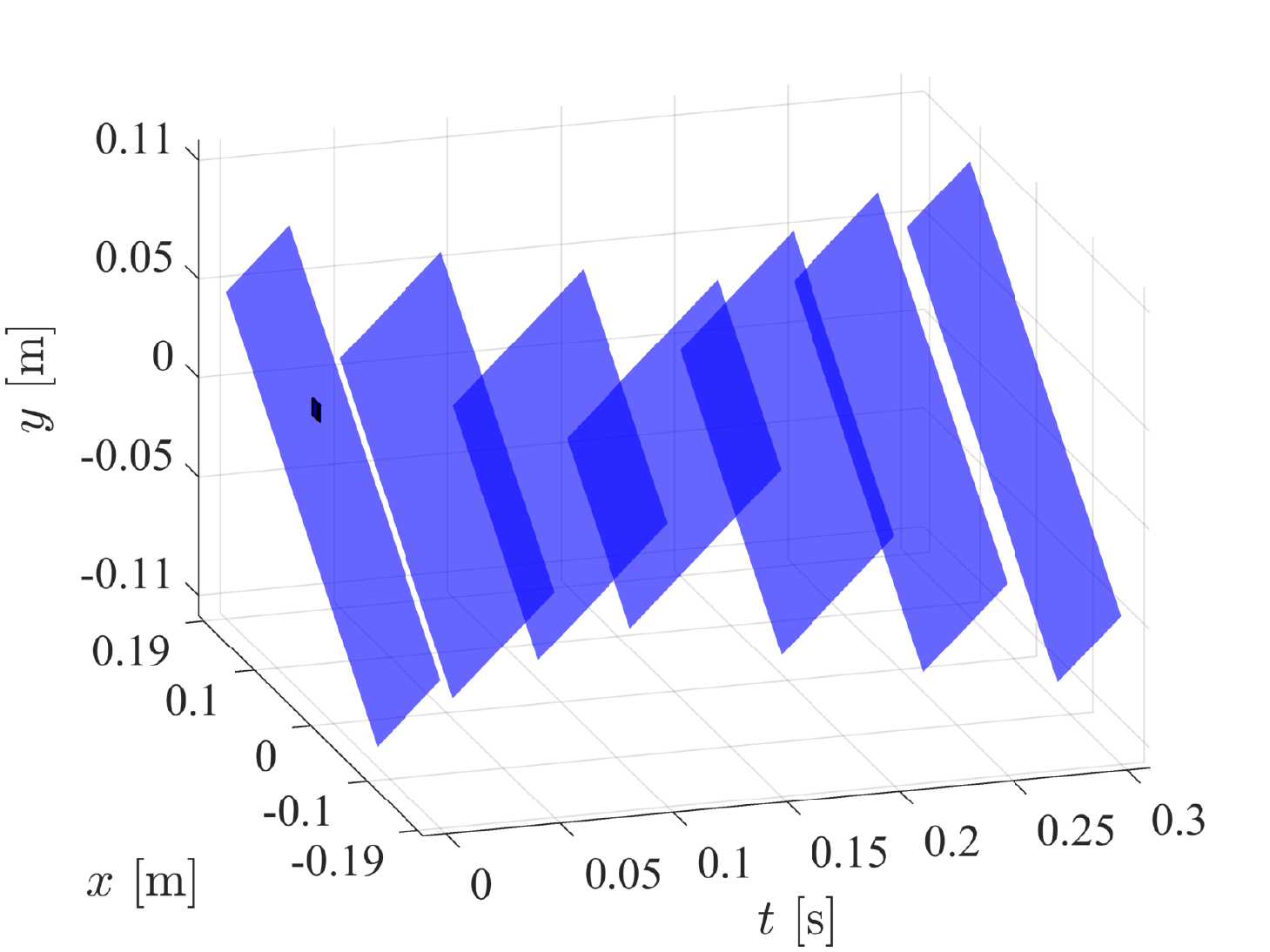}}
    \subfigure[ \footnotesize  $x$-$y$ slices, bound ]{ \label{fig:CITb}
    \includegraphics[width=0.3\linewidth]{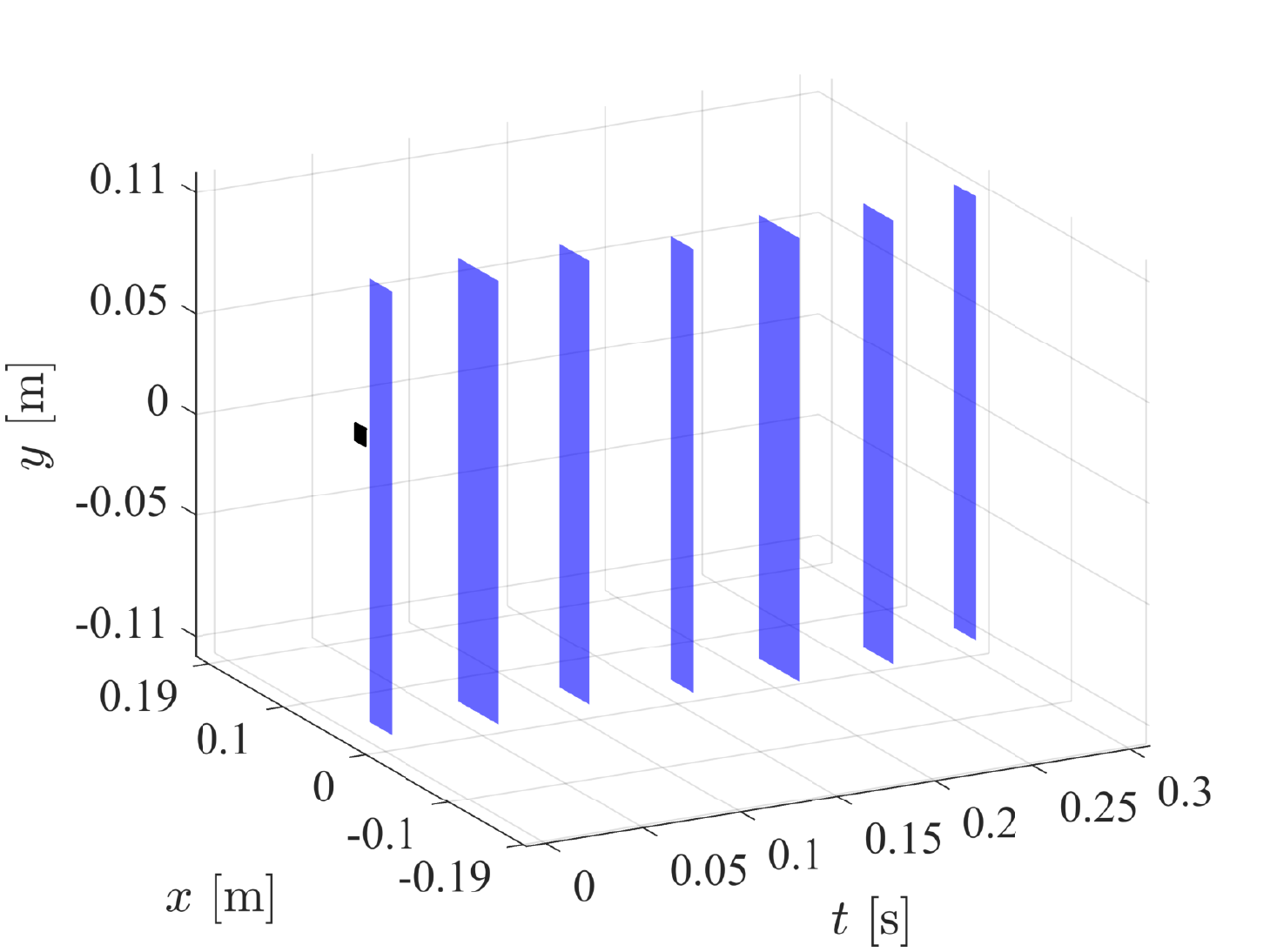}}
    \subfigure[ \footnotesize  $x$-$y$ slices, pace ]{ \label{fig:CITc}
    \includegraphics[width=0.3\linewidth]{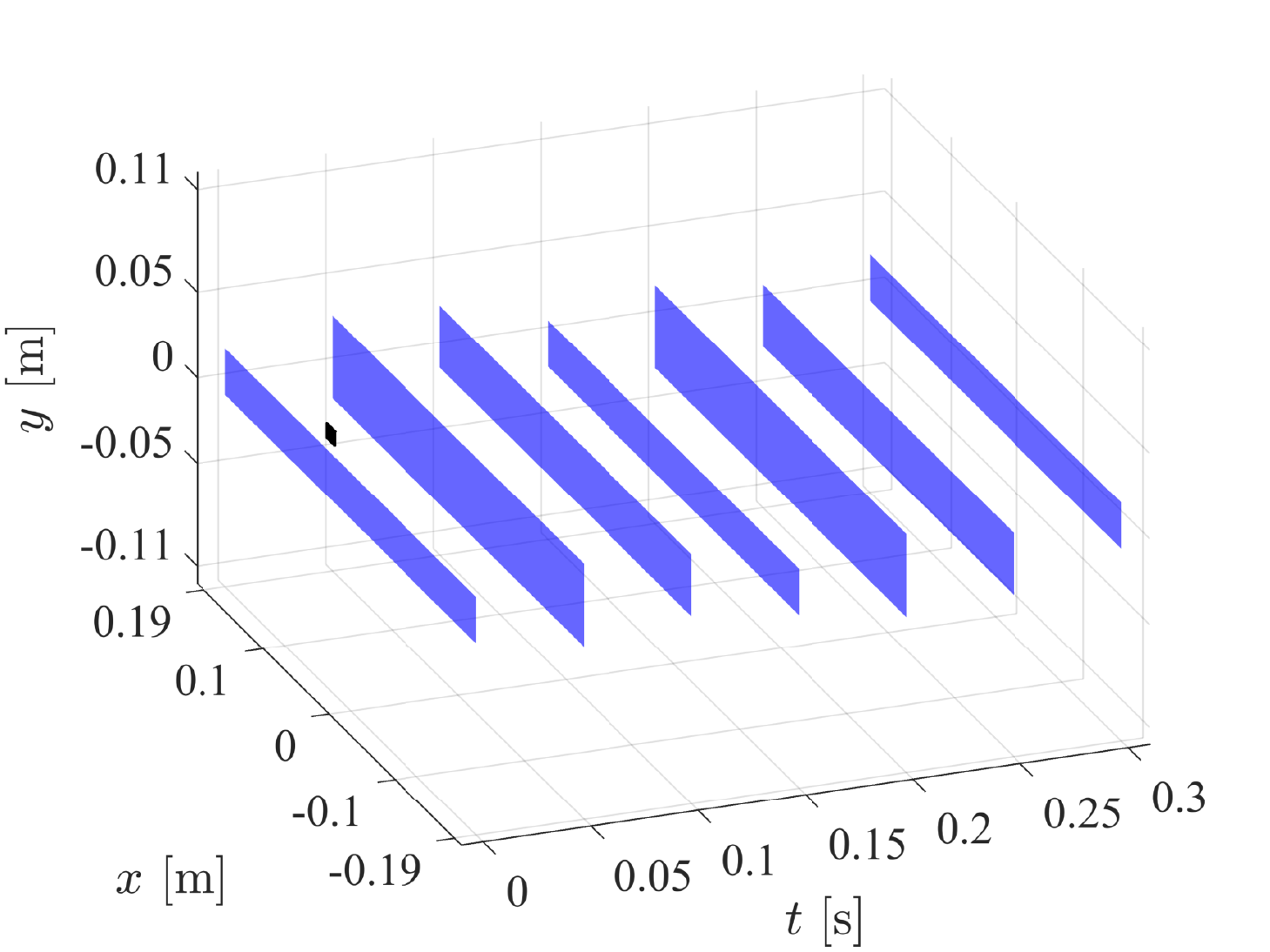}}
    \caption{ \footnotesize   Tube of dynamically balanced states with different gait. }
    \label{fig:CIT}
    \vspace{-10px}
\end{figure*}

For numerical computation, the YALMIP toolbox~\cite{Lofberg2004} and MPT toolbox~\cite{Kvasnica2004} are used for EMPC implementation. For the EMPC problems for different gaits, the terminal constraint sets and the problem horizon are chosen to be the same. With the adopted time discretization and settings, it takes roughly $\SI{10}{min}$ to compute the tube of balanced states and roughly $\SI{5}{min}$ to generate the tube of capturable states for each gait with an Intel (R) Core (TM) i7-9700@ $\SI{3.00}{GHz}$ and $\SI{16}{GB}$ memory. Note that the capturability analysis is done offline and this computation is only required once.

\begin{algorithm}[htbp!]
  \caption{  Computation of tube of capturable states}\label{alg:brs}
  \KwIn{ $\Aa$, $\Bb$, $\X$, $\Bx$, $T$;}
  \KwOut{$\Ca(0;\Bx),\ldots,\Ca(T;\Bx)$.}
  \For{$t=1,\ldots,T_G$}{
   $\Ca(0;\B_t) \leftarrow \B_t$, and $k \leftarrow 0$\;
  \If{$k \le T-1$}{
   $\Ca(k+1;\B_t) = \Pre(\Ca(k;\B_t))\bigcap \X$\;
   $k \leftarrow k+1$\;
  }
  }
  $\Ca(k;\Bx) = \{\Ca(k;\B_t)\}_{t=1}^{T_G} $ for all $k=0,\ldots,T$.
\end{algorithm}

Fig.~\ref{fig:CIT} depicts the x-y slices of the resulting tube of dynamically balanced states over time for different gaits. Each slice in the figure corresponds to a discretization time-step. It can be seen that the feasible set of states that are dynamically balanced changes over time within a single period, which verifies the necessity of considering the tube of dynamic balance in push recovery for quadrupeds. In addition, the difference of shapes of the x-y slices for the tested gaits is obvious from Fig.~\ref{fig:CIT}. Such a difference agrees with the underlying difference between the tested gaits. For example, bound and pace differ only in the size of the constraints, implying the similarity between the slices in Fig.~\ref{fig:CITb} and Fig.~\ref{fig:CITc}. 

Another interesting feature associated with the result is the different shapes of the tubes for different gaits, even though the admissible set of terminal states and the horizon for the EMPC problems are the same. For example, the state $(0.03,0.03,0.4,0.3)$ marked as a black rectangle in Fig.~\ref{fig:CIT} can be kept inside the tube for trot gait, but not with the bound or pace gaits. The periodicity of the evolution of the tubes is also verified by the obtained results. This interesting behavior of dynamic balance for quadrupedal locomotion has not been revealed in traditional capturability analysis for bipeds. Moreover, such a difference would propagate further to the tube of capturable states, which in turn indicates that the underlying gait and disturbance timing both influence the push recovery performance. 

With the above obtained tube of dynamically balanced states, the capturable states can be obtained accordingly. Fig.~\ref{fig:BRS} shows the evolution of the tube of capturable states (backward in time) over time with the trot gait. Two cases with different terminal constraint sets are shown, leading to different sets of capturable states in fact. Taking the volume of computed sets as a quantifier for convergence, it can be seen from Fig.~\ref{fig:BRS} that the volume of the computed sets does not change after about $8$ iterations (i.e., $\SI{1.2}{s}$ and four gait periods).

\begin{figure}[tp!]
    \centering
    \subfigure[ \footnotesize  x-y slice, trot, case 1]{
    \includegraphics[width=0.23\linewidth]{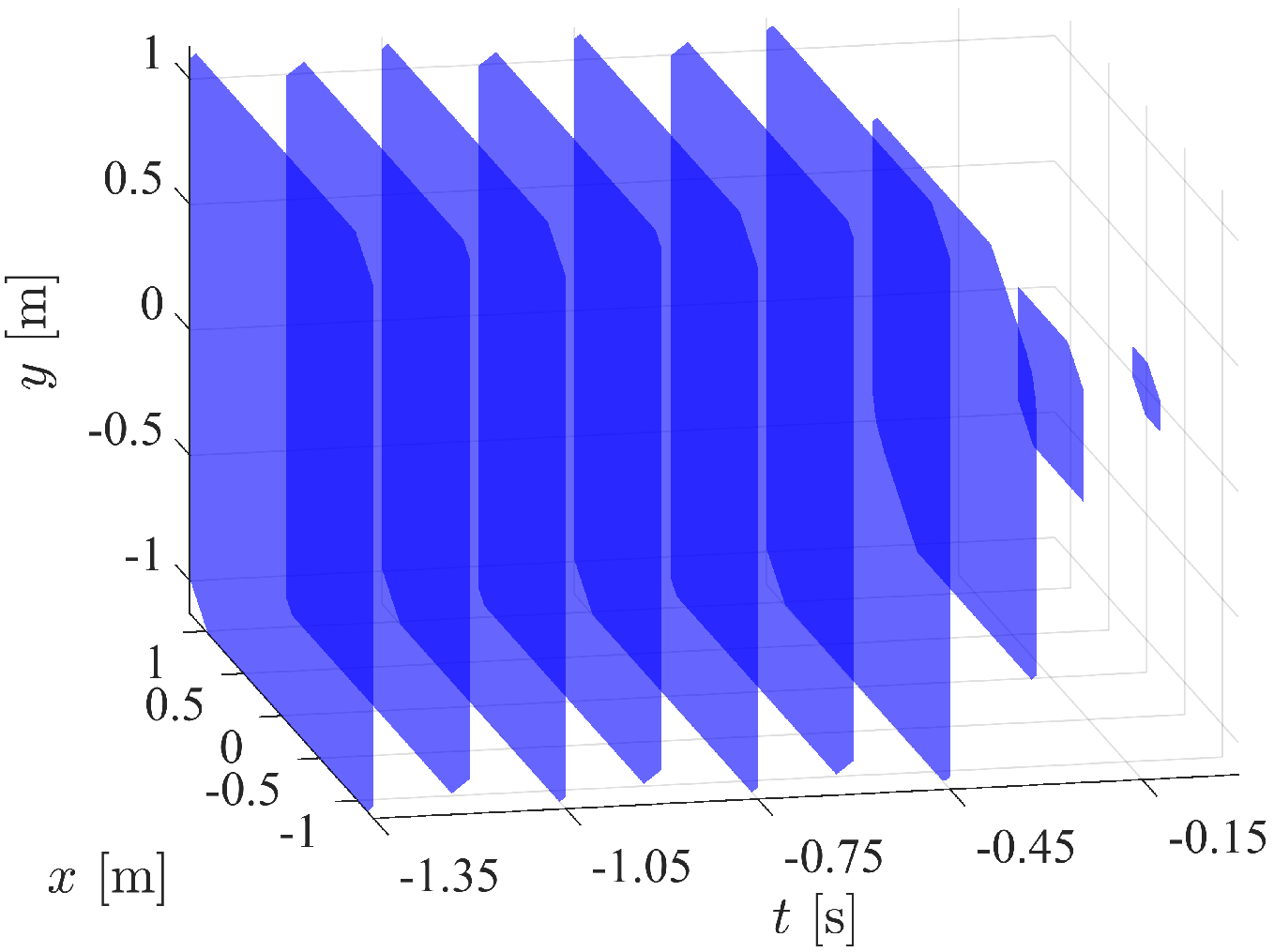}}
    \subfigure[ \footnotesize  Volume over iteration, case1]{
    \includegraphics[width=0.23\linewidth]{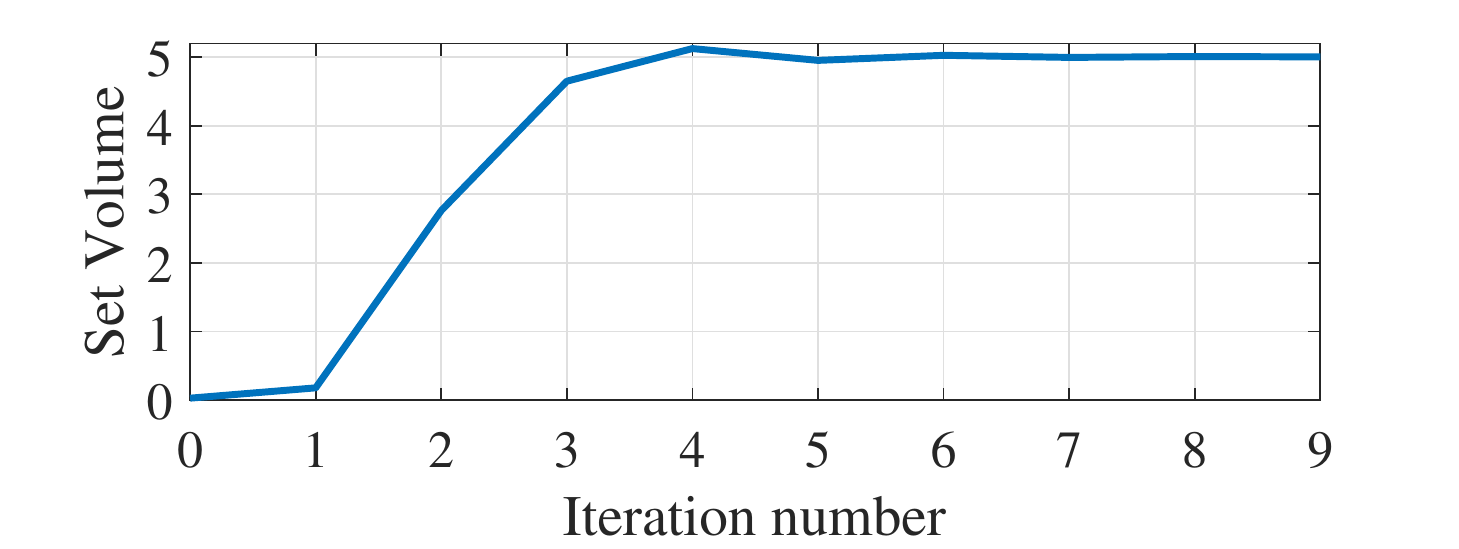}}
    \subfigure[ \footnotesize  x-y slice, trot, case 2]{
    \includegraphics[width=0.23\linewidth]{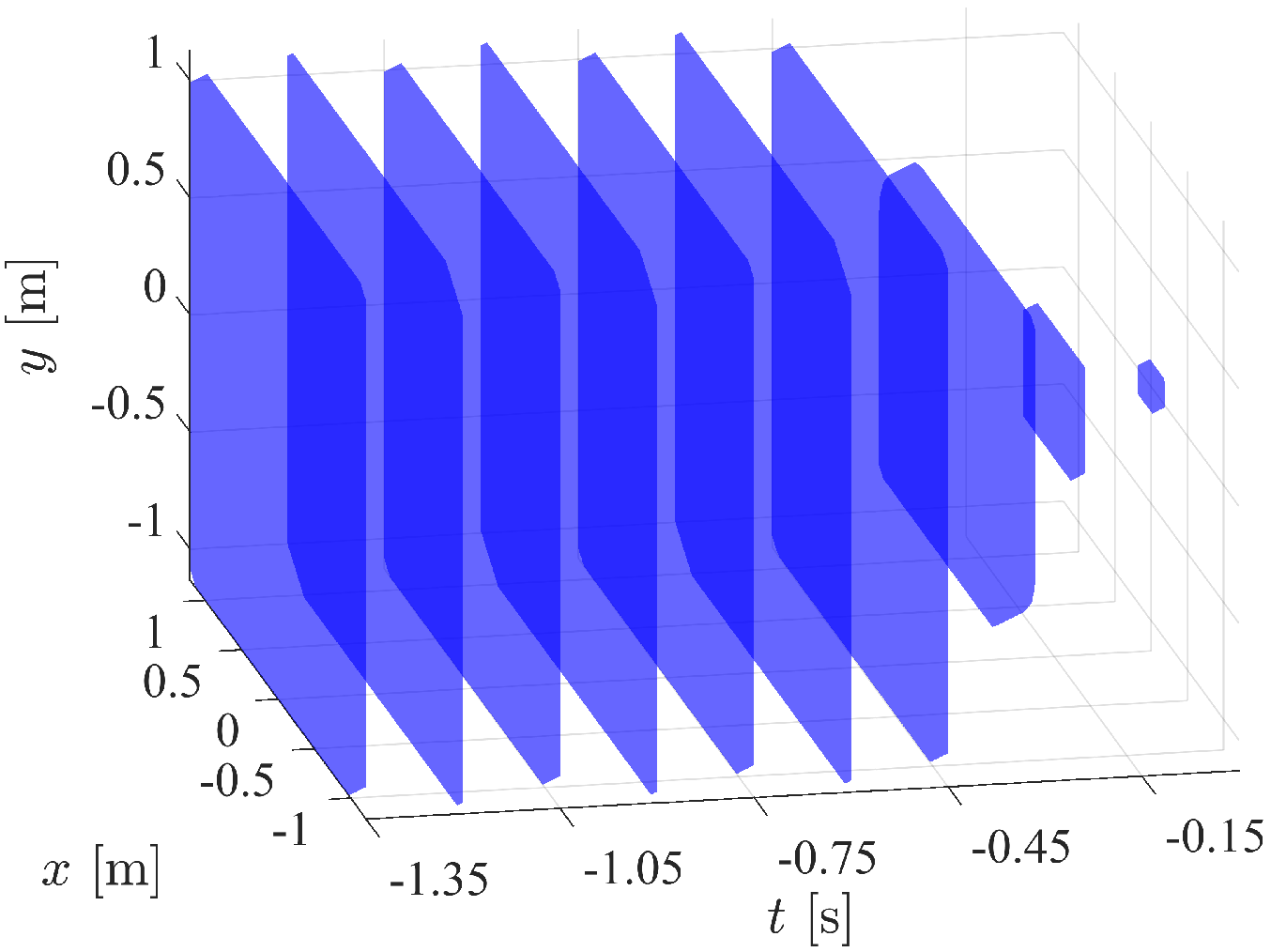}}
    \subfigure[ \footnotesize  Volume over iteration, case 2]{
    \includegraphics[width=0.23\linewidth]{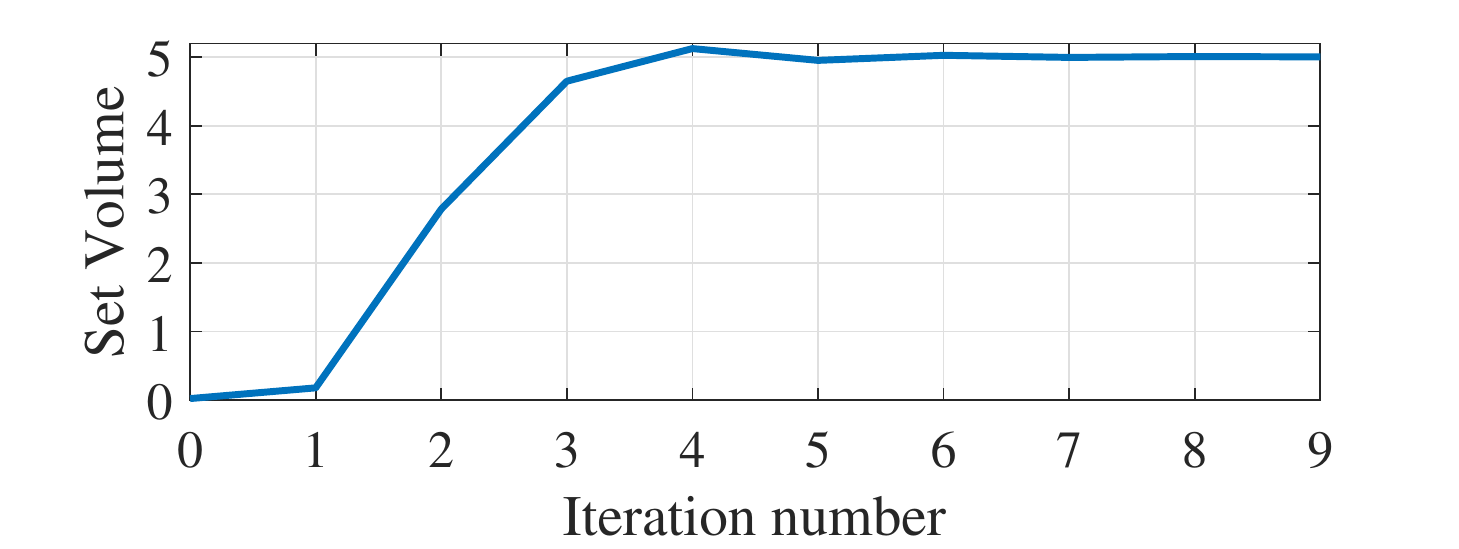}}
    \caption{ \footnotesize Tube of capturable states evolving backward in time. Case 1 corresponds to the evolution of capturable states with terminal set selected to be the first slice of tube of balance states for trot gait in Fig.~\ref{fig:CIT}(a). Case 2 corresponds to the evolution of capturable states with terminal set selected to be the fourth slice of tube of balance states for trot gait in Fig.~\ref{fig:CIT}(a). }
    \label{fig:BRS}
    \vspace{-15px}
\end{figure}

\subsection{Discussions}

In terms of computation of the tube of balanced states and the capturable states, the $\Pre$ operator-based procedure and the EMPC-based approach give the same results if the iteration number in Algorithm~\ref{alg:cit} and the prediction horizon in~\eqref{eq:empc_cit} are compatible. On the other hand, EMPC requires finding the associated value function as well, which introduces additional computational complexities. In fact, the most time-consuming step in solving the EMPC problems is the construction of the critical regions (i.e., $\X_{\B,i}^N\ (\X_{\Ca,i}^N)$ in~\eqref{eq:critical_region}) forming the partition of the overall feasible region. 

The analysis tools and results discussed in this section are tightly connected with the persistent feasibility and stability theories in the model predictive control literature. For classical MPC problems, persistent feasibility is commonly ensured via requiring the terminal constraint set to be control invariant (Theorem 12.1 in~\cite{Borrelli2017}). By constructing the tube of dynamically balanced states, we are trying to establish an admissible terminal constraint encoding the time-varying feature of dynamic balancing for quadrupedal locomotion. Such a control-invariant property of the tube subsequently ensures persistent feasibility of the second EMPC problem~\eqref{eq:empc_brs} for computing the capturable states. In essence, we extend the seminal results in classical MPC to periodically time-varying systems for applications in quadrupedal locomotion.

From a practical application perspective, the above analysis results have strong implications in guiding footstep adaptation according to the underlying gait and different impact timings. In the following section, we develop a CoM-footstep planning scheme for quadrupedal push recovery based on the quadrupedal analysis results.

\section{Quadrupedal Capturability based CoM and Footstep Planning for Push Recovery}\label{sec:foot_com}
In the above two sections, we devoted our main efforts to the formalization and analysis of quadrupedal capturability with dynamic balance, which lays the foundation for synthesizing reliable push recovery controllers. In this section, a comprehensive footstep and CoM planning scheme is developed based upon the capturability analysis. 
 
Throughout the development of push recovery planning scheme, we adopt the discretized LIP dynamics~\eqref{eq:dlip} for planning of the robot's motion. Planning of quadrupedal push recovery requires finding both the footsteps $\w_k$ and the CoP induced by $\u_k$ so that the resulting CoM trajectory restores balance. Mathematically speaking, this motion planning problem can be formulated as the following optimal control problem of the LIP dynamics.
\begin{subequations}\label{eq:plan_ocp}
\al{\min\limits_{\w_k,\u_k,\forall k} & \left( \sum\limits_{k=0}^{N-1} \norm{\x_k}^2_Q +\norm{\u_k}^2_R \right)+ \norm{\x_N}^2_{Q_f} \\ \st & \x_{k+1} = \Aa\x_k+\Bb \w_k \u_k, \quad  \forall k =0:N-1,\\ & \x_k \in \X,\quad \u_k\in \U, \quad  \w_k \in \W_k(\x_k) \\ & \w_N = \w_\B, \quad  \x_N \in \Ca(T;\Bx). \label{eq:plan_ocp_cons_w}  }
\end{subequations}
where the set $\W_k$ takes into account of the time-dependent gait restriction (e.g., the footstep remains fixed if the corresponding leg is in stance) as well as potential kinematic constraints for contact locations. In the above problem, the terminal constraint~\eqref{eq:plan_ocp_cons_w} enforces that the terminal state is capturable and the terminal footsteps match the footsteps for balancing. In addition, the bilinear term $\w_k\u_k$ coupling footsteps with the CoP further complicates the motion planning problem.

To address the above difficulties, a novel approach for finding the terminal constraint based on quadrupedal capturability analysis is developed first. Given the target footsteps, a tractable scheme for jointly determining the coupled CoM and footstep reference trajectories is then devised.

\subsection{Planning for target footsteps}
Finding the target footsteps is commonly ignored in most push recovery control frameworks. Different target footsteps shape the structure of the optimal control problem~\eqref{eq:plan_ocp} through the terminal constraint~\eqref{eq:plan_ocp_cons_w}. In this subsection, we develop a simple yet effective quadratic-programming based approach for target contract reference points planning. 

Our strategy for planning the footsteps relies on an observation that the sets of balanced and capturable states are translationally symmetric, i.e., they are identical modulo a horizontal translation to the footsteps and CoM states. Denoting by $\Bx(\w)$ and $\Ca(\w) = \Ca(T;\Bx(\w))$ the tube of balanced states and the set capturable states to highlight the explicit dependence on footsteps $\w\in \R^{2\times 4}$, the above horizontal translational symmetry property is summarized in the following proposition. 

\begin{proposition}\label{prop:linear_cap_set}
The sets of balanced states and capturable states posses a horizontal translational symmetry, i.e., if $\w_2 = \w_1+\Delta \w$ and $\X_{T,2} = \X_{T,1} + \Delta \Ww$, with $\Delta \w = (\Delta w_x,  \Delta w_y )^\top$ and $\Delta \Ww = (\Delta w_x, 0, \Delta w_y, 0)^\top$, then 
\subeq{\al{ &\Bx(\w_2) = \Bx(\w_1+\Delta \w)  = \Bx(\w_1)+ \Delta \Ww  \text{ and,}\\ &\Ca(\w_2) =\Ca(\w_1+\Delta \w)  = \Ca(\w_1)+ \Delta \Ww. }} 
\end{proposition}
\begin{proof}
See Appendix~\ref{app:prop1}.
\end{proof}

Leveraging this property, the quadrupedal capturability analysis in the previous section can be considered as generating a nominal set of capturable states with a nominal set of footsteps. During online execution, the target footsteps are computed by searching for the deviation from the nominal one $\Delta \w$. The online adjustment step can be achieved via solving the following optimization problem
\subeq{\al{ \min\limits_{(\Delta w_x,\Delta w_y)}\ &  J_\text{PR}(\Delta w_x,\Delta w_y) \\ \st \quad \  &  (c_x-\Delta w_x,\dot{c}_x,c_y-\Delta w_y,\dot{c}_y)\in \Ca(T;\Bx), \label{eq:cap_com_cons}}} 
where $(c_x,\dot{c}_x,c_y,\dot{c}_y)$ is the real-time measured CoM state. The constraint~\eqref{eq:cap_com_cons} enforces that the target footsteps should capture the quadruped's state if its upcoming footsteps are chosen to be the ones in the optimized reference frame.

\begin{remark}\label{rmk:gene_cp}
While this formulation appears to be different at first glance, the above optimization generalizes the traditional usage of capture point in push recovery. Traditionally, once capture point is computed, it is considered as a reference point on the ground specifying where the CoM will eventually come to a stop. However, this capture point strategy is not generalizable to quadrupeds, due to the previously discussed issues with quadrupeds in Section~\ref{sec:sys_qcap}. 
\end{remark}

\begin{figure}[tbp!]
	\centering
	{\includegraphics[width=0.7\linewidth]{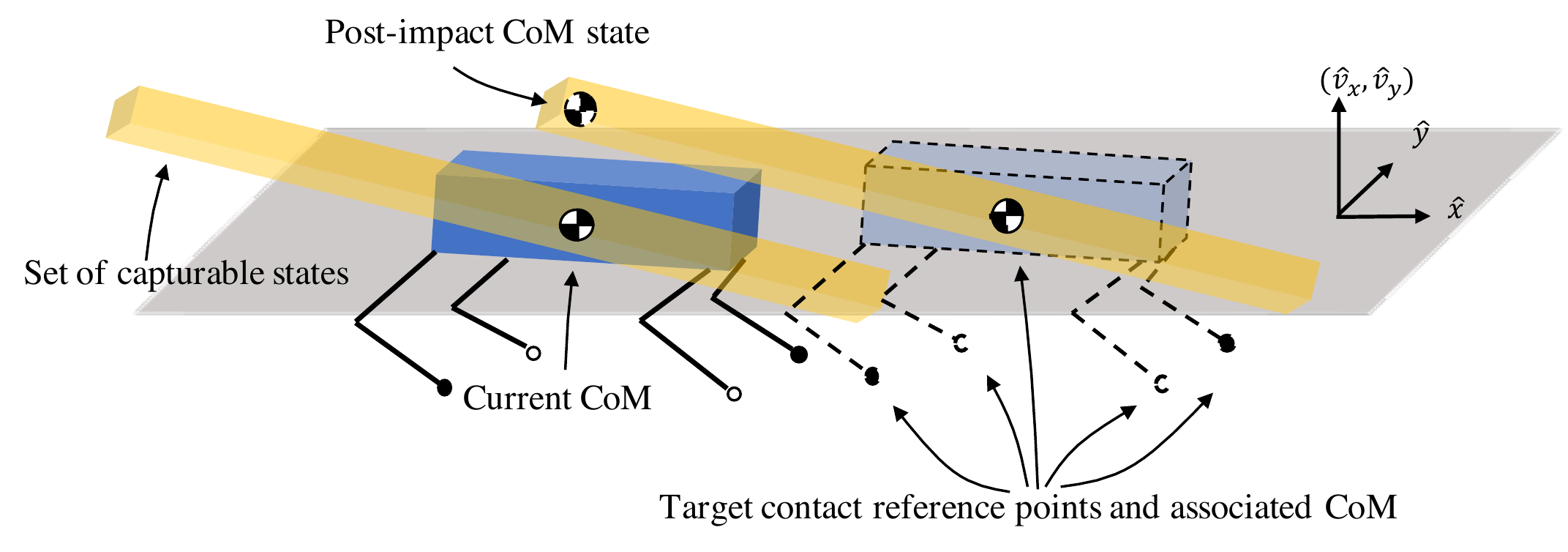}}
	\caption{ \footnotesize   Illustration on target footsteps planning. Once the robot is disturbed, it gains nontrivial velocity. The CoM state must reside inside the green region to ensure capturability. Target footsteps are found via moving the set of capturable states (green polytope) in space such that the post-impact CoM state is embraced.  }
	\label{fig:cap_com}
	\vspace{-10px}
\end{figure}

Fig.~\ref{fig:cap_com} gives a graphical explanation on how the above optimization determines the target footsteps to further illustrate how the change of footsteps helps balancing. 

Note that, the cost function $J_\text{PR}$ is not yet specified. Remember that the EMPC based capturability analysis approach gives a convex and piecewise quadratic value function over the set of capturable states (Theorem~\ref{thm:empc_pro}), which can be directly applied. However, solving a optimization with piecewise quadratic cost function remains numerically challenging. To further simplify this optimization problem, we approximate the piecewise quadratic value function with a single quadratic approximation by the following optimization problem
\subeq{\al{\min\limits_{\Pf,\epsilon } & \quad \epsilon  \\ \st & \quad\Pf \succ \epsilon \Id, \\ & \quad   \Pf\succ \Pf_i, \forall \Pf_i \in \eqref{eq:pwq2} .}}
The above optimization finds an upper bound on the value function obtained from the EMPC capturability analysis and is a semi-definite programming that can be solved offline. 

Combined with the polytopic feature of the set of capturable states, the above simplification reduces the nonlinear optimization~\eqref{eq:cap_com_cons} to a quadratic programming problem given by:
\subeq{\label{eq:cap_com}\al{  \min\limits_{\Delta \w = (\Delta w_x,\Delta w_y)}\ & \ar{{c} c_x-\Delta w_x \\ \dot{c}_x \\ c_y-\Delta w_y\\ \dot{c}_y\\ 1} ^\top \Pf  \ar{{c} c_x-\Delta w_x \\ \dot{c}_x \\ c_y-\Delta w_y\\ \dot{c}_y\\ 1} \\ \st \quad \quad  &  H^{\cal C} \ar{{c} c_x-\Delta w_x \\ \dot{c}_x \\ c_y-\Delta w_y\\ \dot{c}_y} \le h^{\cal C}. }} 

The above two-dimensional QP can be solved very efficiently online, actively adjusting the target configuration with the most up-to-date information about the robot's state. Such a strategy partially addresses the deficiency of only considering frame-fixed contact points in the capturability analysis without resorting to complicated multi-step capturability analysis. 

Once the target final footsteps (i.e., stopping location) are determined through solving~\eqref{eq:cap_com}, the follow-up question is how to coordinate the intermediate footsteps and CoM trajectory to maneuver to the planned target configuration. An iterative MPC-based approach is devised to address this problem below.

\subsection{Motion Planning for CoM and Footstep Trajectories}

Due to the complex coupling between CoM and footstep locations, directly solving the planning problem for both CoM and footsteps~\eqref{eq:plan_ocp} remains difficult even with the simple 3D-LIP dynamics. A simple strategy to address this challenge is to alternatively solve the CoM planning problem with a given footstep sequence and update footstep sequence based on the optimized CoM trajectory in an iterative manner.

\subsubsection{CoM planning}
Given an initial sequence of upcoming footsteps $\{\w_k\}_{k\in \Na}$, the LIP dynamics become a simple linear time varying system. Planning of the CoM trajectory with this LIP dynamics over horizon $N_{P-1}$ becomes a standard optimal control problem with linear time-varying dynamics given below
\subeq{\label{eq:com_plan} \al{
\min\limits_{\u_k} &  \sum\limits_{k=0}^{N_P-1} (\norm{\x_k}^2_Q + \norm{\u_k}^2_R ) + \norm{\x_{N_P} - \x_D^*}^2_{Q_f} \\  \st  & \x_{k+1} = \Aa\x_k +\Bb \w_k \u_k, \qquad \forall k\in \Na_P, \\ & \u_k \in \S_{\lambda},  \qquad \qquad \qquad  \qquad \  \forall k \in \Na_P,}}
where $\x_D^*$ is the CoM state corresponding to the target footsteps obtained from the previous subsection and $\S_{\lambda}$ is the input constraint set defined in~\eqref{eq:sigma} associated with the underlying gait information $\sigmaa$. 

As compared with the original planning problem~\eqref{eq:plan_ocp}, the complex coupling is removed thanks to the knowledge about the footstep sequence, making the above problem a quadratic program in nature and hence tractable to solve. 

\subsubsection{Footstep planning}
The footstep planning (see Fig.~\ref{fig:over_arch}) aims to update the sequence of footsteps according to the planned CoM trajectory. Denoting by $\{\w^m_k\}_{k\in \Na}$ the footstep sequence at the $m$-th iteration and by $\{\boldsymbol{\delta} \w^m_k\}_{k\in \Na}$ the adjustment of the overall sequence, the footstep update problem is formulated as below.
\begin{figure}[tp!]
	\centering
	{\includegraphics[width=0.7\linewidth]{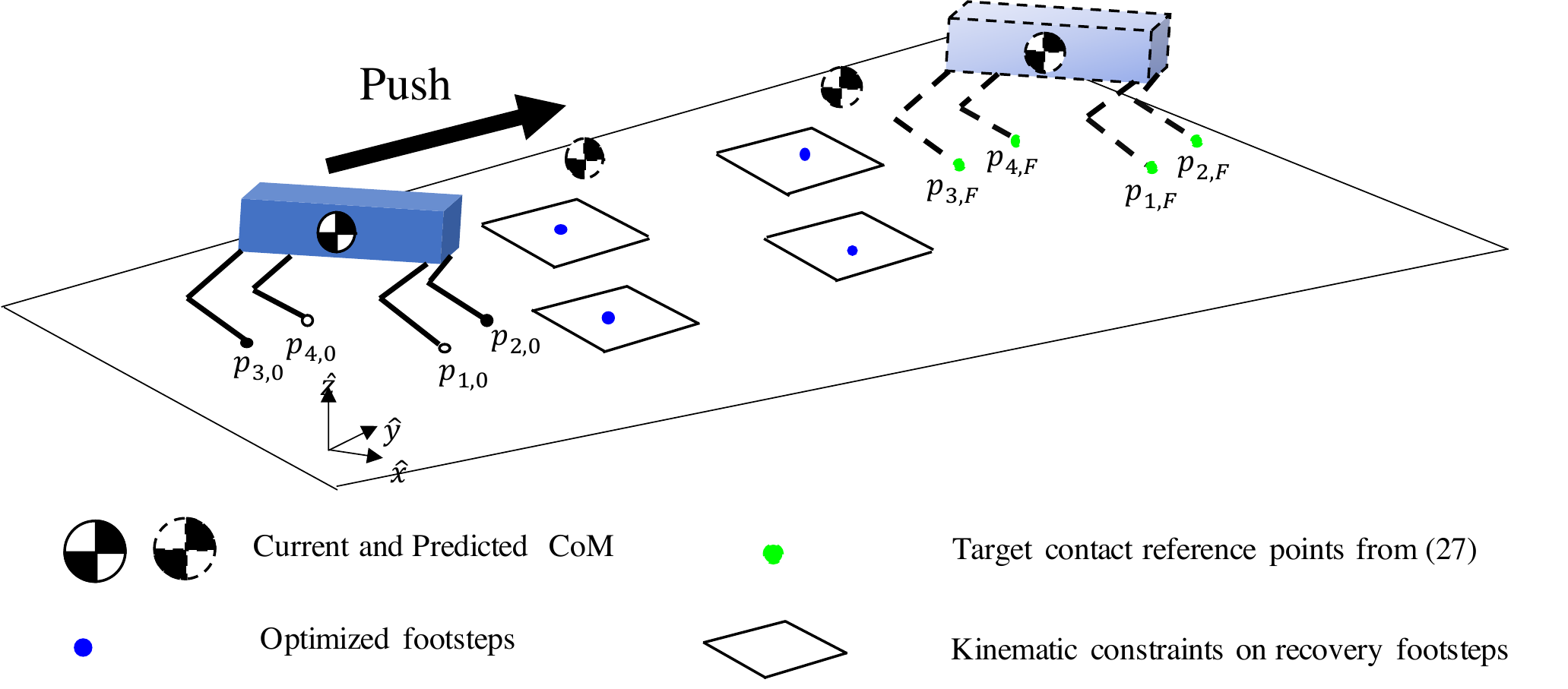}}
	\caption{ \footnotesize   Illustration on CoM and footstep planning. }
	\label{fig:plan_com-fh}
\end{figure}
\subeq{\label{eq:updateFH}\al{ \min\limits_{\boldsymbol{\delta} \w^m_k, k\in \Na_S}  &    J_\text{FH}(\boldsymbol{\delta} \w^m_k ) \\ \st \quad &  \boldsymbol{\delta} \w^m_N = \0,  \label{eq:updateFH_term} \\ &  \w^m_k\! +\! \boldsymbol{\delta} \w^m_k \in \c_{xy,t_k}  \!+\! [\underline{\Delta h},\overline{\Delta h}],  \forall k\in \Na_S.\label{eq:updateFH_kine} }}
where $\c_{xy,t_k}$ is the planar CoM position at each contact switching time $t_k$, and $\Na_S = 1,\ldots,N_S$ denotes the indices of expected number of upcoming steps.

Initialization of the above problem $\{\w^0_k\}_{k=1}^{N_S}$ is set via first predicting the CoM trajectory with the currently measured CoM position and velocity and then generating the sequence of $\{\w^0_k\}_{k=1}^{N_S}$ with nominal relative positions between all feet and the CoM.

In principle, the cost function $ J_\text{FH}(\boldsymbol{\delta} \w^m_k ) $ should characterize how the adjustment on the footstep would impact the performance of the resulting CoM trajectory. However, the bilinear coupling between CoM and footsteps makes the exact characterization impractical. To alleviate this issue, we adopt a simple quadratic approximation to the cost function to ensure efficient online implementation.
\eq{J_\text{FH}(\boldsymbol{\delta} \w^m_k) = \| \boldsymbol{\delta} (\w^m_k)^\top  \boldsymbol{\delta}\w^m_k \|^2_{Q_\text{FH}}. }

\subsubsection{Number of steps planning}
The number of steps needed to recover from external disturbance is another critical quantity ensuring successful recovery. Following the simple approach used in~\cite{Park2021}, we propose to determine the number of steps by solving the above iterative QPs for a range of $N$ and pick the one with smallest cost. In practice, we can simply set the number of steps $N$ to be fixed and relatively large, and solve the corresponding planning problem. Fig.~\ref{fig:plan_com-fh} graphically illustrates how the above developed MPC problem determines footstep profiles. 

\subsection{Implementation Details and Discussions}
To sum up this section, we provide some discussions on implementation details. First, the pseudo-code of an implementable algorithm for the planning of CoM and footstep reference trajectories is given.

\begin{algorithm}[htbp!]
  \caption{  CoM and Footstep Planning}\label{alg:com_foot_plan}
  \KwIn{ $\Reach$, $c_{x}$, $c_{y}$, $\dot{c}_{x}$, $\dot{c}_{y}$, $\w_0$, MaxIter}
  \KwOut{ $\w^*, \X^* ,\U^*$}
  \If{($c_{x}, \dot{c}_{x}, c_{y}, \dot{c}_{y})\notin\Reach$}{
   $ (\Delta w_{x}^{*}, \Delta w_{y}^{*})\leftarrow$ Solve~\eqref{eq:cap_com}\; 
   $\c_F \leftarrow (c_x,c_y,c_z) + (\Delta w_{x}^{*}, \Delta w_{y}^{*}, 0)$\;
   $\w_{F} \leftarrow \w_0 + \Delta \w$\;
  \While{ Iter $<$ MaxIter}{z
   $\boldsymbol{\delta} \w_k^\text{Iter} \leftarrow$ Solve~\eqref{eq:updateFH}\;
   $\w_k^\text{Iter} = \w_k^\text{Iter-1}+\boldsymbol{\delta} \w_k^\text{Iter}$\;
   $(\X^\text{Iter} ,\U^\text{Iter}) \leftarrow $Solve~\eqref{eq:com_plan}.
  }
  }
\end{algorithm}

The proposed planning scheme requires solving only quadratic programs. The QP for planning of target footsteps is only two-dimensional, which is solved at the same frequency as the low-level controller, updating the reference for the subsequent CoM and footstep trajectories planning problem in real time. In the iterative planning of the CoM and footstep trajectories, a number of moderate-scale QPs need to be solved. For the footstep planning problem, each QP involves $2N+2$ decision variables and roughly $4N+4$ inequality constraints, with $N$ being the number of predicted upcoming footsteps that is typically no larger than $5$. On the other hand, each CoM planning QP involves $(n_c-1)N_p$ variables, with $n_c$ being the number of contact points at each time depending on the underlying gait and $N_p$ being the number of discretized time-steps. Generally speaking, these two QPs are the most complex part in the proposed framework and need to be iteratively solved until convergence, which would presumably be the bottleneck of the planning scheme in practice. In practical implementation, solutions to the CoM-footstep trajectory planning problem are updated at the same frequency as the planner, giving the algorithm enough time to iterate without compromising the real-time feature of the overall planning scheme.

\section{Simulation and Hardware Experiment Results}\label{sec:validation}
\begin{table}[h!]\footnotesize 
    \centering
     \begin{tabular}{l|c|c|c|c|c|c}
     \toprule
     \multirow{2}{*}{} &
      \multicolumn{2}{c}{Trot } &
      \multicolumn{2}{c}{Bound} &
      \multicolumn{2}{c}{Pace} \\
      & Proposed$|$Baseline & Baseline$|$Proposed & Proposed$|$Baseline & Baseline$|$Proposed  & Proposed$|$Baseline & Baseline$|$Proposed  \\
      \midrule
    Timing 1 & $234/234 (100\%)$ &  $234/866 (27.02\%)$ & $ 198/198 (100\%)$ & $198/981 (20.18\%)$ & $ 219/219 (100\%)$ & $219/887 (24.69\%)$ \\
    Timing 2 & $219/219 (100\%)$ & $219/1160 (18.88\%)$ & $211/211 (100\%)$ &  $ 211/919 (22.96\%)$ & $ 219/219 (100\%)$ & $219/1061 (20.64\%)$\\
    Timing 3 & $ 229/229 (100\%)$ & $ 229/830 (27.59\%)$ & $208/208 (100\%)$ & $208/653 (31.85\%)$ & $229/229 (100\%)$ &$229/959 (23.88\%)$ \\
    Timing 4 & $219/219 (100\%)$ & $219/1195 (18.33\%)$ & $207/207 ( 100\%)$ & $207/1295 (15.98\%)$ & $206/206 (100\%)$ &$206/1105 (18.64\%)$ \\
    \bottomrule
    \end{tabular}
    \caption{Comparison between the proposed and baseline approaches.}
    \label{tab:stat_compare}
\end{table}
To demonstrate the performance of the proposed capturability based push recovery framework, we conduct various simulation and hardware experiments. In general, we exploit the flexibility of conducting simulations to verify the key contributions of the proposed framework, including the influences of different gaits, different impact timings and comparison with state-of-the-art controllers. Hardware experiments are conducted to highlight practical applicability of the proposed controller and its advantage over existing alternatives. \begin{figure}[tbp!]
	\centering \footnotesize
\begin{minipage}[h]{0.27\linewidth}
\includegraphics[width=\linewidth]{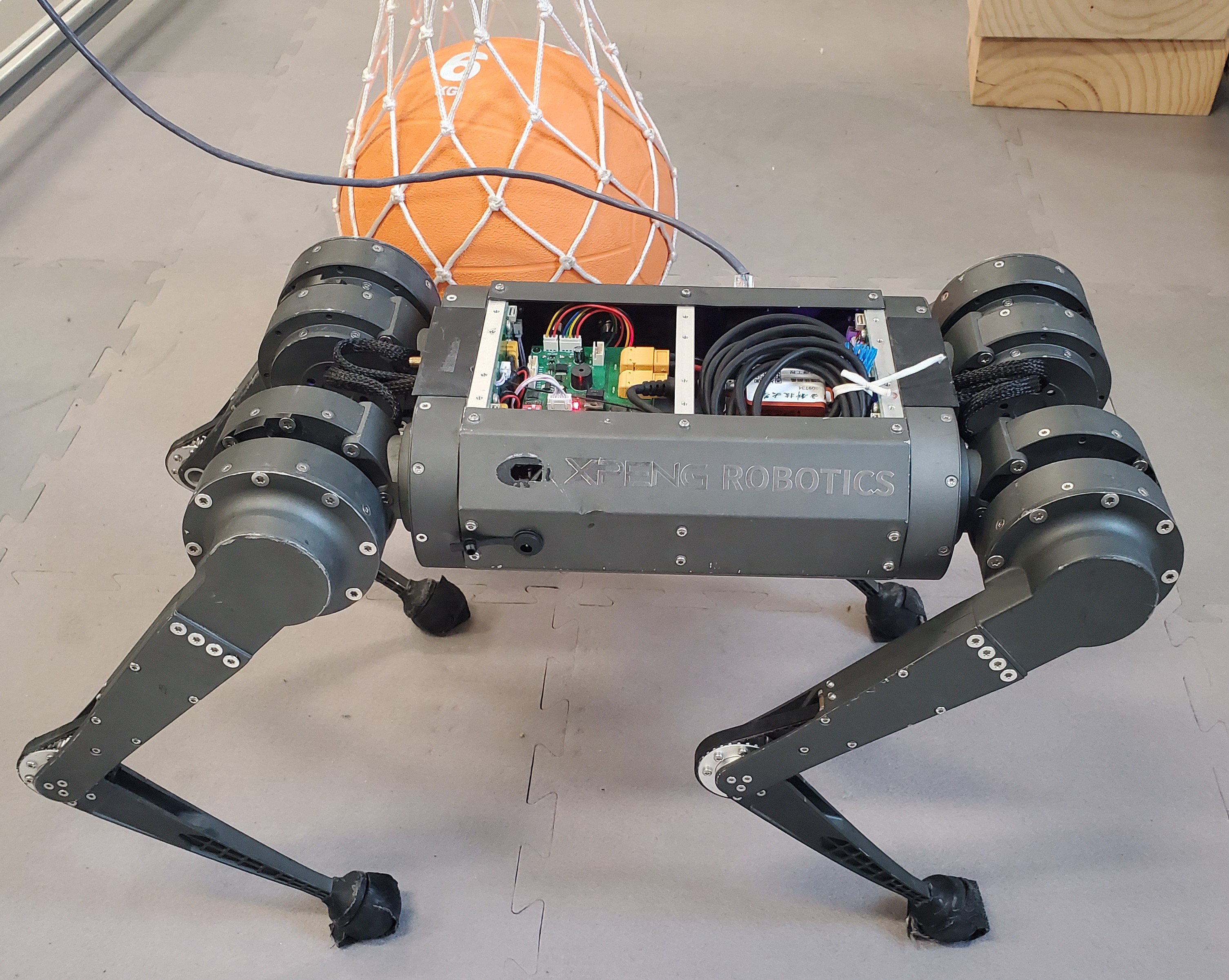}
\end{minipage}\hspace{0.15\linewidth}
\begin{minipage}[h]{0.48\linewidth}
\begin{tabular}{cc}
\hline Parameter & Value \\ \hline
$m$ &  9\SI{}{kg}(10.5  \SI{}{kg}$^\dagger$) \\
$\I_{xx}$ & 0.07  \SI{}{kgm^{2}}\\
$\I_{yy}$ & 0.26  \SI{}{kgm^{2}}\\
$\I_{zz}$ & 0.242  \SI{}{kgm^{2}}\\
$\mu$ & 0.5  \\
$f_{\textrm{min}}$ & 5  \SI{}{N}\\
$f_{\textrm{max}}$ & 150  \SI{}{N}\\
\hline
\end{tabular}
\end{minipage}
\caption{ \footnotesize Replicated Mini Cheetah robot and parameters used in simulator. $^\dagger$: Mass of the replica is $10.5  \SI{}{kg}$, which is different from the Mini Cheetah robot used in simulation. }
\label{fig:mc}
\vspace{-10px}
\end{figure}

In all simulation and hardware experiments, we adopt a convex MPC and whole body impulsive control (MPC + WBIC) based controller proposed by~\cite{Kim2019} as the baseline controller for comparison. This approach augments the convex MPC based quadrupedal locomotion controller with a whole-body control step that achieves highly dynamic locomotion (up to $\SI{3.7}{m\per s}$ and push recovery with hardware).

\subsection{Experiment Platform and Specifications}
The MIT Mini Cheetah robot~\cite{Katz2019} is used in the simulation validations with the open-source simulation platform that is available online~\cite{MC}. A replica of the Mini Cheetah robot (see Fig.~\ref{fig:mc}) is used in hardware experiments. Parameters of the robot and simulation settings are provided in Fig.~\ref{fig:mc}. Throughout the hardware tests, the proposed planning scheme is integrated with the state estimator used in~\cite{Bledt2018} and tracking controller~\cite{DiCarlo2018} that are open-source available~\cite{MC} with no additional parameter tuning.

For the experiment tests, all online computations are conducted with an on-board Intel Atom x5-Z8350 processor @ $\SI{1.44}{GHz}$. State estimator and low-level motion controller run at $\SI{500}{Hz}$, the motion planner for the CoM and footsteps runs at $\SI{33}{Hz}$ (compatible with the baseline convex MPC+WBIC approach). Three different gaits (i.e., trot, pace and bound) are employed to test the performance of the overall framework. The capturability results obtained offline in Section~\ref{sec:analysis_QC_res} are used in the following experiments.

\subsection{Simulation Validations}

We first test the quadruped's push recovery ability under disturbances with different directions and magnitudes at different timings. To highlight the influence of impact timings, we pick four timings (two contact switching time instances and two mid-stance time instances) as depicted in Fig.~\ref{fig:gait}. External disturbances are modeled as instantaneous velocity changes at the floating-base (CoM) in all simulations. To clearly show the performance of the proposed planning framework and rule out other possible factors affecting the results, an ideal full state feedback is assumed. 

For each case, instantaneous longitudinal and lateral CoM velocity changes ranging in $[-\SI{6}{m\per s}, \SI{6}{m\per s}] \times$  $[-\SI{5}{m\per s}, \SI{5}{m\per s}]$ with resolution $\SI{0.2}{m\per s}$, i.e., $3000$ different instances, are tested. In total, $36000$ different scenarios are tested in simulations. A grid point is marked a success if all contact-checking points on the body do not collide with the ground and the quadruped's floating base state recovers from the impact back to balance after $\SI{5}{s}$.

Fig.~\ref{fig:timing_1}-\ref{fig:timing_3} and Table~\ref{tab:stat_compare} jointly show the performance comparison between the proposed approach and the baseline method. In the figures, green markers denote the impacts for which the proposed framework can successfully recover from the push and pink markers correspond to the successful recoveries with the baseline convex MPC+WBIC based approach. It can be seen that the proposed approach is able to reject \textbf{twice} the impact as compared with the baseline approach on average. In certain directions, the disturbance rejection ability of the proposed approach boosts up to \textbf{three times}. Table~\ref{tab:stat_compare} summarizes the comparison between the proposed and baseline approaches. To quantify the difference between performance of these two approaches, we use the following two conditional ratios as the metrics:
\eqn{\ald{\eta_{p|b} &= \frac{ \# \text{ of successes with \textbf{proposed} and \textbf{baseline} }}{\#  \text{ of successes with \textbf{baseline} }} \\ \eta_{b|p} & = \frac{\# \text{ of successes with \textbf{proposed} and \textbf{baseline} }}{\# \text{ of successes with \textbf{proposed} }}}}
From Table~\ref{tab:stat_compare}, it can be seen that the metric $\eta_{p|b}$ are $100\%$ for all tested cases, meaning that for all cases that the baseline controller succeeds in rejecting the disturbance, the proposed controller can accomplish the task as well. On the other hand, $\eta_{p|b}$ ranges from $15.98\%$ to $31.85\%$ across the tested scenarios, meaning that the proposed controller can handle a lot more cases than the baseline controller. In fact, the smaller $\eta_{p|b}$ is, the better the proposed controller is in the corresponding scenario. Overall, it is clear that the proposed approach performs much better than the baseline approach from the simulation test.

\begin{figure*}[tp!]
    \centering
    \subfigure[ \footnotesize  Timing $T_1$]{
    \includegraphics[width=0.235\linewidth]{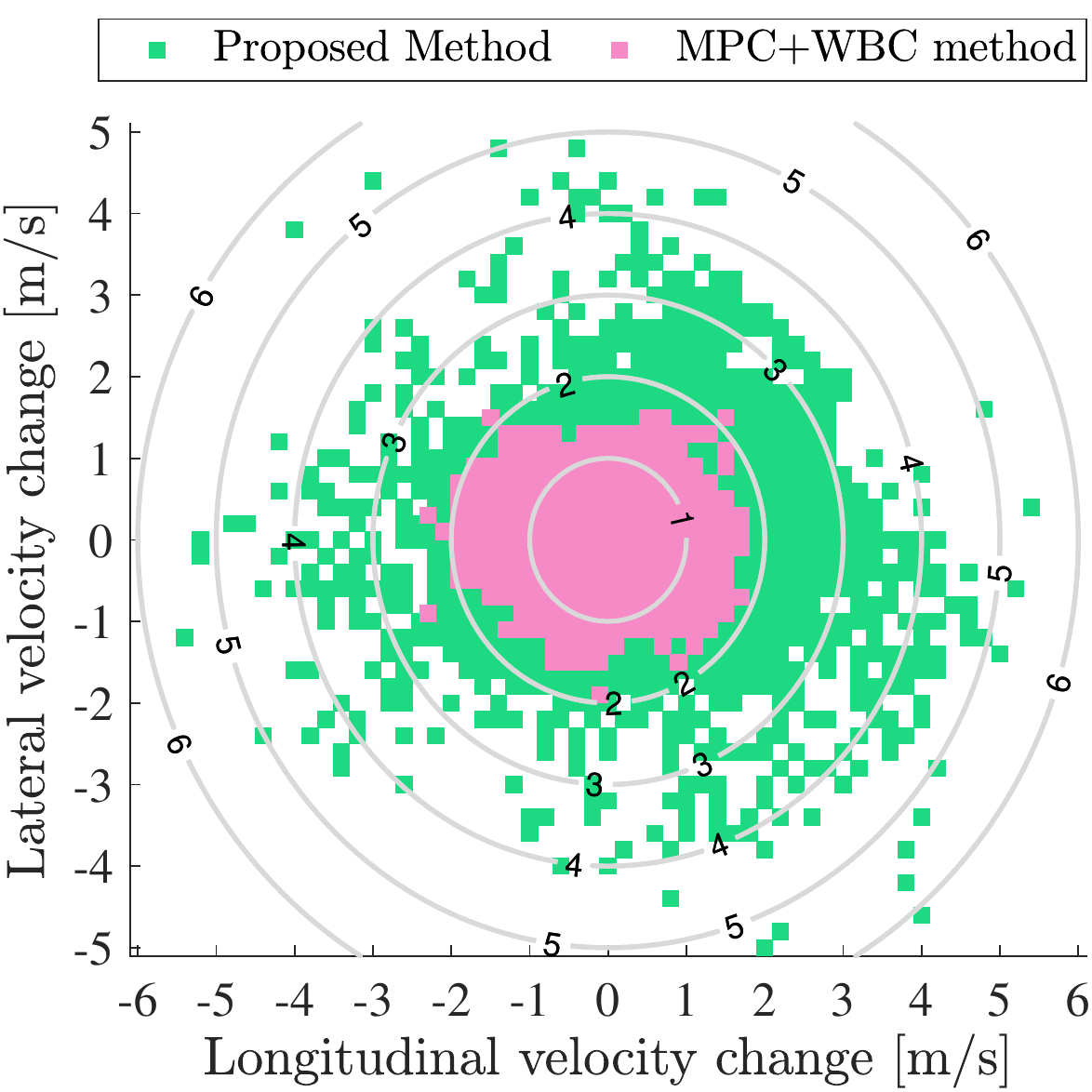}}
    \subfigure[ \footnotesize Timing $T_2$]{
    \includegraphics[width=0.235\linewidth]{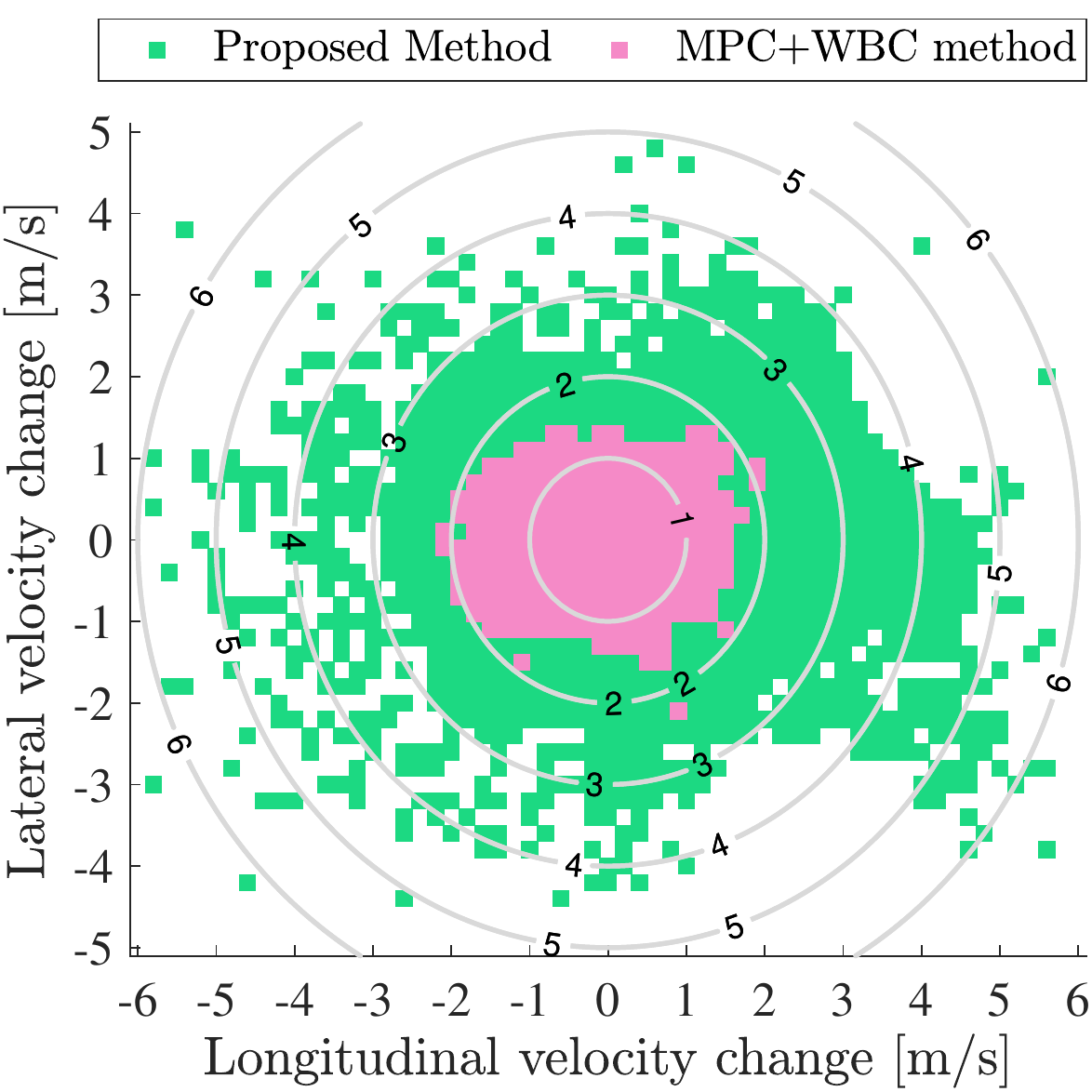}}
    \subfigure[ \footnotesize Timing $T_3$]{
    \includegraphics[width=0.235\linewidth]{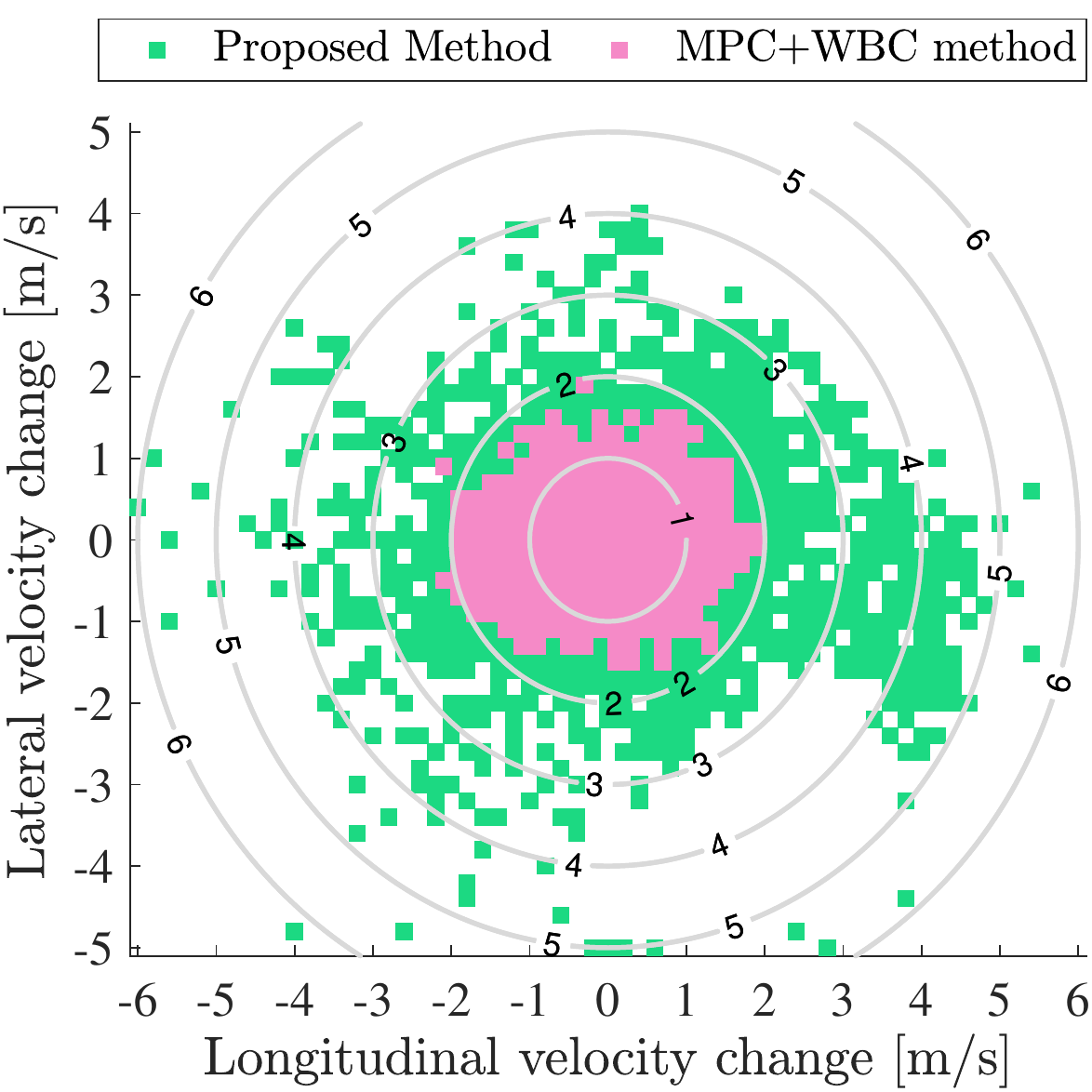}}
    \subfigure[ \footnotesize  Timing $T_4$]{
    \includegraphics[width=0.235\linewidth]{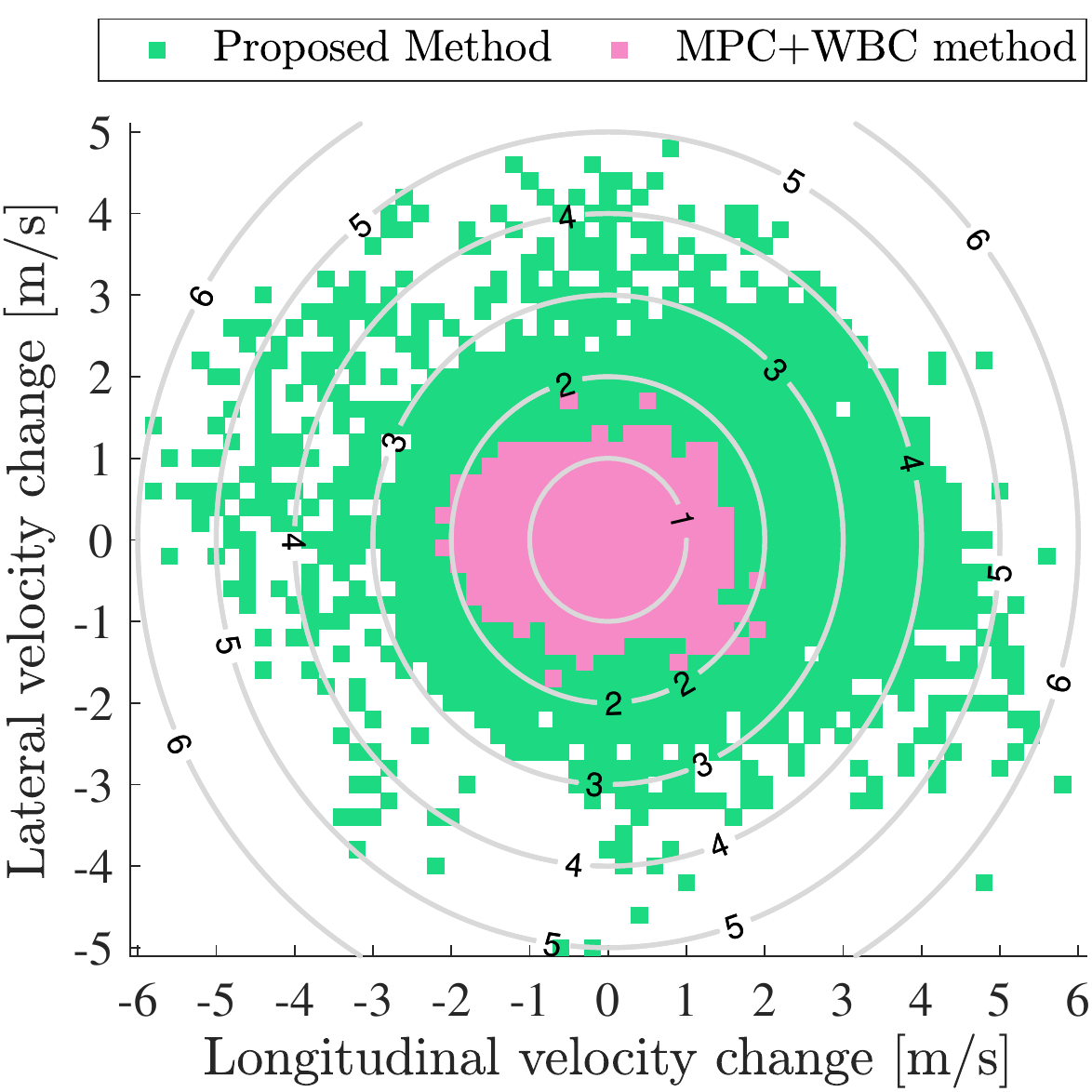}}
    \caption{ \footnotesize   Simulation results for {\bf trot} gait with different {\bf timings}.  }
\label{fig:timing_1}
\end{figure*}

\begin{figure*}[tp!]
    \centering
    \subfigure[ \footnotesize  Timing $T_1$]{\label{fig:timing_e}
    \includegraphics[width=0.235\linewidth]{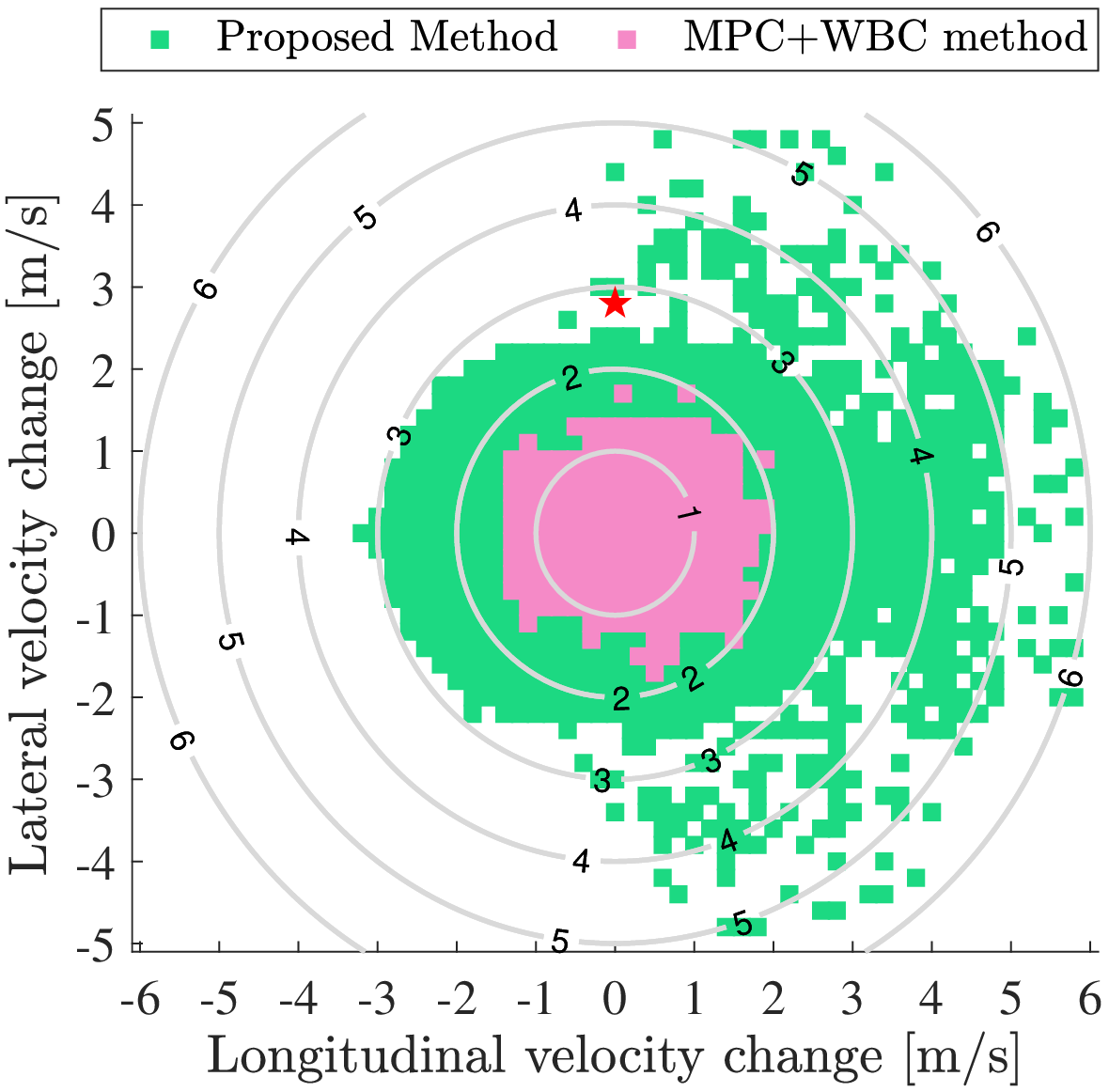}}
    \subfigure[ \footnotesize  Timing $T_2$]{
    \includegraphics[width=0.235\linewidth]{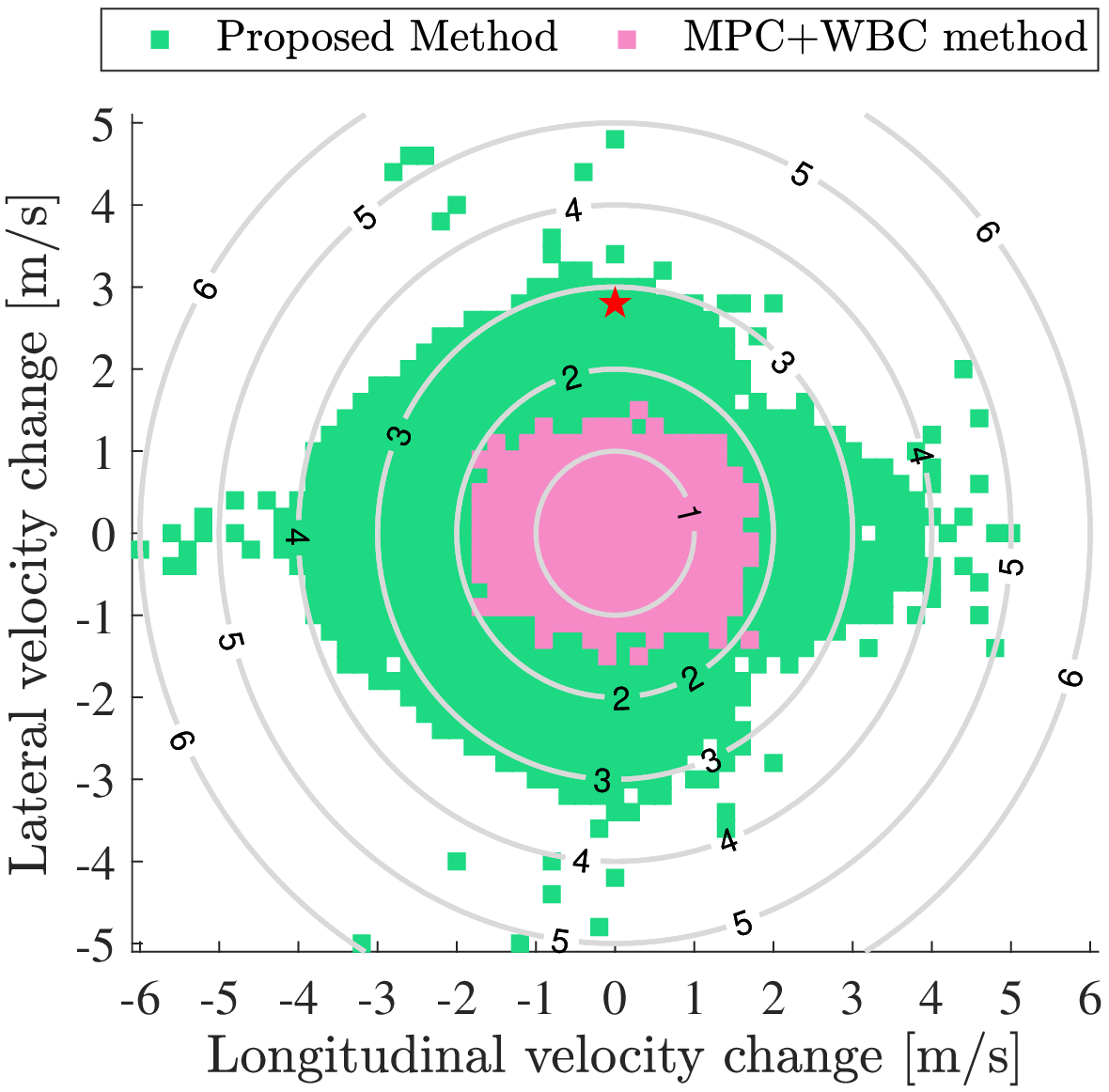}}
    \subfigure[ \footnotesize  Timing $T_3$]{\label{fig:timing_c1}
    \includegraphics[width=0.235\linewidth]{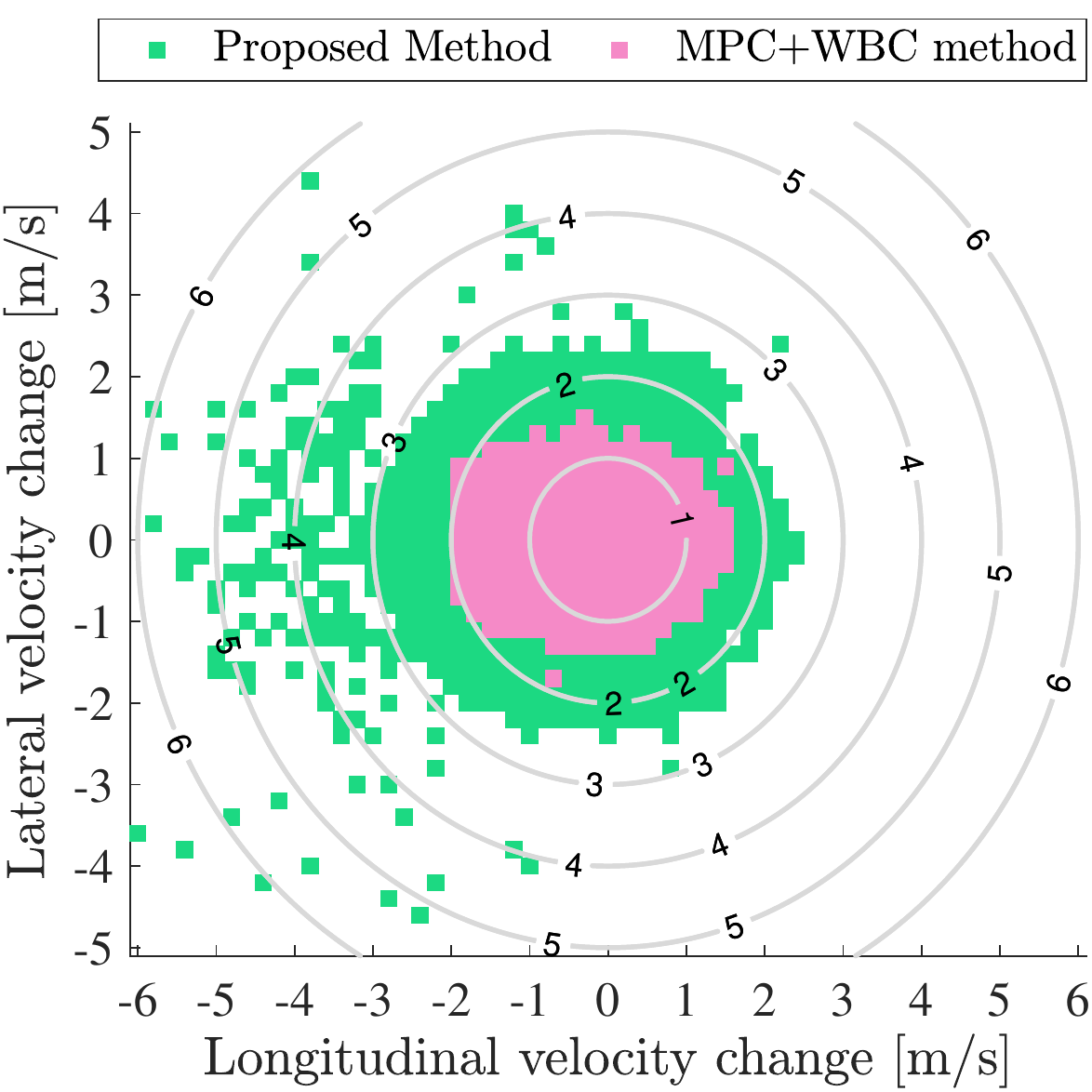}}
    \subfigure[ \footnotesize  Timing $T_4$]{\label{fig:timing_h}
    \includegraphics[width=0.235\linewidth]{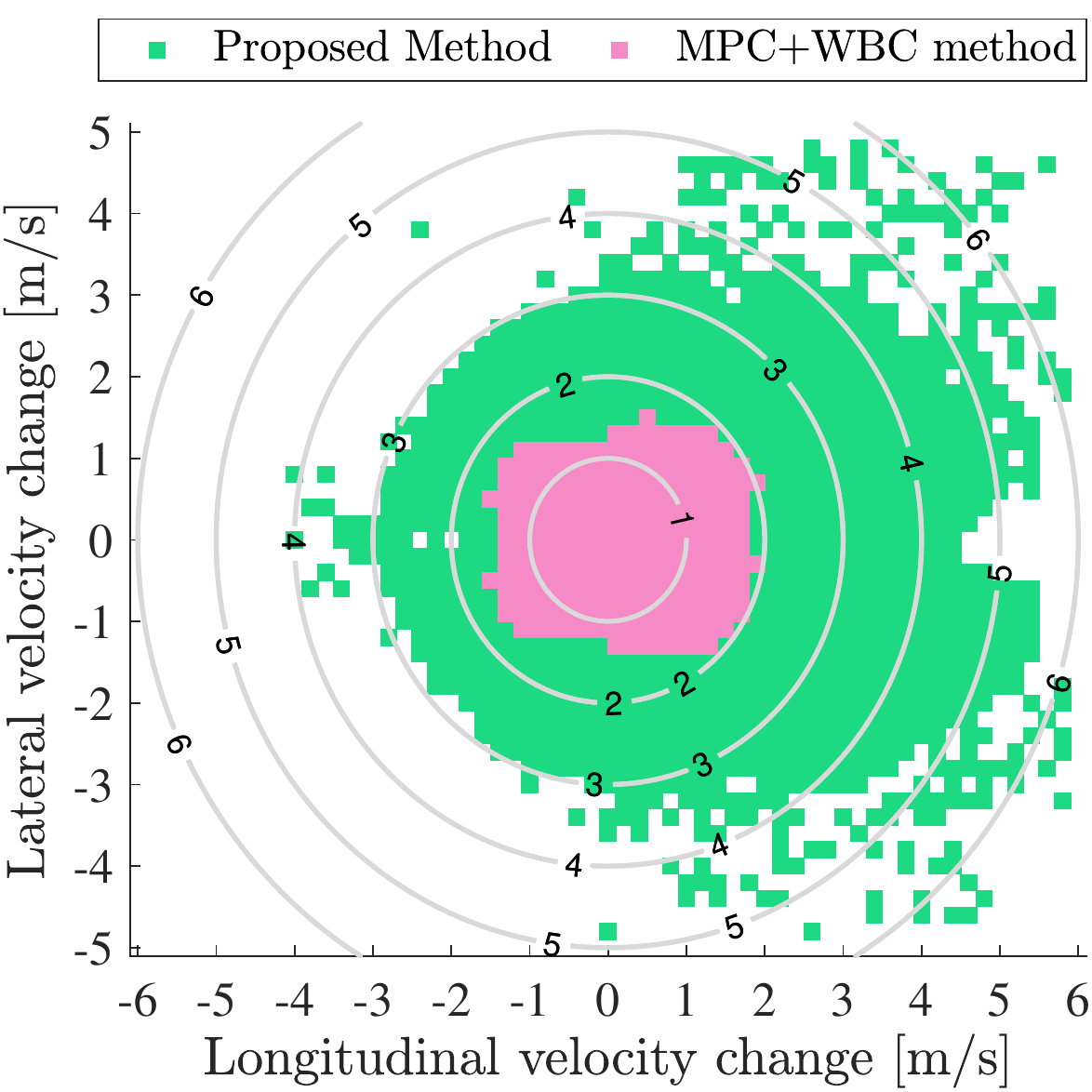}}
    \caption{ \footnotesize   Simulation results for {\bf bound} gait with different {\bf timings}. Black stars indicate the particular test scenario in Fig.~\ref{fig:simu_state_traj_2} and Fig.~\ref{fig:simu_state_traj_4}. }
    \label{fig:timing_2}
\end{figure*}

\begin{figure*}[tp!]
    \centering
    \subfigure[ \footnotesize  Timing $T_1$]{
    \includegraphics[width=0.235\linewidth]{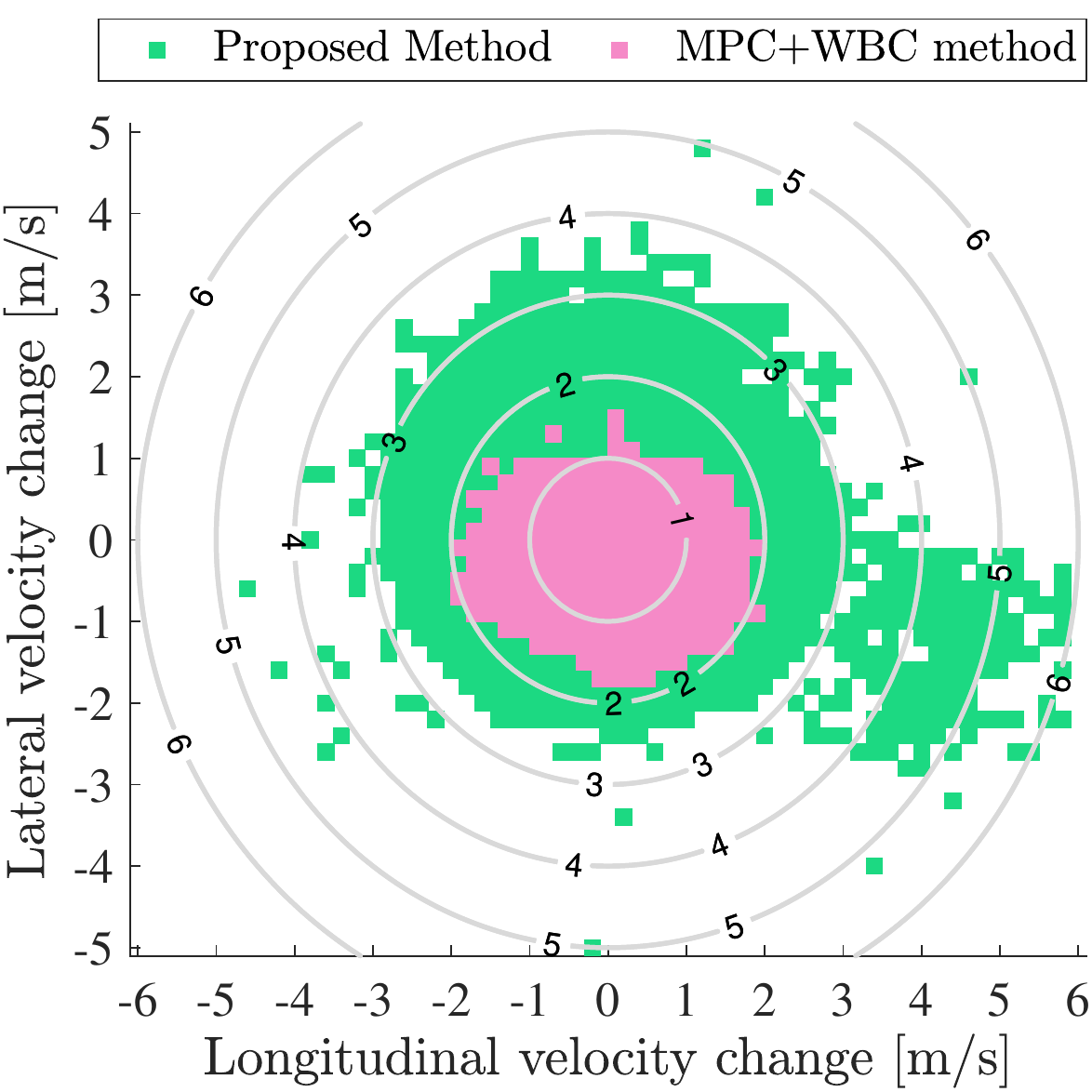}}
    \subfigure[ \footnotesize  Timing $T_2$]{
    \includegraphics[width=0.235\linewidth]{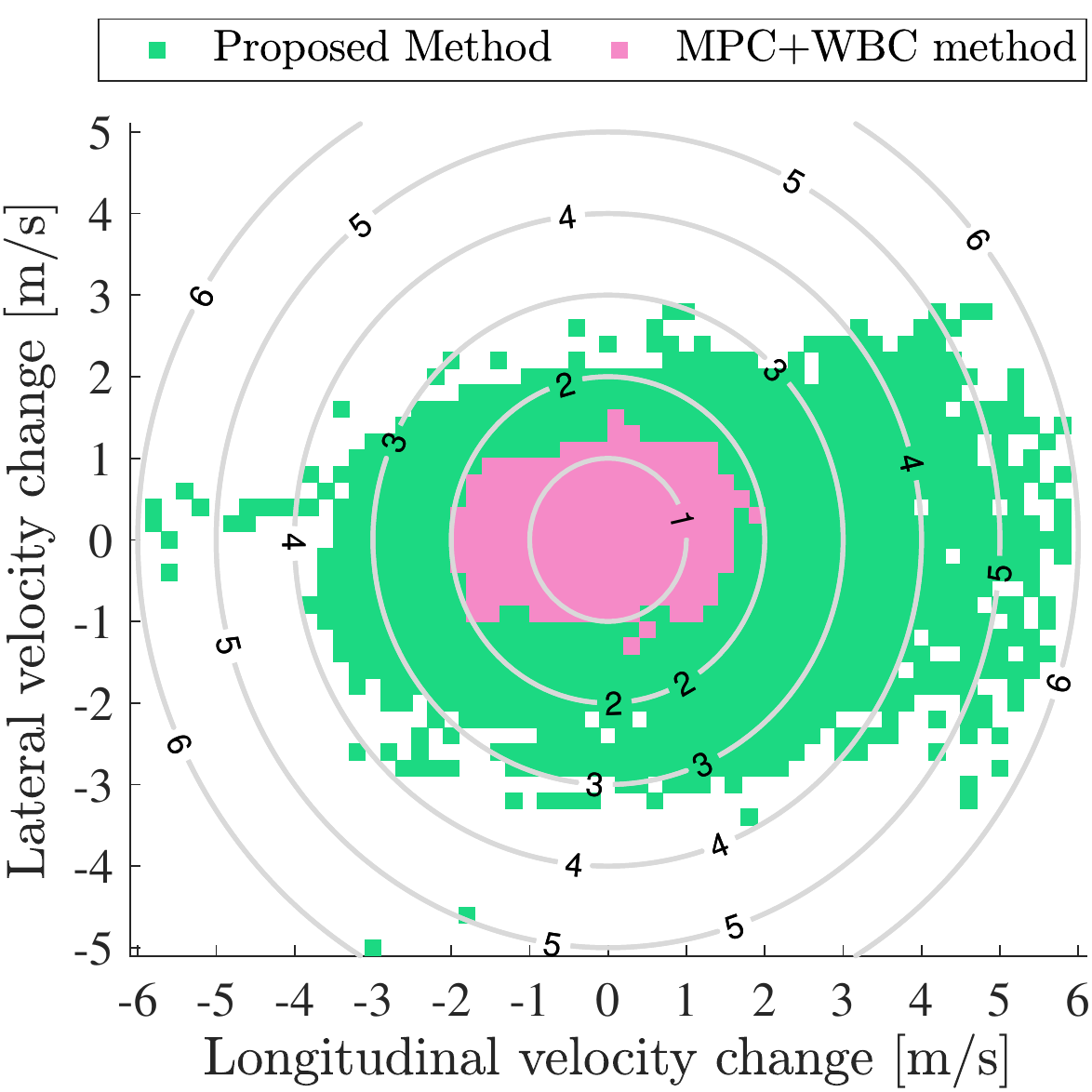}}
    \subfigure[ \footnotesize  Timing $T_3$]{
    \includegraphics[width=0.235\linewidth]{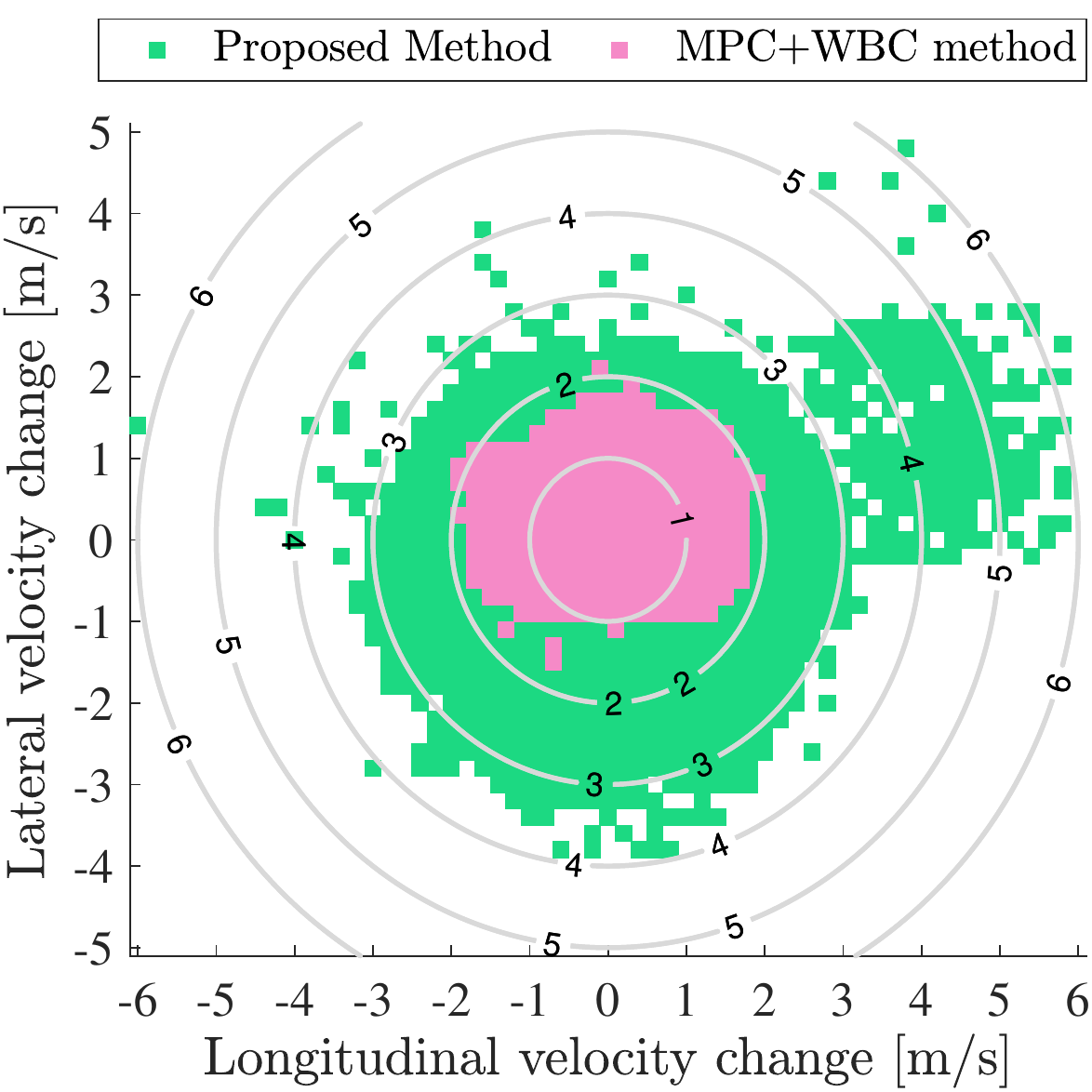}}
    \subfigure[ \footnotesize  Timing $T_4$]{
    \includegraphics[width=0.235\linewidth]{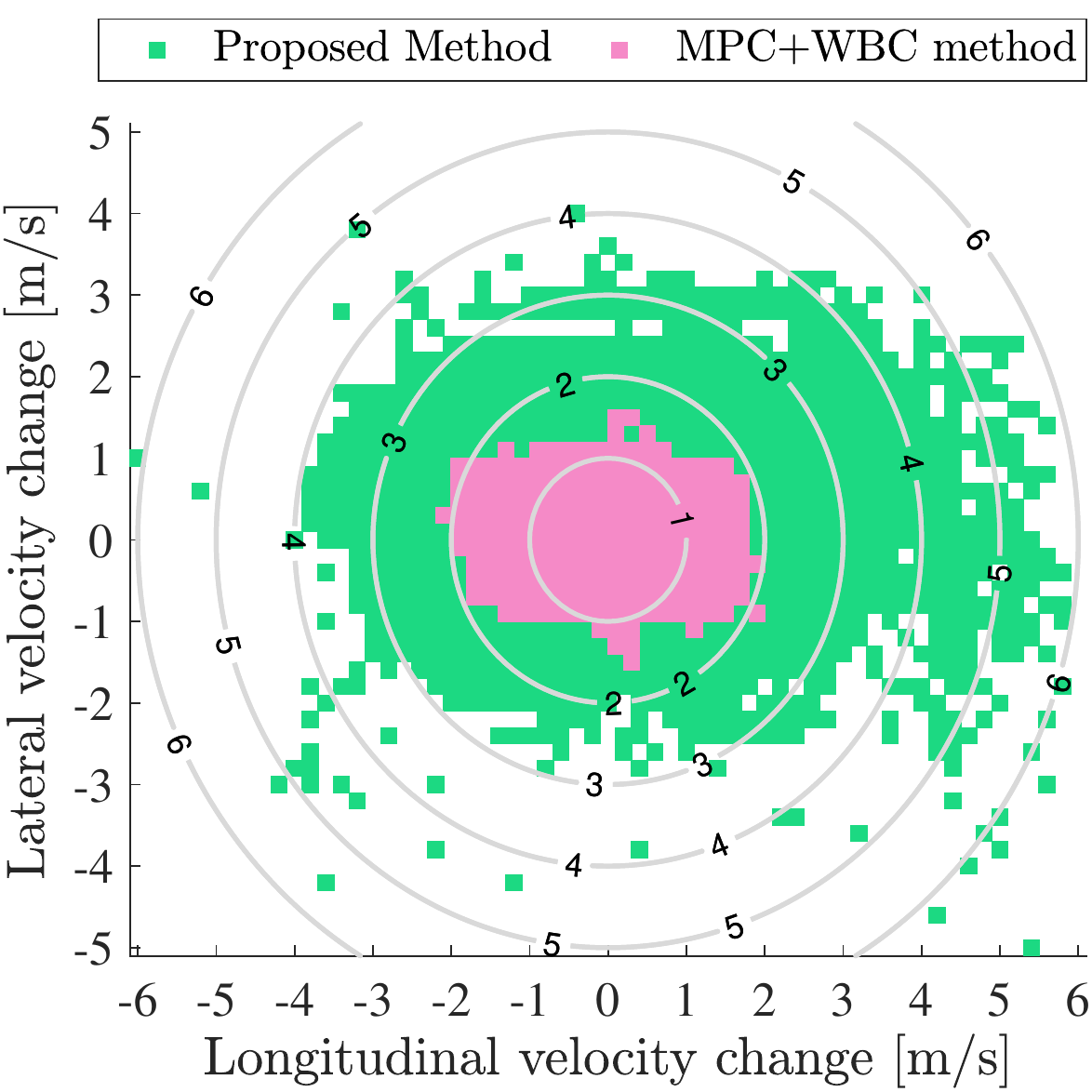}}
    \caption{ \footnotesize   Simulation results for {\bf pace} gait with different {\bf timings}.  }
    \label{fig:timing_3}
\end{figure*}

In addition to the big-picture comparison, a more interesting observation from the results is that the shapes of successful recoverable regions for the baseline approach look similar across all test cases, while they vary quite significantly with the proposed approach. The main underlying reason would probably be the explicit consideration of gait and timing information in the capturability analysis. For example, if the front legs are in mid-stance and rear legs are swinging (Fig.~\ref{fig:timing_c1}), the robot is incapable of rejecting large forward pushes (with a max of about $\SI{2.5}{m\per s}$). While if the disturbance is applied a bit later when the rear legs are transitioning from swing to stance and the front legs are about to swing, the robot is capable of rejecting a much larger forward push (up to about $\SI{6}{m\per s}$). With the proposed approach, the influences of gaits and impact timings are more explicitly characterized, making future investigations on gait timing adaptation more promising. Such an explicit characterization of the gait and timing influences and the accompany improvements in push recovery ability for quadrupedal locomotion is among the major contributions of this work.

\begin{figure}[bp!]
    \centering
    \subfigure[ \footnotesize  CoM velocity]{
    \includegraphics[width=0.3\linewidth]{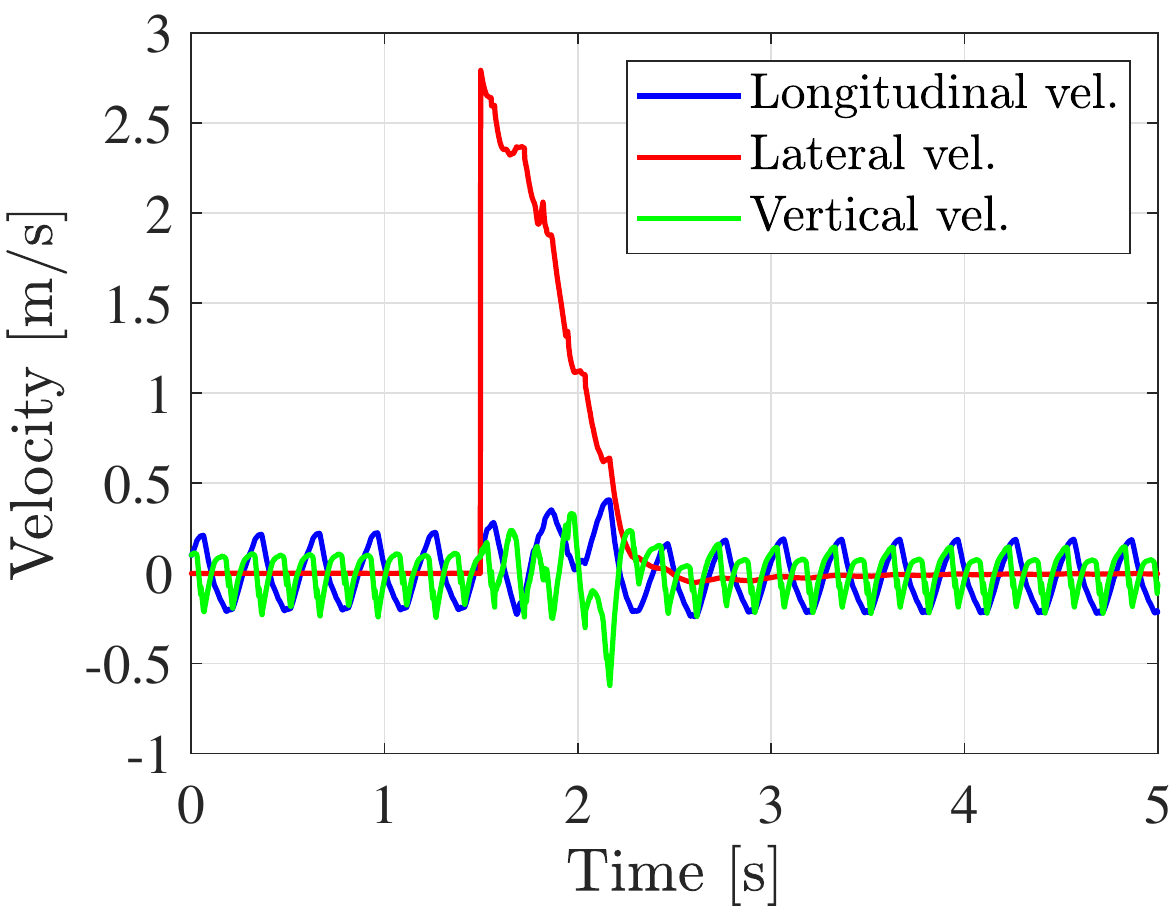}\label{fig:simu_state_traj_2_b}}
    \subfigure[ \footnotesize  Floating base orientation]{
    \includegraphics[width=0.3\linewidth]{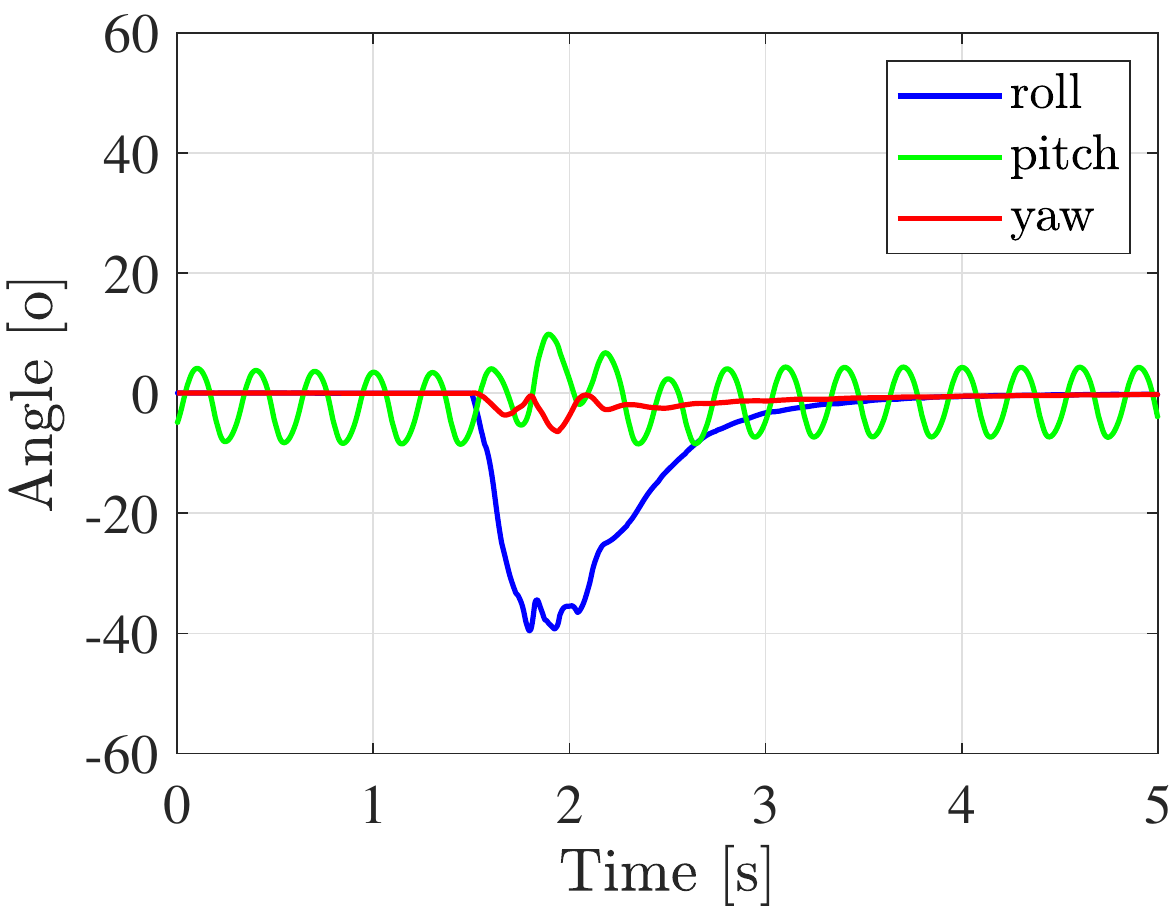}\label{fig:simu_state_traj_2_c}}
    \caption{ \footnotesize Simulation result with the \textbf{proposed} controller. State trajectories with \textbf{bound} gait under impact $\delta v_x =\SI{ 0.0 }{m \per s} , \delta v_y = \SI{ 2.8 }{m \per s}$ at timing \textbf{$T_2$}.}
    \label{fig:simu_state_traj_2}
\end{figure}

\begin{figure}[bp!]
    \centering
    \subfigure[ \footnotesize  CoM velocity]{
    \includegraphics[width=0.3\linewidth]{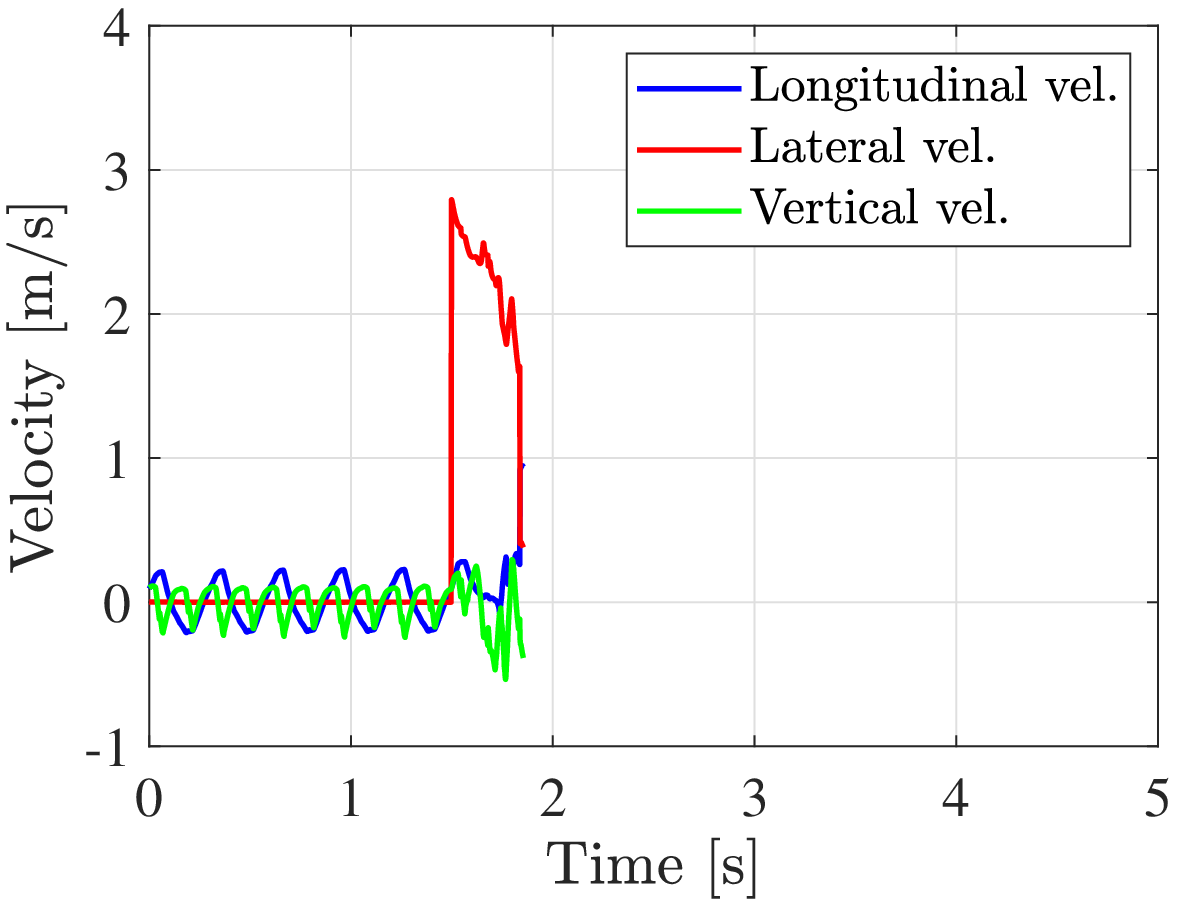}}
    \subfigure[ \footnotesize  Floating base orientation]{
    \includegraphics[width=0.3\linewidth]{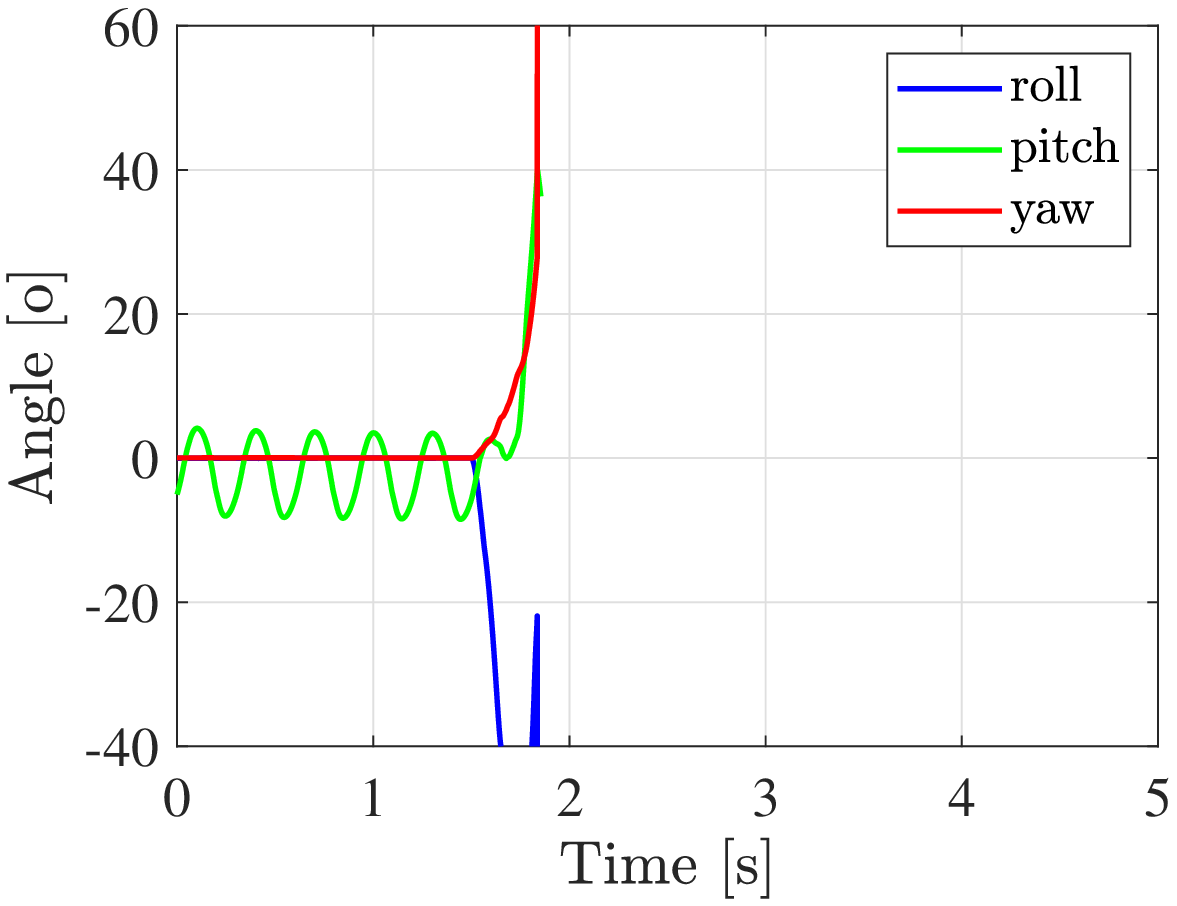}}
    \caption{ \footnotesize Simulation result with the \textbf{baseline} controller. State trajectories with \textbf{bound} gait under impact $\delta v_x =\SI{ 0.0 }{m \per s} , \delta v_y = \SI{ 2.8 }{m \per s}$ at timing \textbf{$T_2$}.}
    \label{fig:simu_state_traj_2}
\end{figure}

\begin{figure}[bp!]
    \centering
    \subfigure[ \footnotesize  CoM velocity    ]{
    \includegraphics[width=0.3\linewidth]{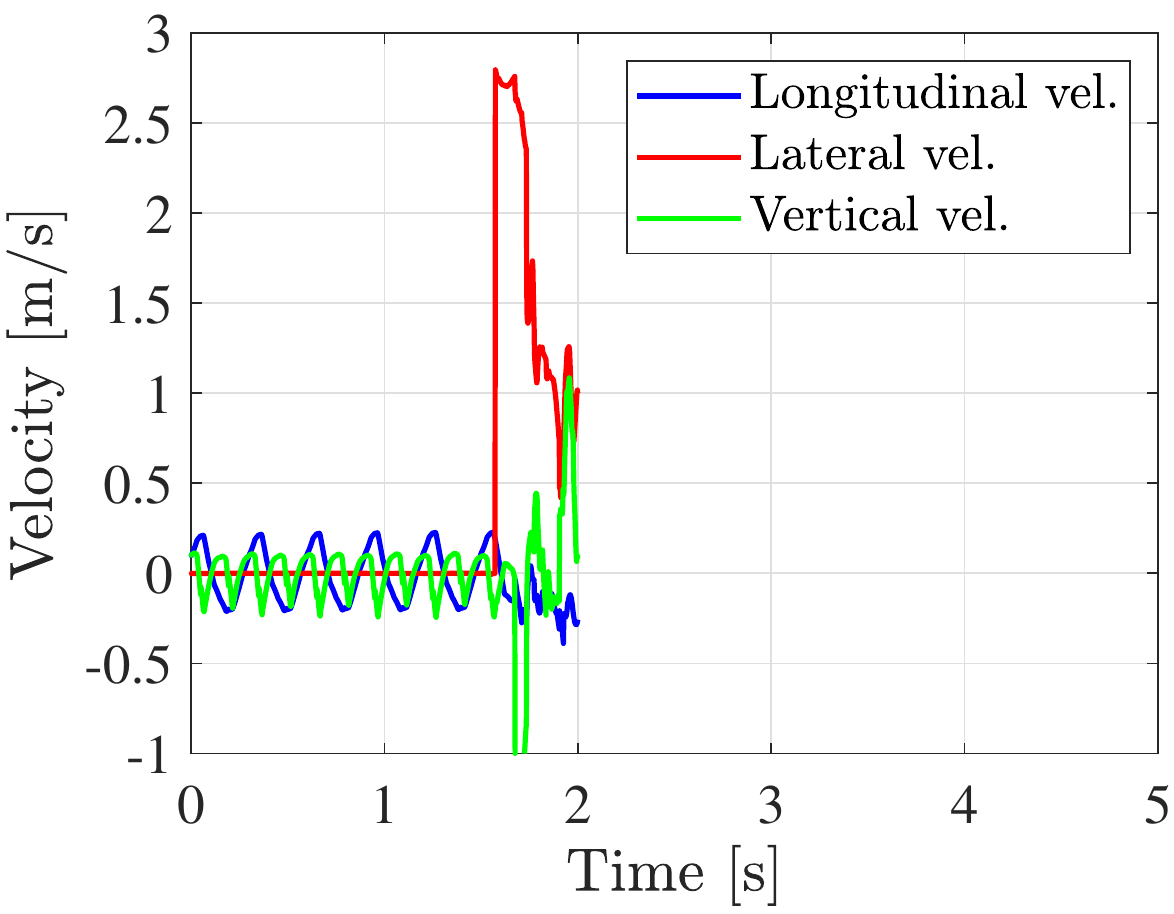}\label{fig:simu_state_traj_4_b}}
    \subfigure[ \footnotesize  Floating base orientation]{
    \includegraphics[width=0.3\linewidth]{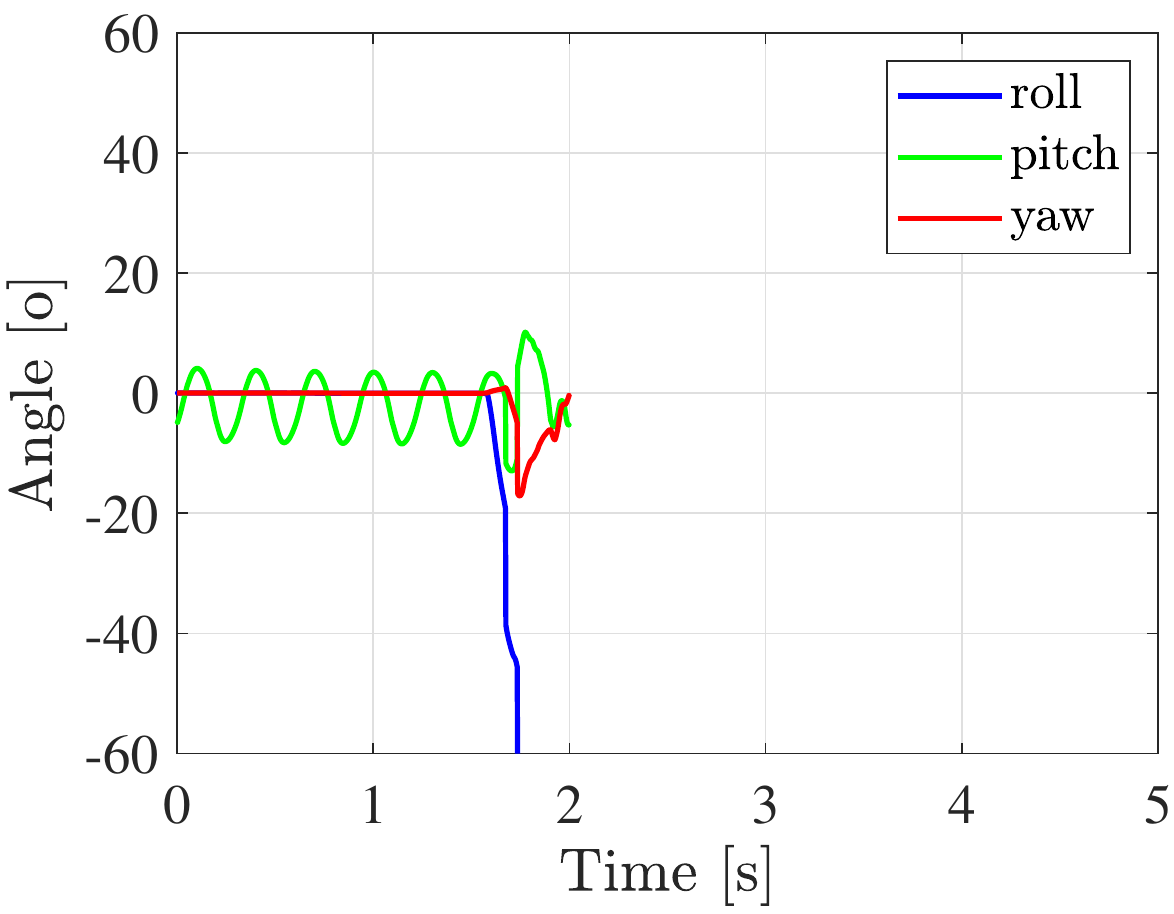}}
    \caption{ \footnotesize Simulation result with the \textbf{proposed} controller. State trajectories with \textbf{bound} gait under impact $\delta v_x =\SI{ 0.0 }{m \per s} , \delta v_y = \SI{ 2.8 }{m \per s}$ at timing \textbf{$T_1$}.}
    \label{fig:simu_state_traj_4}
\end{figure}

To scope into the detailed behavior of the quadrupeds during push recovery with the proposed framework, we investigate the simulated state trajectories of the quadruped with the proposed controller under the same impact $(\delta v_x,\delta v_y) = (\SI{0}{m \per s},\SI{2.8}{m \per s})$ (i.e., a lateral impact of magnitude $\SI{25.2}{Ns}$) with different disturbance timings. Fig.~\ref{fig:simu_state_traj_2} and~\ref{fig:simu_state_traj_4} show the trajectories of the robot state in reaction to the disturbances under bound gait. From the figures, it can be first observed that the robot's state is periodically changing when bounding in place, which can be characterized by the proposed dynamic balance concept. Second, it is clear that the impact timing has a significant influence on the push recovery capability. Due to the specific contact configuration at the moment of impact, the robot's ability to modify its state before contact switching is different, which is clear by comparing Fig.~\ref{fig:simu_state_traj_2_b} and Fig.~\ref{fig:simu_state_traj_4_b}. Furthermore, it can be observed from Fig.~\ref{fig:simu_state_traj_2_c} that the robot's floating base tilts up to $\ang{40}$ in the roll direction during the recovery phase. Such a rotation is clearly beyond the range in which the LIPM simplification holds, while the proposed controller remains successful in rejecting the external disturbance. 

By investigating the simulation results presented in this subsection, the advantage of the proposed push recovery controller over existing methods has been clearly demonstrated via extensive tests with different impacts at different timings. In terms of the area of manageable external impacts, the proposed approach is at least {\bf  300\%} that of the baseline approach. The need of considering dynamic balance is verified through looking at the detailed response to particular impacts. Detailed comparison between the robot's reaction to the same impacts with different controllers further strengthens the contributions of the proposed capturability-based push recovery controller. In the following subsection, we implement the proposed controller on a real quadruped robot to further test its effectiveness.

\begin{figure}[tbp!]
	\centering
	{\includegraphics[width=\linewidth]{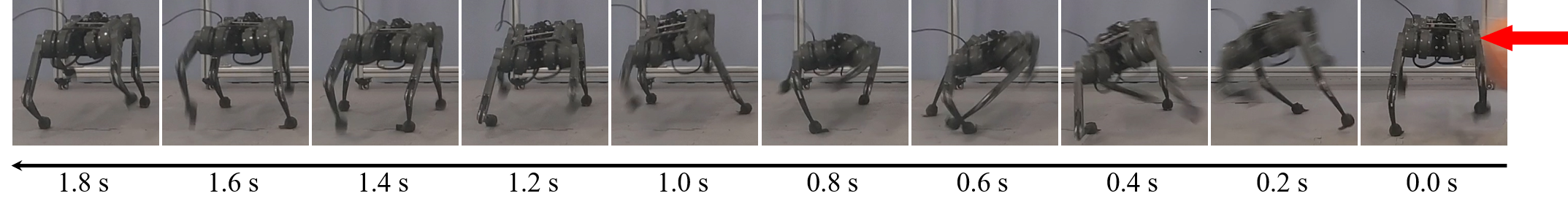}}
	\caption{ \footnotesize   Snapshots of hardware experiment with {\bf trot} gait and the {\bf proposed} controller. Corresponding data given in Fig.~\ref{fig:exp1}.}
	\label{fig:exp1_snap}
	\vspace{-10px}
\end{figure}
\begin{figure}[tbp!]
	\centering
	{\includegraphics[width=\linewidth]{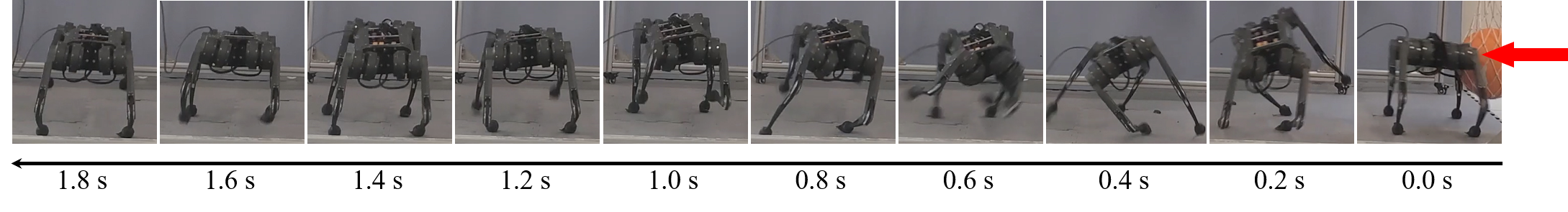}}
	\caption{ \footnotesize   Snapshots of hardware experiment with \textbf{bound} gait and the \textbf{proposed} controller. Corresponding data given in Fig.~\ref{fig:exp2}.}
	\label{fig:exp2_snap}
	\vspace{-10px}
\end{figure}

\begin{figure}[tbp!]
	\centering
	{\includegraphics[width=\linewidth]{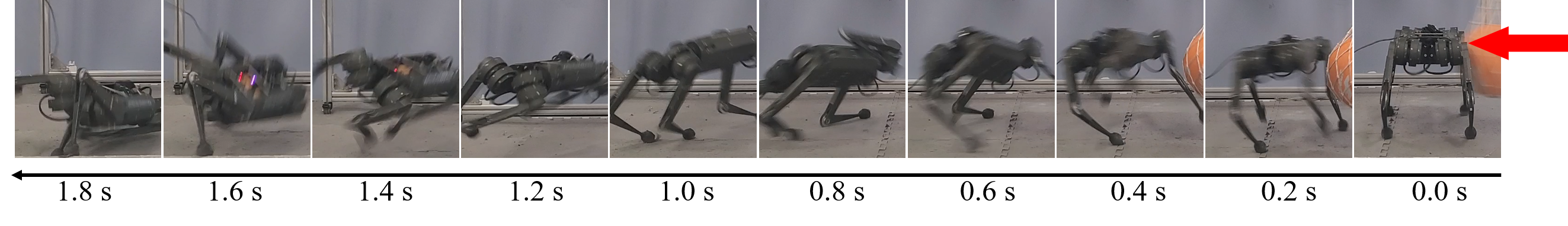}}
	\caption{ \footnotesize   Snapshots of hardware experiment with \textbf{trot} gait and the \textbf{baseline} controller. Corresponding data given in Fig.~\ref{fig:exp1_mpc}.}
	\label{fig:exp3_snap_mpc_trotting}
	\vspace{-10px}
\end{figure}

\begin{figure}[tbp!]
	\centering
	{\includegraphics[width=\linewidth]{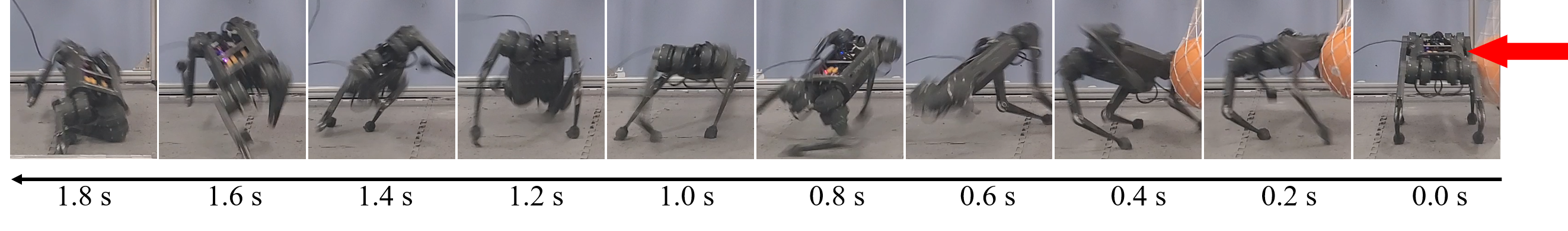}}
	\caption{ \footnotesize   Snapshots of hardware experiment with \textbf{bound} gait and the \textbf{baseline} controller. Corresponding data given in Fig.~\ref{fig:exp2_mpc}.}
	\label{fig:exp3_snap_mpc_bounding}
	\vspace{-10px}
\end{figure}

\begin{figure}[tbp!]
\centering
\subfigure[ \footnotesize  CoM velocity ]{
\includegraphics[width=0.3\linewidth]{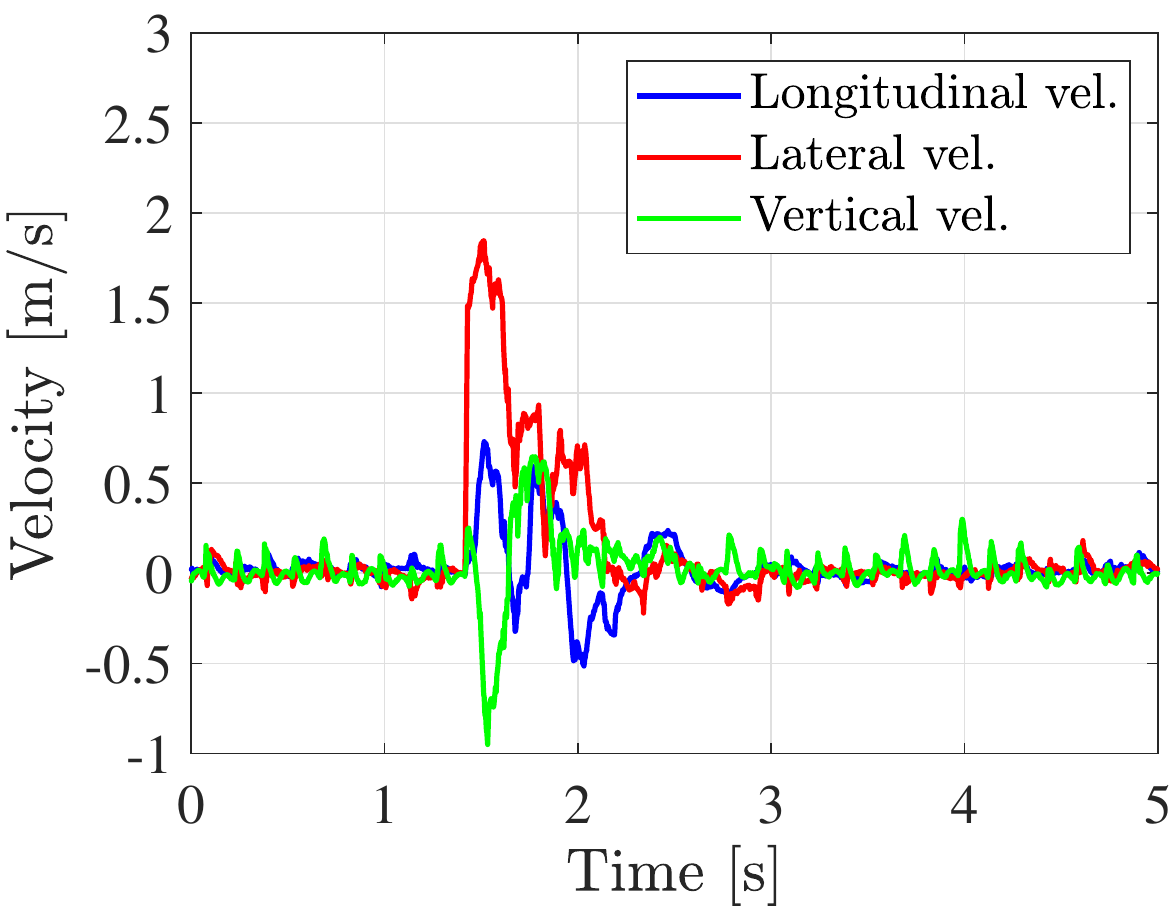}}
\subfigure[ \footnotesize  Floating base orientation ]{
\includegraphics[width=0.3\linewidth]{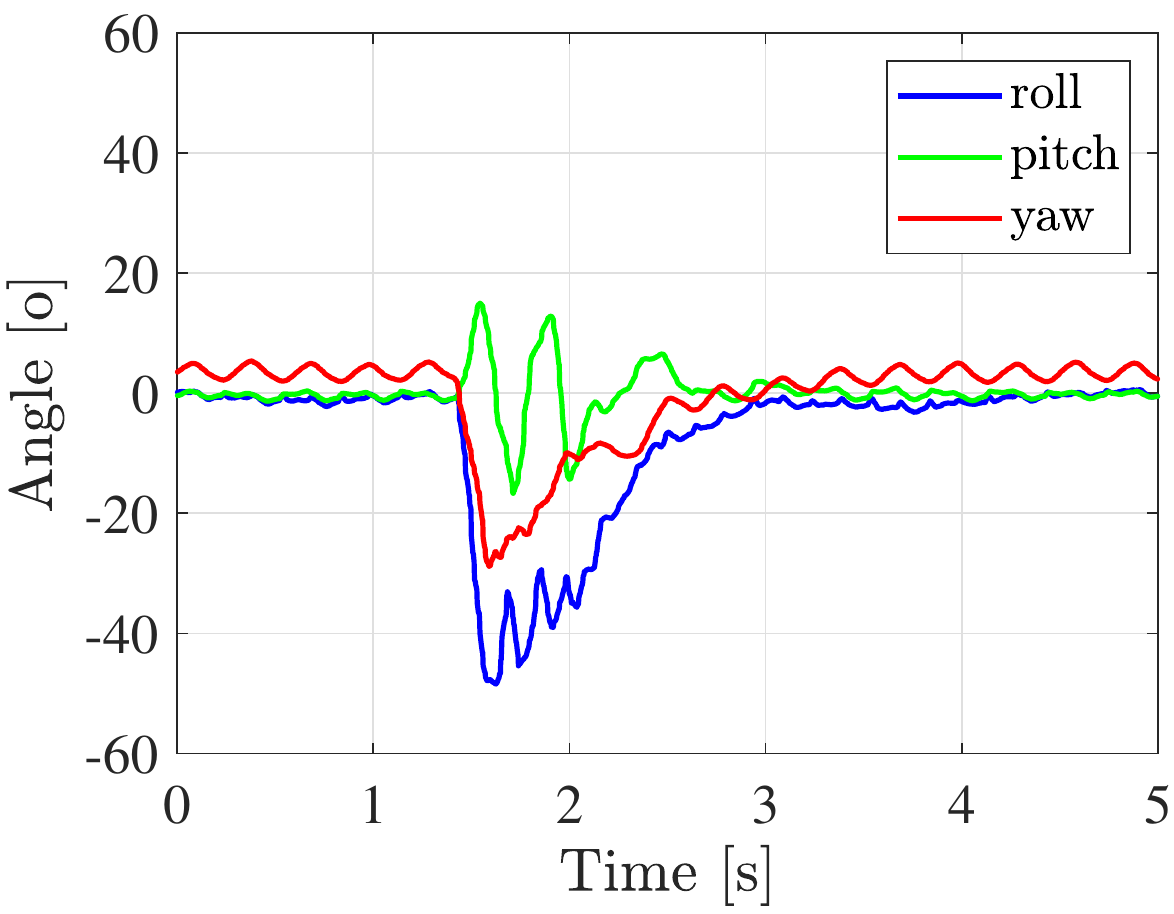}}
\caption{ \footnotesize   Experimental data with proposed approach and \textbf{trot} gait. Impact is about $\SI{18.9}{Ns}$ ($\SI{1.8}{m\per s}\times \SI{10.5}{kg}$).}
\label{fig:exp1}
\end{figure}

\begin{figure}[tbp!]
\centering
\subfigure[ \footnotesize  CoM velocity ]{
\includegraphics[width=0.3\linewidth]{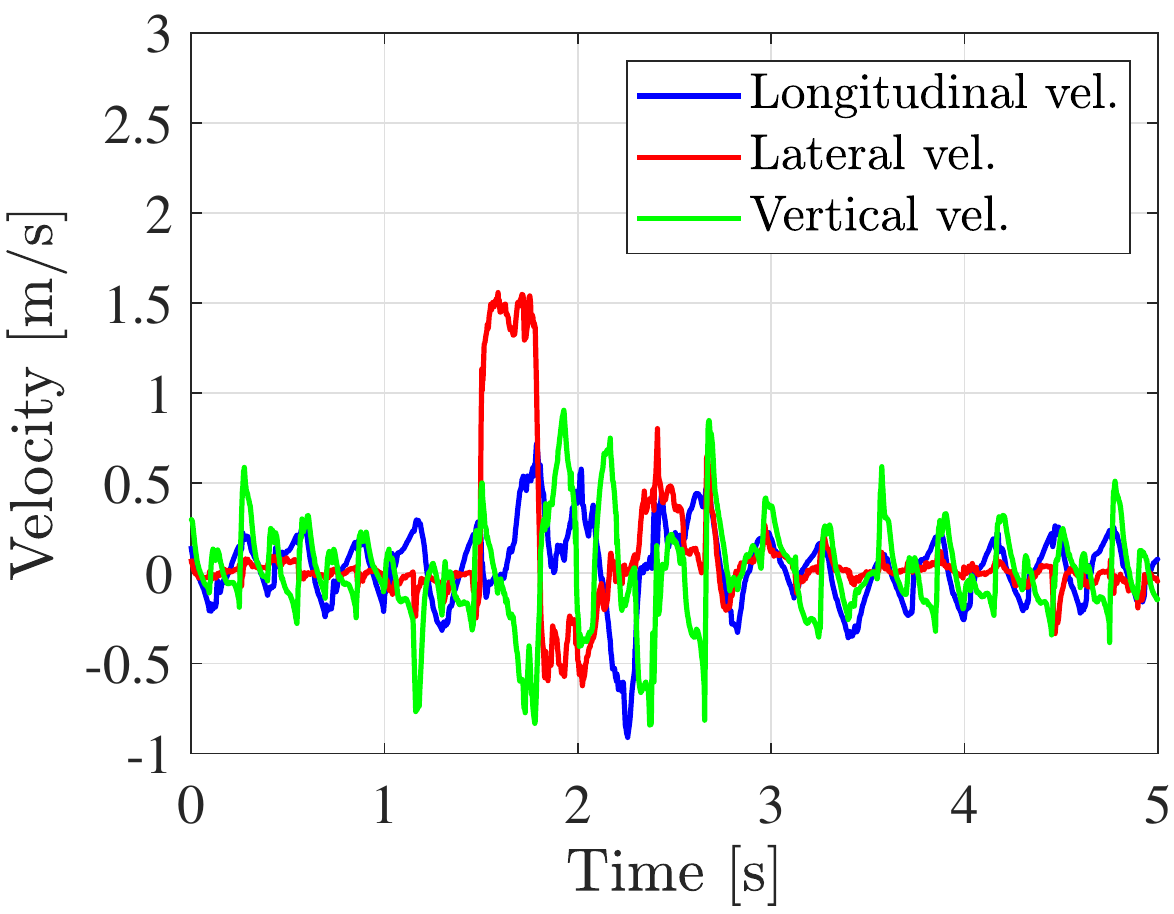}}
\subfigure[ \footnotesize  Floating base orientation ]{
\includegraphics[width=0.3\linewidth]{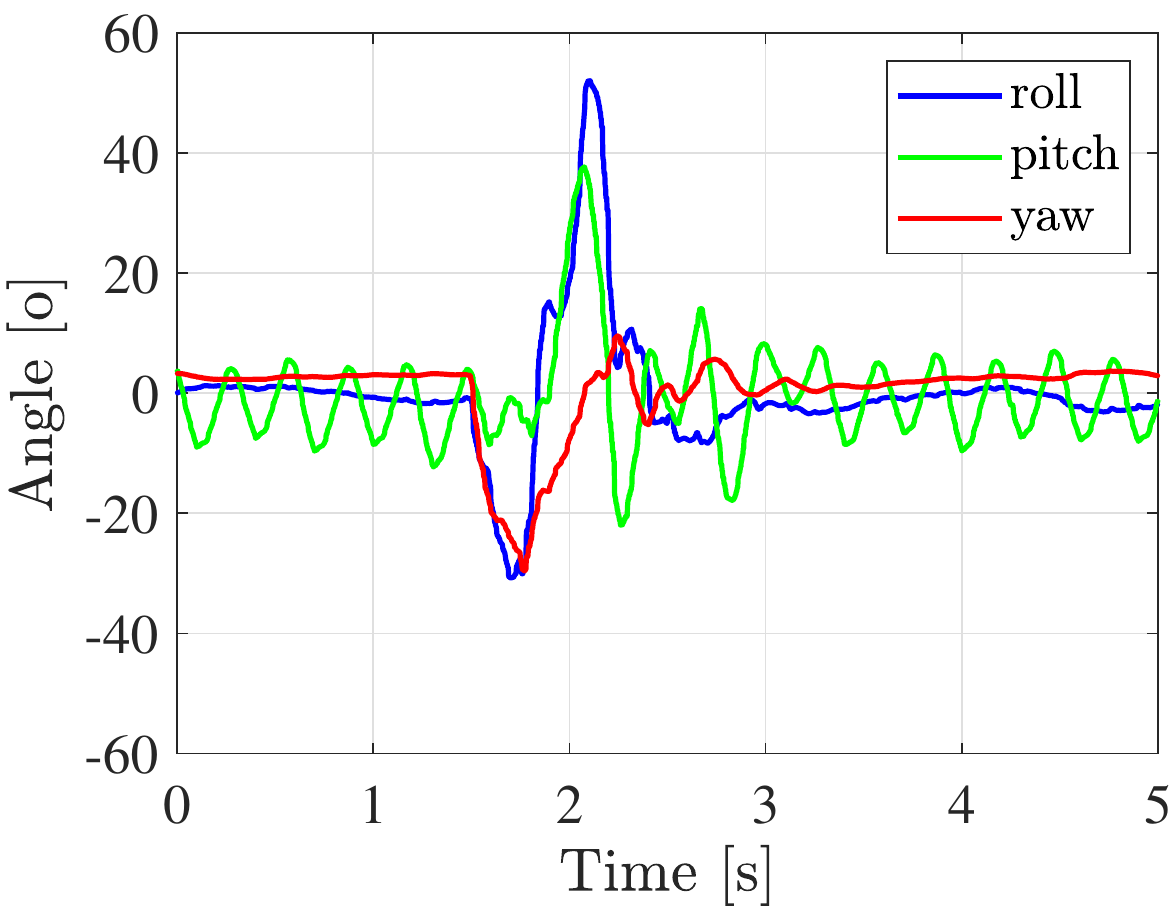}}
\caption{ \footnotesize   Experimental data with \textbf{proposed} controller and \textbf{bound} gait. Impact is about $\SI{16.8}{Ns}$ ($\SI{1.6}{m\per s}\times \SI{10.5}{kg}$).}
\label{fig:exp2}
\end{figure}

\subsection{Experiment Validations}
Extensive hardware experiments have also been performed to validate the proposed approach. For hardware experiments, we test the proposed approach with lateral disturbances that are simple to generate and analyze quantitatively. To quantify the performance of the proposed framework precisely, the external disturbances are applied via a pendulum ball with mass of $\SI{6}{kg}$. During experiments, the ball is pulled to a angle and released so that it moves under gravity and hit the robot at the lowest point (Fig.~\ref{fig:mc}). 

Similar to the previous simulation validations, both the proposed and the baseline MPC+WBIC based approaches are tested and compared. Overall, the experimental tests show that the proposed approach performs up to {\bf 200\%} as good as the baseline controller. The reader is kindly referred to the supplementary video for additional hardware and simulation experiments. 

\begin{figure}[tbp!]
\centering
\subfigure[ \footnotesize  CoM velocity]{
\includegraphics[width=0.3\linewidth]{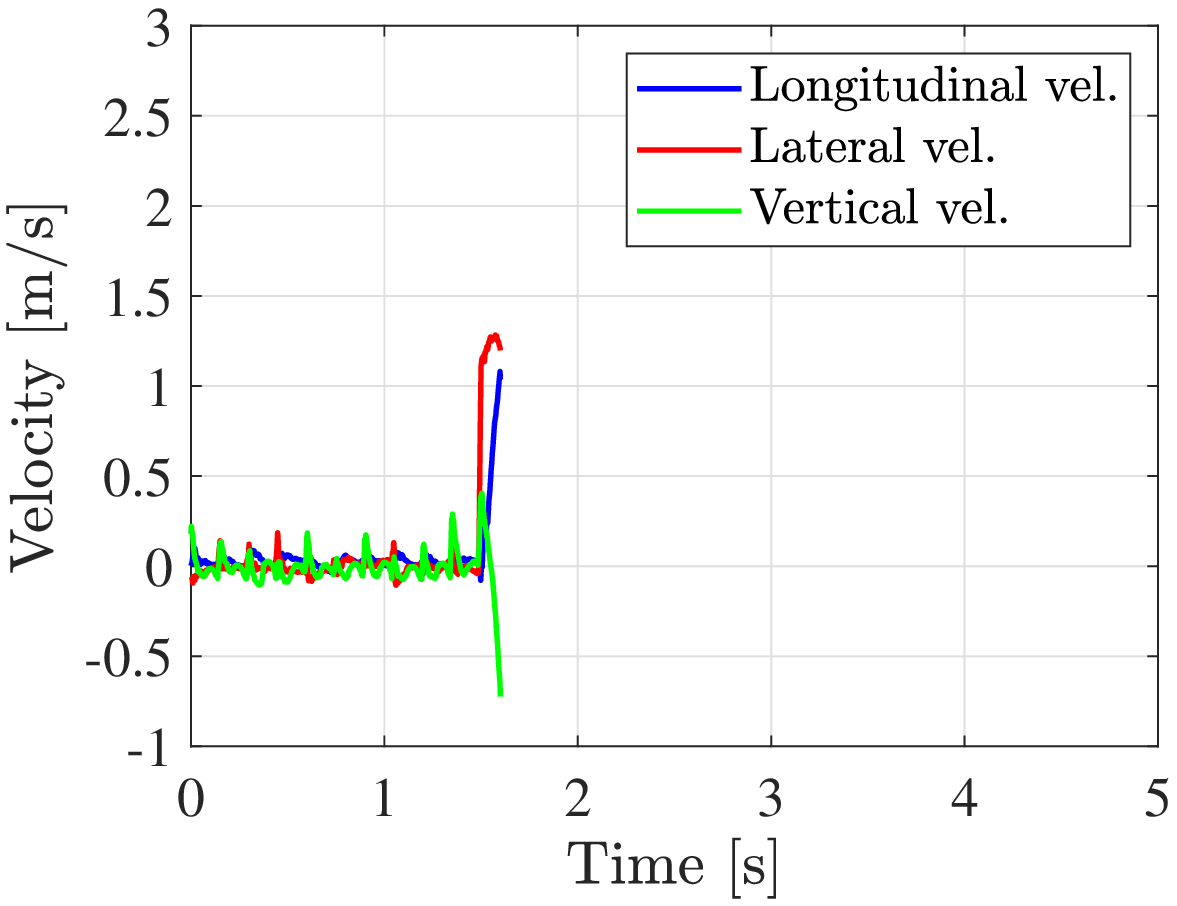}}
\subfigure[ \footnotesize  Floating base orientation]{
\includegraphics[width=0.3\linewidth]{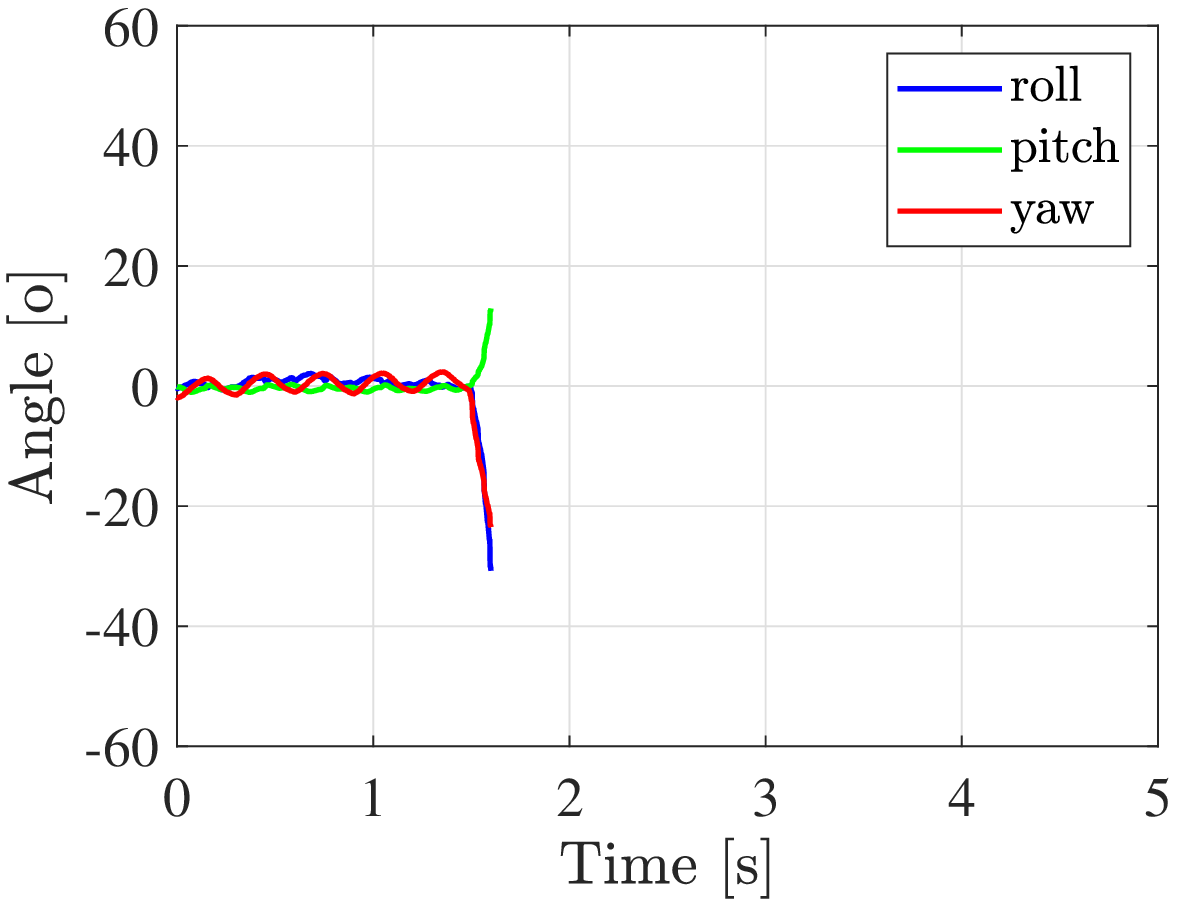}}
\caption{ \footnotesize   Experimental data with \textbf{baseline} controller and \textbf{trot} gait. }
\label{fig:exp1_mpc}
\end{figure}

\begin{figure}[tbp!]
\centering
\subfigure[ \footnotesize  CoM velocity]{
\includegraphics[width=0.3\linewidth]{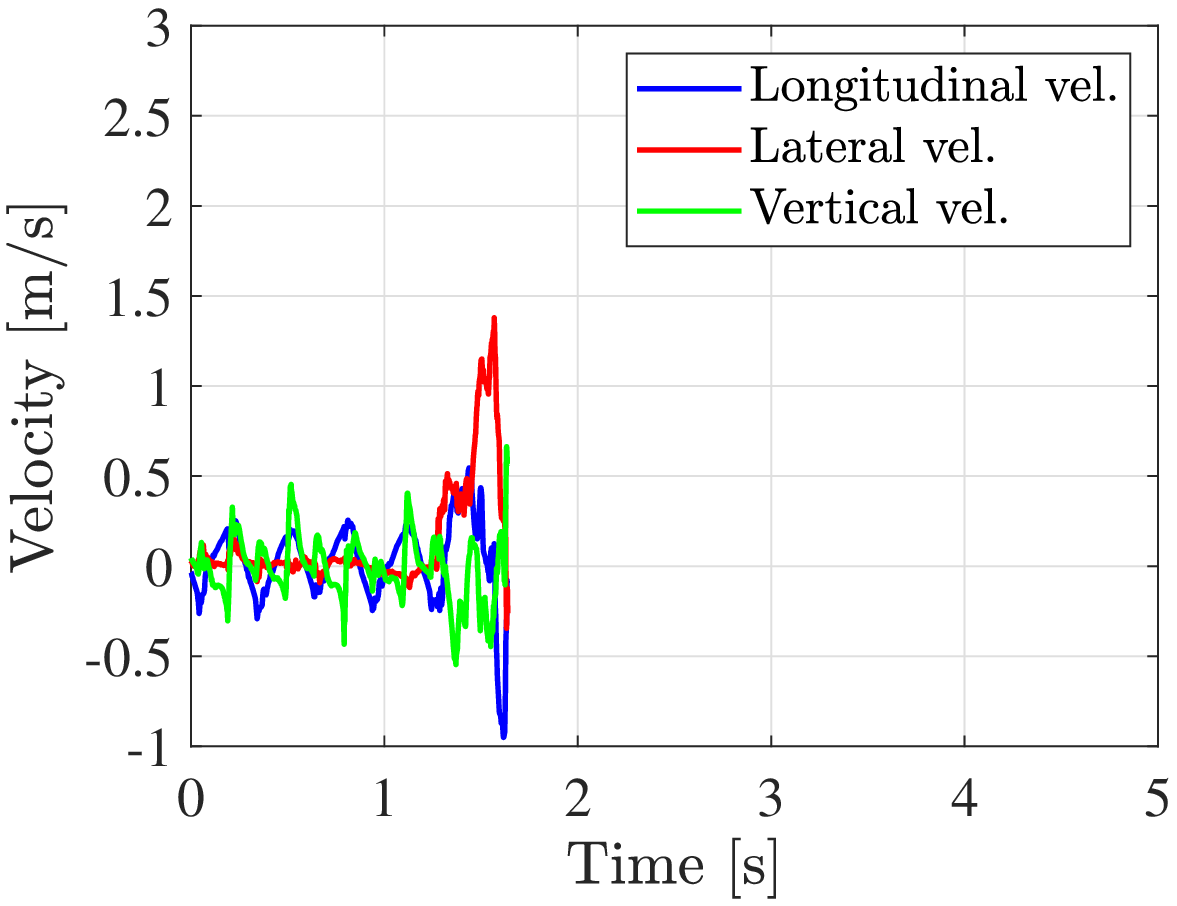}}
\subfigure[ \footnotesize  Floating base orientation]{
\includegraphics[width=0.3\linewidth]{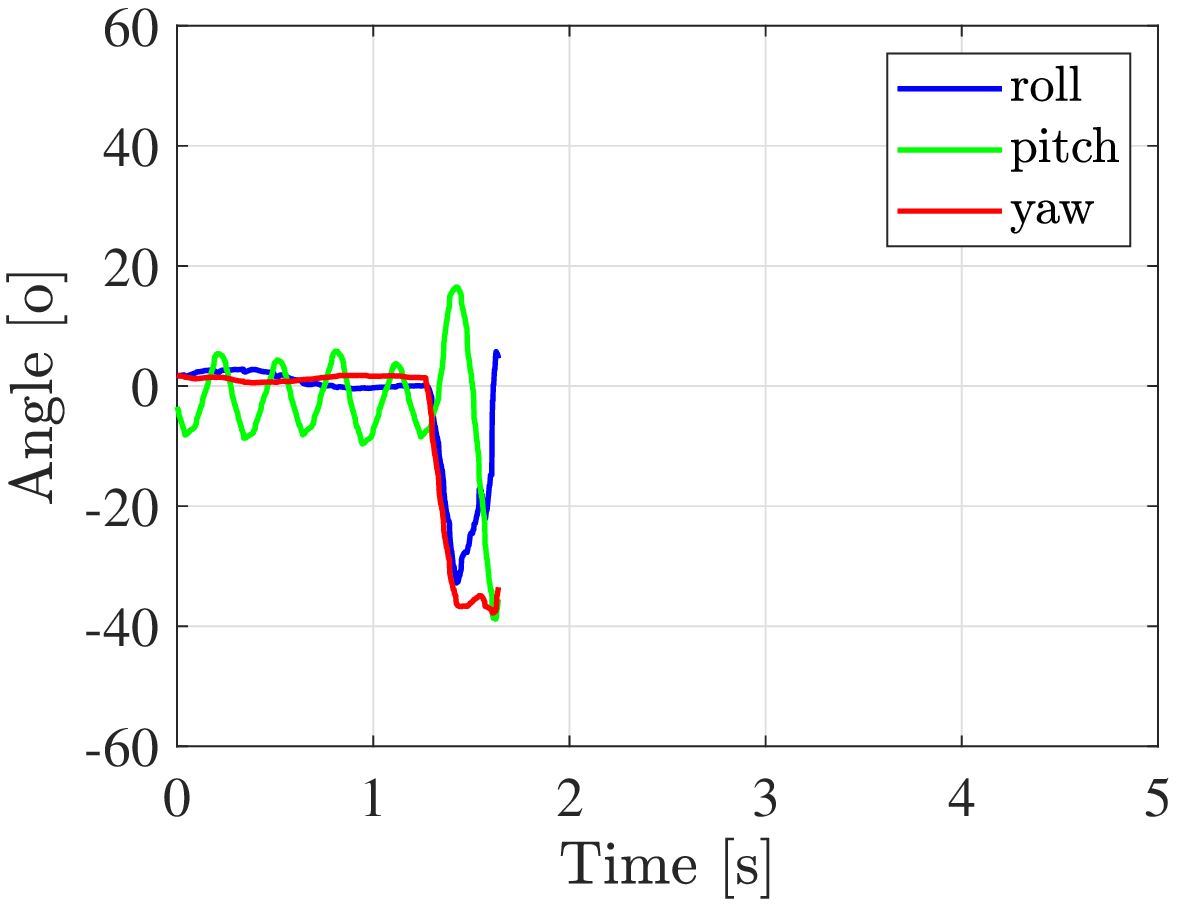}}
\caption{ \footnotesize   Experimental data with \textbf{baseline} controller and \textbf{bound} gait. }
\label{fig:exp2_mpc}
\end{figure}

Snapshots of two pushes with the proposed approach and two pushes with the baseline approach are presented in Figs.~\ref{fig:exp1_snap}-\ref{fig:exp2_snap} and Figs.~\ref{fig:exp3_snap_mpc_trotting}-\ref{fig:exp3_snap_mpc_bounding}, respectively. By comparing Fig.~\ref{fig:exp1} and Fig.~\ref{fig:exp1_mpc}, it can be seen that for the impact with magnitude roughly $\SI{10.5}{Nm}$ ($\SI{1}{m\per s}\times \SI{10.5}{kg}$), the proposed controller succeeds in rejecting the disturbance with both trot and bound gaits. However, the baseline controller fails at this level of impact with both gaits. For trot gait, the proposed controller is able to reject external disturbance up to roughly $\SI{18.9}{Nm}$ ($\SI{1.8}{m\per s}\times \SI{10.5}{kg}$) according to Fig.~\ref{fig:exp1}. For bound gait, the maximum successful disturbance rejection tested in experiment is roughly $\SI{16.8}{Nm}$ ($\SI{1.6}{m\per s}\times \SI{10.5}{kg}$). From these presented results, the advantage of the proposed controller over the baseline controller is further justified.  

By comparing the hardware experiment result (Fig.~\ref{fig:exp1}) and the simulation result (Fig.~\ref{fig:simu_state_traj_2}), it is clear that the robot's floating base experiences a larger fluctuation during balance. In addition, due to various uncertainties in hardware experiments such as state estimation noise, unmodeled motor dynamics, etc., the dynamic balance motion in hardware experiment is not perfectly periodic, as can be seen from Fig.~\ref{fig:exp1} for example. 

By comparing the trajectories plotted in Figs.~\ref{fig:exp1}-\ref{fig:exp2_mpc}, it can be observed that the baseline controller fails under smaller disturbances as compared with the proposed controller. In particular, the failure commonly occurs during the first step of the recovery phase. Without considering a dedicated footstep planning scheme associated with capturability, the footstep locations selected by the baseline controller are not good enough to support a recovery from hard pushes. On the contrary, the proposed controller is not only able to resist larger impact, but also behaves differently. To be more specific, the proposed controller typically fails during the end of the recovery phase after taking one or two successful steps. These observations suggests that the proposed controller is principally different from and has significant advantage over the baseline controller. 

\subsection{Discussions}

We summarize this section with additional discussions on the experimental results. 

The first thing worth attention is the discrepancy between the simulation and hardware experiments. For example, the robot is able to resist lateral impact of magnitude $\SI{2}{m\per s}$ regardless of the underlying gait and impact timing, according to Figs.~\ref{fig:timing_1}-~\ref{fig:timing_3}. However, in hardware experiments, we never achieve this number. Among various contributing factors, the state estimator is an important one. Recall that we assumed perfect state estimation in all simulations and used the non-further-tuned Kalman filter based state estimation provided in~\cite{DiCarlo2018} in hardware tests. With properly tuned parameters or better state estimators that return more accurate estimated state, the quadruped is expected to resist harder pushes in practice. In addition, as can be observed from videos of the hardware tests, self collision between legs during recovery and imperfect swing leg trajectory planning and tracking cause failure of push recovery as well. These issues are not reflected in simulation, partially leading to the discrepancy. Despite the performance discrepancy, the advantage of the proposed controller can still be validated by the simulation and hardware experiments. 

Regarding the hardware experiments, our testing scenario is setup in a way to push the limit of the robot in response to external pushes and the experiments are so expensive to do. During the experiments, the robot's shank broke once and the robot's front right leg broke off at the attachment to the body (Ab-Ad joint) twice when implementing our algorithm.

With the collected data from both simulation and hardware experiments, the importance of considering capturability in footstep location determination and the influence of gait timings on the push recovery ability are demonstrated. In addition, the presented capturability analysis for quadrupeds that considers dynamical balance actually offers a strategy of adjusting gait timing in reaction to large external disturbance. In short, if the robot's state is capturable according to the analysis but the current contact configuration does not agree with the one corresponding to capturability analysis, it naturally suggests that the robot should change contact configuration (i.e., adjust gait) to handle the external disturbance. Investigations on these gait adjustment are important directions of our future researches.

\section{Concluding Remarks}\label{sec:concl}
In this paper, we study capturability analysis and push recovery for quadrupedal robots, which is a critical problem for safe and reliable operation of quadrupeds. Due to the distinct multi-contact feature and the resulting flexibility in gait variability for quadrupeds, well established capturability concept and subsequent push recovery strategies for bipeds cannot be directly applied. To address this issue, we first adopt a switched system to model dynamic quadruped locomotion and then instantiate the fundamental capturability concept with the switched system model to obtain the quadrupedal capturability notion. An EMPC-based approach is then utilized to formulate and solve the corresponding capturability analysis problem. Relying on the analysis result, a novel planning scheme for finding a target footsteps and the transitioning footsteps is devised, which involves solving only quadratic programs during online implementation. Extensive simulations and experiments are conducted to validate the performance of the proposed framework. The validations suggest an up to $100\SI{}{\%}$ improvements in rejecting external pushes as compared with state-of-the-art model predictive control and whole-body impulse control (MPC+WBIC) based control strategy.

\label{sec:concl}
\bibliographystyle{unsrt}
\bibliography{Quad_PR}

\begin{appendices}

\section{Proof of Theorem~\ref{thm:pre_con}}\label{app:pre_con}

To prove the desired result, we first recall the main results in~\cite{Bertsekas1972}. Consider a general time-invariant nonlinear system $x^+ = f(x,u)$ and define the operators $\Rx$ and $\Cx$ as follows 
\eq{\ald{ &\Rx(X) = \{x : \exists u, \text{ s.t. } x^+ = f(x,u)\in X \}\bigcap X \\ &\Cx_n(X)  = \Cx(\Rx^{n-1}(X)) \\  &\Cx(Z) = \{(x,u) : x\in Z, u\in U, f(x,u)\in Z \}}} for sets $X$, $U$ and $Z$. Then we have the following result. 

\begin{lemma}\label{lem:bertres}
Assume there exists a positive integer $n_0$ such that $\Cx_n(X)$ are nonempty and compact for all $n\ge n_0$, then 
\eq{\emptyset \neq \Rx^*(X)  = \bigcap\limits_{n=1}^\infty \Rx^n(X)}
\end{lemma}

The time-invariance property of the periodic state evolution between two consecutive steps~\eqref{eq:peri_evo} allows us to leverage the above lemma to prove the desired results. In essence, it remains to show the compactness of $\Gamma(\X)$ and $\Gamma_k(\X)$ defined in~\eqref{eq:gamma} provided a compact and nonempty set $\X$ and compact and nonempty set $\U$.

\begin{lemma}\label{lem:compact}
Given compact and nonempty sets $\X$ and $\U$, then  $\Gamma(\X)$ and $\Gamma_k(\X)$ defined in~\eqref{eq:gamma} are compact for all $k\ge 0$.
\end{lemma}
\begin{proof}
Since all sets considered in this paper are Euclidean, they are compact if closed and bounded. Consider $\Gamma(\X)$ first. Due to the compactness of $\U$, the Cartesian product $\U^{T_G}$ serving as the input constraint set for~\eqref{eq:peri_evo} is compact as well. In addition, the linear mapping $[\bar{\Aa}\ \ \bar{\Bb}]$ is apparently continuous and hence implies compactness of the set $\Gamma(\X)$. Following similar arguments, the compactness of $\Gamma_k(\X)$ can be easily established, which proves this lemma.
\end{proof}

Theorem~\ref{thm:pre_con} then follows from combining Lemma~\ref{lem:bertres} and Lemma~\ref{lem:compact}.

\section{Proof of Proposition~\ref{prop:linear_cap_set}}\label{app:prop1}

The following lemma regarding the Minkowski addition $\oplus$ and the composition of polyhedron and affine mapping $\circ$ is needed in our subsequent proof of Proposition~\ref{prop:linear_cap_set}.

\begin{lemma}\label{lem:setop}
Suppose $w$ is a vector, $\Q$ is a generic set, $\P$ is a polyhedral sets with $H$-representation given by $\P = \{x: Hx\le h\}$, and $A(x)$ is a linear mapping, then the following relationships hold
\subeq{\al{ & (\P+w)\oplus \Q =  \P\oplus \Q +w, \\ & (\P + w)\circ A = \{x: HA x \le h+Hw\}  \\ & \P \circ A +w = \{x: HAx \le h+HAw\}   }}
\end{lemma}
\begin{proof}
\begin{itemize}
    \item The first relationship holds due to the following derivation \eqn{\ald{(\P+w)\oplus \Q & = \{y+q: y\in \P+w, q\in \Q\}\\ & = \{y-w+q+w: y-w\in \P, q\in \Q\} \\ & = \P\oplus \Q +w }}
    \item We first observe:
    \eqn{\ald{\P+w &= \{x:H(x-w)\le h\} \\ &= \{x:Hx\le h+Hw\}}}
    Then we have:
    \eqn{ (\P + w)\circ A  =\{x: HAx\le h+Hw\} } 
    \item We have 
    \eqn{\ald{\P \circ A +w &  = \{x: x-w\in \P\circ A\} \\ & = \{x: HA(x-w)\le h\} \\ & = \{x: HAx \le h + HAw\}}}
\end{itemize}
\end{proof}

Given the above lemma, we are ready to prove Proposition~\ref{prop:linear_cap_set}. For the ease of exposition, we prove the proposition with the following linear dynamics with 3D-LIP model
\eq{\ald{&x_{k+1}   = Ax_k+Bu_k \\& x_k\in \X, \ u_k\in \U \\ & x_T\in \X_T }}
where $A = e^{A_\text{LIP} \dt}$ and $B = A_\text{LIP}(e^{A_\text{LIP} \dt} - I )B_\text{LIP}$ with $A_\text{LIP}$, $B_\text{LIP}$ given in~\eqref{eq:3dlip}.

\begin{proof}

To prove Proposition~\ref{prop:linear_cap_set}, we first observe that determination of the tube of balanced states $\Bx$ and the set of capturable states $\Ca$ rely on the basic $\Pre$ operation defined in~\eqref{eq:pre}. Therefore, proving Proposition~\ref{prop:linear_cap_set} essentially becomes verifying the following condition. 

Given $\Omega_{2,k}  = \Omega_{1,k}+\Delta \Ww$, then 
\eqn{\Omega_{2,k+1} =\Omega_{1,k+1}+\Delta \Ww} with $\Omega_{i,k+1} = \Pre(\Omega_{i,k+1})$.

We prove the above condition by induction. 

\begin{itemize}
    \item It is clear that, the initialization satisfies the condition, i.e., 
    $\X_{T,2} = \X_{T,1}+\Delta\Ww$
    \item Suppose $\Omega_{2,k}  = \Omega_{1,k}+\Delta \w$, then  
    \eq{\label{eq:omega2}\ald{\Omega_{2,k+1}
    & = \Pre(\Omega_{2,k}) \\ 
    & = \Big(\Omega_{2,k}\oplus (-B\circ \U)\Big)\circ A\\
    & = {\Big(}{\Large(}\Omega_{1,k}+\Delta \Ww{\Large)}\oplus {\Large(}-B\circ (\U+\Delta\w){\Large)}{\Big)}\circ A \\
    & = {\Big(}\Omega_{1,k}\oplus {\Large(}-B\circ \U - B \Delta\w){\Large)} +\Delta \Ww{\Big)}\circ A \\ 
    & = {\Big(}\Omega_{1,k}\oplus {\Large(}-B\circ \U {\Large)} + \Delta \Ww - B\Delta \w  {\Big)}\circ A  }}
    Now, thanks to the polytopic nature of the set $\Omega_{1,k}\oplus {\Large(}-B\circ \U {\Large)}$, we assume that it has the following $H$-representation:
    \eqn{\Omega_{1,k}\oplus {\Large(}-B\circ \U {\Large)} = \{x: \tilde{H} x\le \tilde{h}\}.}
    
    It follows from Lemma~\ref{lem:setop} that 
    \eq{\label{eq:omega_2}\Omega_{2,k+1} = \{x: \tilde{H} A x \le \tilde{h}+H (\Delta \Ww-B\Delta \w) \}.}
    
    Also, Lemma~\ref{lem:setop} indicates that 
    \eq{\label{eq:omega_1}\ald{\Omega_{1,k+1} &  = \Pre(\Omega_{1,k})+\Delta \Ww \\
    & ={\Big(}\Omega_{1,k}\oplus {\Large(}-B\circ \U {\Large)} {\Big)}\circ A +\Delta \Ww \\ 
    & = \{x: \tilde{H} A x \le \tilde{h}+H A\Delta \Ww \}}}
    
    By comparing~\eqref{eq:omega_2} and~\eqref{eq:omega_1}, it is immediate that the Proposition is proved if we have $A\Delta \Ww = \Delta \Ww-B\Delta \w$.
    
    Now, recall that $A = e^{A_\text{LIP} \dt}$ and $B = A_\text{LIP}^{-1}(e^{A_\text{LIP} \dt} - I)B_\text{LIP}$, and 
    \eqn{A_\text{LIP} =  \ar{{cccc} 0 & 1 & 0 & 0\\ \frac{\gr}{\hz} & 0 & 0 & 0\\ 0 & 0 & 0 & 1\\ 0 & 0 & \frac{\gr}{\hz} & 0},\ \ B_\text{LIP} = \ar{{cccc} 0 & 0 & 0 & 0\\ - \frac{\gr}{\hz} & 0 & 0 & 0\\ 0 & 0 & 0 & 0 \\ 0 & 0  & -\frac{\gr}{\hz}&0}  .} Denoting by $\omega = \sqrt{\frac{\gr}{\hz}}$, then we have 
    \subeq{\al{ A  & =  \ar{{cc} \bar{A}& \0 \\ \0 & \bar{A} }  ,\text{ with}
    \\ \bar{A} & = \ar{{cc}\cosh(\omega\dt) & \frac{1}{2\omega} (e^{\omega\dt} - e^{-\omega\dt})  \\ \omega \sinh(\omega\dt) & \cosh(\omega \dt)}
    \\  B &=\ar{{cc}1- \cosh(\omega\dt) & 0 \\ - \omega \sinh(\omega\dt) & 0 \\ 0 &  1- \cosh(\omega\dt) \\ 0 & -\omega \sinh(\omega\dt)}}}
    Then, it is immediate that 
    \eq{(A-I)\Delta \Ww  = \ar{{c} (\cosh(\omega \dt)-1)\Delta w_1 \\   (\omega \sinh(\omega \dt))\Delta w_1 \\(\cosh(\omega \dt)-1)\Delta w_2 \\   (\omega \sinh(\omega \dt))\Delta w_2 } = -B\Delta \w,}which hence completes the proof.
\end{itemize}
\end{proof}
\end{appendices}

\end{document}